%% file: main.tex
\documentclass[11pt]{article}
\usepackage[OT1]{fontenc}
\usepackage{smile}
\usepackage[
  colorlinks,
  linkcolor=red,
  anchorcolor=blue,
  citecolor=blue
]{hyperref}
\usepackage{mathrsfs,fullpage}
\usepackage[protrusion=true, expansion=true]{microtype}
\usepackage{mathtools}
\usepackage{authblk}
\usepackage{algorithmicx,algpseudocode}
\usepackage{graphicx}
\usepackage[normalem]{ulem}
\usepackage{soul}
\soulregister\cite7
\soulregister\citep7
\soulregister\citet7
\soulregister\ref7
\soulregister\eqref7

\ifx\counterwithout\undefined\usepackage{chngcntr}\fi
\counterwithout{equation}{section}
\usepackage{math}
\allowdisplaybreaks[1]
\mathtoolsset{showonlyrefs=true}

\begin{document}

\title{Online Covariance Estimation for Averaged Stochastic Gradient Descent on Riemannian Manifolds}

\author[1]{Jingrong Guan}
\author[2]{Liwei Jiang}
\author[3]{Sen Na}
\affil[1]{Department of Industrial Engineering and Management Sciences, Northwestern University}
\affil[2]{Edwardson School of Industrial Engineering, Purdue University}
\affil[3]{School of Industrial and Systems Engineering, Georgia Institute of Technology}
\date{}
\maketitle

\begin{abstract}

We study online inference for averaged Riemannian stochastic gradient descent (RSGD) when the unknown parameter lies on a manifold. Although the averaged RSGD iterate is asymptotically normal, practical inference requires estimating the unknown covariance operator in its limiting distribution. Existing online batch-means methods cannot be applied directly because a general Riemannian manifold lacks an addition operation, and the tangent space containing the Riemannian gradient varies across iterations. Through vector transport, we construct overlapping and non-overlapping covariance estimators that adaptively combine information from successive batches while accounting for the changing tangent spaces. The proposed estimators are online, require neither Riemannian Hessian computations nor a prespecified terminal sample size, and attain the same convergence rates as their Euclidean counterparts~\citep{Zhu2023Online}. Together with asymptotic normality, these estimators enable online statistical inference for RSGD on manifolds. We illustrate the proposed framework through a detailed study of streaming principal component analysis on~Grassmann manifolds.

\end{abstract}


\section{Introduction}\label{sec:1}

We study online statistical inference for a local minimizer of stochastic optimization problem on Riemannian manifold. Let \(\xi\sim\mathcal D\) be a random data sample, and let \(f(x,\xi)\) be the corresponding stochastic loss. We consider
\begin{equation}\label{eq:stochastic-optimization-problem}
\min_{x\in\mathcal M}F(x):=\mathbb E_{\xi\sim\mathcal D}[f(x,\xi)],
\end{equation}
where the parameter of interest \(x_\star\) is a local minimizer of the population objective \(F\) on the manifold $\mM$. Many problems in statistics, machine learning, and signal processing can be modeled on manifolds, such as principal component analysis with orthogonality constraints \citep{Edelman1998Geometry}, dictionary learning on the sphere \citep{Sun2017Complete}, low-rank matrix completion \citep{Boumal2011RTRMC} and low multilinear rank tensor approximation \citep{Ishteva2011Best, Wen2026Online}.

The classical theory for stochastic approximation is often restricted to Euclidean space. In this setting ($\mathcal{M}=\mathbb{R}^d$), stochastic gradient descent (SGD) is one of the most standard methods for solving \eqref{eq:stochastic-optimization-problem}, which replaces the population gradient $\nabla F(x)$ with an unbiased estimate $\nabla f(x,\xi)$. In particular,~given independent observations $\xi_n\sim\mathcal D$, an initial point $x_0$, and step sizes $\gamma_n>0$, SGD performs the update
\begin{equation}\label{eq:intro-sgd}
x_n=x_{n-1}-\gamma_n\nabla f(x_{n-1},\xi_n).
\end{equation}
SGD is widely used for large-scale parameter estimation and online learning because of its low computational and memory cost \citep{Bottou2018Optimization}. Let $\bar x_n=n^{-1}\sum_{i=1}^n x_i$ denote the averaged SGD iterate. \citet{Ruppert1988Efficient} and \citet{Polyak1992Acceleration} showed that this averaged iterate is asymptotically normal and attains the \textit{optimal} covariance in the local asymptotic minimax sense of H\'ajek and Le Cam \citep{Hajek1972Local,LeCam1972Limits}. Specifically, consider step sizes $\gamma_n=\gamma n^{-\alpha}$ with $\gamma>0$~and $\alpha\in(1/2,1)$. Under suitable conditions,
\begin{equation}\label{eq:euclidean-clt}
\sqrt n(\bar x_n-x_\star)\xrightarrow{\mathrm d}\mathcal N(0,\Sigma),\qquad \Sigma=A^{-1}SA^{-1},
\end{equation}
where
\begin{equation*}
A:=\nabla^2F(x_\star) \qquad \text{and}\qquad S:=\mE [\nabla f(x_\star,\xi)\nabla f(x_\star,\xi)^\top]
\end{equation*}
denote the Hessian of the population objective at $x_\star$ and the covariance of the stochastic gradient noise, respectively. This result characterizes the estimation error at the $n^{-1/2}$ scale and provides the foundation for statistical inference.

To construct valid confidence intervals based on \eqref{eq:euclidean-clt}, one needs to estimate the asymptotic covariance $\Sigma$. \citet{Chen2020Statistical} proposed two consistent estimators: a plug-in estimator that requires Hessian information and a batch-means estimator that uses only SGD iterates. \citet{Zhu2023Online} further developed overlapping (OL) and non-overlapping (NOL) variants of the batch-means estimator that can be updated without prior knowledge of the terminal sample size, thereby enabling fully online covariance estimation. The key idea is to divide the iterates into blocks of increasing size, with each block providing an approximately independent estimate of the covariance matrix. Specifically, let $B_j$ denote a batch ending at iteration $j$, with length $\ell_j=|B_j|$ and average $\bar x_j^B=\ell_j^{-1}\sum_{x\in B_j}x$, and let $E_n$ index the batches used at iteration $n$. Both the OL and NOL estimators take the form
\begin{equation}\label{eq:intro-online-bm}
\wh\Sigma_{n}=\frac{\sum_{j\in E_n}\ell_j^2(\bar x_j^B-\bar x_n)(\bar x_j^B-\bar x_n)^\top}{\sum_{j\in E_n}\ell_j}.
\end{equation}
With a suitable choice of batch size sequence (see Section~\ref{sec:3.3}), \citet{Zhu2023Online} proved that, under suitable conditions, both estimators satisfy
\begin{equation*}
\mE \|\wh\Sigma_n-\Sigma \|\lesssim n^{-(1-\alpha)/4}.
\end{equation*}
with the convergence rate approaching $n^{-1/8}$ as $\alpha\rightarrow 1/2$.

Despite these advances in online statistical inference in Euclidean space, stochastic approximation on a general Riemannian manifold is more complicated because the Euclidean SGD update in \eqref{eq:intro-sgd} does not, in general, preserve the manifold constraint. At a point $x\in\mM$, the Riemannian gradient $\grad f(x,\xi)$ lies in the tangent space $\T_x\mM$. Riemannian stochastic gradient descent (RSGD) therefore uses a retraction $R_x$ to map a tangent vector, such as a step along the negative Riemannian gradient, back onto the manifold. Precisely, the RSGD update takes the form
\begin{equation}\label{eq:intro-rsgd}
x_n=R_{x_{n-1}}\bigl(-\gamma_n\grad f(x_{n-1},\xi_n)\bigr).
\end{equation}
\citet{Bonnabel2013Stochastic} established the convergence of RSGD, while, more recently, \citet{Tripuraneni2018Averaging} extended iterate averaging from SGD to RSGD for statistical inference. The authors used the retraction and its local inverse to recursively update the average as each new iterate becomes available, yielding
\begin{equation}\label{eq:intro-rsgd-average}
\wt x_n=R_{\wt x_{n-1}}\bigl(n^{-1}R_{\wt x_{n-1}}^{-1}(x_n)\bigr).
\end{equation}
The average thus moves toward the new iterate with weight $1/n$. For example, when $\mM=\mathbb R^d$, we can take $R_x(v)=x+v$ and $R_x^{-1}(y) = y-x$. Then, \eqref{eq:intro-rsgd} reduces to the standard SGD update $x_n=x_{n-1}-\gamma_n\nabla f(x_{n-1},\xi_n)$, and \eqref{eq:intro-rsgd-average} reduces to the ordinary recursive average $\wt x_n=\wt x_{n-1}+n^{-1}(x_n-\wt x_{n-1})$. Section~\ref{sec:2} provides the geometric definitions and assumptions used throughout this paper.

Under suitable local smoothness and stability assumptions, \citet{Tripuraneni2018Averaging} proved asymptotic normality of the average $\wt x_n$ in \eqref{eq:intro-rsgd-average}. Mapping the averaged point back to the tangent space $\T_{x_\star}\mM$ via $R_{x_\star}^{-1}$ and expressing the resulting tangent vector in an orthonormal basis of the tangent space~gives
\begin{equation*}
\sqrt n\,R_{x_\star}^{-1}(\wt x_n)\xrightarrow{\mathrm d}\mathcal N(0,\Sigma),\qquad \Sigma=A^{-1}SA^{-1}.
\end{equation*}
Here $A=\Hess F(x_\star)$ is the Riemannian Hessian and $S$ is the covariance of $\grad f(x_\star,\xi)$; both act on $\T_{x_\star}\mM$. This 
asymptotic behavior of RSGD provides a basis for statistical inference on Riemannian manifolds. Motivated by the design of batch-means covariance estimator in \citet{Zhu2023Online}, we~raise the following question:
\par\medskip
\textit{Can the batch-means covariance estimator be adapted to the averaged RSGD, and can the resulting estimator retain the same convergence rate as in the Euclidean setting?}
\par\medskip

In this paper, we provide an affirmative answer to the above question. In particular, we use retraction maps and vector transports to extend the online OL and NOL batch-means constructions of \citet{Zhu2023Online} to averaged RSGD. To this end, we compute for each batch $B_j$ its iterate average $c_j$ using the retraction-based averaging rule in \eqref{eq:intro-rsgd-average}, and construct the covariance estimator from the vectors $b_{j,n}=R_{\wt x_n}^{-1}(c_j)$, which represent the batch averages in the tangent space at the overall average $\wt x_n$. This extension presents two challenges. (i) As the overall average is updated, the tangent space containing the vectors $\{b_{j,n}\}_j$ changes from $\T_{\wt x_n}\mM$ to $\T_{\wt x_{n+1}}\mM$. Thus, computing each $b_{j,n+1}=R_{\wt x_{n+1}}^{-1}(c_j)$ directly would require revisiting all previous batch averages and contradict the online spirit. 
To obtain an online update, we approximate $b_{j,n}$ by $\tb_{j,n}$ (see \eqref{eq:batch-vector-update}) and update these approximations recursively by accounting for the change in the overall average and applying vector transport to move the resulting vectors to the new tangent space. This avoids repeatedly recomputing the inverse retractions of previous batch averages. (ii) These approximate updates introduce errors that may accumulate over time. To control them, we separate the Euclidean batch-means error from the additional errors induced by the manifold geometry. Under standard conditions (cf. Theorem~\ref{thm:covariance-rate}), we show that these geometric errors are of higher order than the leading Euclidean error and therefore do not affect the overall convergence rate. Altogether, we establish that the convergence rate of covariance estimation on manifolds matches that in Euclidean space.

\vskip4pt
\noindent\textbf{Main contributions.} 
The main contributions of the paper are threefold. First, we construct an online batch-means covariance estimator $\wt\Sigma_n$ for averaged RSGD, with OL and NOL versions. To our knowledge, this is the first estimator of the limiting covariance of averaged RSGD. Our estimator can be updated online and requires neither Riemannian Hessian computations nor a prespecified terminal sample size. Second, we establish the convergence rate for our estimator. Under the conditions of Theorem~\ref{thm:covariance-rate}, choosing the batch exponent $\beta=2/(1-\alpha)$ yields the bound
\begin{equation*}
\mE\big\|(\Gamma_{x_\star}^{\wt x_n})^{-1}\circ\wt\Sigma_n\circ\Gamma_{x_\star}^{\wt x_n}-\Sigma\big\|\lesssim n^{-(1-\alpha)/4}.
\end{equation*}
Here, $\Gamma_{x_\star}^{\wt x_n}$ denotes the parallel transport from $\T_{x_\star}\mM$ to $\T_{\wt x_n}\mM$. This rate matches the corresponding rate of \citet{Zhu2023Online} in the Euclidean setting and approaches $n^{-1/8}$ as $\alpha \to 1/2$. Finally, we combine covariance consistency with asymptotic normality of averaged RSGD to develop online statistical inference on manifolds. In particular, we construct asymptotically valid confidence intervals for smooth scalar functions, as well as joint confidence regions.

\subsection{Related work}\label{sec:1.1}

\noindent\textbf{Inference for stochastic approximation.} The asymptotic theory of stochastic approximation in the Euclidean setting has led to extensive research on online statistical inference. For SGD, alternatives to the batch-means covariance estimators discussed above include online bootstrap and random scaling \citep{Fang2018Online,Lee2022Fast}. Subsequent work has reduced the computational cost of online bootstrap by using fewer perturbed SGD runs \citep{Lam2026Cheap} and extended random scaling to weighted averages of the iterates \citep{Wei2026General}.

Statistical inference under constraints has also received growing attention. Asymptotic analyses have examined how constraint geometry affects the statistical efficiency of stochastic optimization methods \citep{Duchi2021Asymptotic}. Stochastic sequential quadratic programming (SSQP) methods have been developed for problems with nonlinear equality and inequality constraints \citep{Na2023Adaptive, Na2023Inequality, Fang2024Fully}. Building on these designs, recent work has established asymptotic normality of SSQP iterates and developed covariance estimation and random-scaling procedures for online inference \citep{Na2025Statistical, Du2025Online, Gao2025Online, Na2026Derivative}. Related developments for nonsmooth stochastic approximation include asymptotic normality and online batch-means covariance estimation \citep{Davis2024Asymptotic,Jiang2025Online}. Although manifold optimization can often be formulated as constrained optimization, our analysis focuses on the averaged iterate generated intrinsically by RSGD and on a covariance estimator evolving across the changing tangent spaces at this average.

\vskip4pt
\noindent\textbf{Riemannian optimization.} Riemannian optimization has received significant attention in the machine learning community. Comprehensive treatments can be found in \citet{Udriste1994Convex, Absil2008Optimization, Boumal2023Introduction}. In the deterministic setting, convergence guarantees have been established for gradient, trust-region, and quasi-Newton methods \citep{Zhang2016First,Absil2007Trust,Ring2012Optimization}. For stochastic objectives, \citet{Bonnabel2013Stochastic} established convergence guarantees for RSGD under suitable conditions. Stochastic variance reduction methods have also been developed for finite-sum objectives \citep{Zhang2016Riemannian,Sato2019Riemannian}. As discussed above, \citet{Tripuraneni2018Averaging} extended iterate averaging to RSGD and established asymptotic normality of the averaged iterate. This result provides the theoretical foundation for our study of online covariance estimation for averaged RSGD.

\subsection{Paper organization}\label{sec:1.2}

Section~\ref{sec:2} reviews preliminaries on Riemannian manifolds, introduces the assumptions used throughout the paper, and recalls the asymptotic normality result for averaged RSGD. Section~\ref{sec:3} constructs the OL and NOL online batch-means covariance estimators, establishes their convergence rates, and outlines the main proof sketch. Section~\ref{sec:4} develops confidence intervals and joint confidence regions on manifolds based on these results, using streaming principal component analysis on Grassmann manifolds as an illustrative example. Section~\ref{sec:5} evaluates covariance estimation and statistical inference through numerical experiments. Section~\ref{sec:6} concludes the paper and discusses directions for future~work.

\subsection{Notation}\label{sec:1.3}

We write $\mathbb N=\{1,2,\ldots\}$ for the positive integers, $|B|$ for the cardinality of a finite set $B$, and $\lfloor t\rfloor$~for the floor of $t\in\mR$. We use $I$ to denote the identity map or matrix, with a space or dimension subscript when needed, and $e_i$ to denote the $i$th canonical basis vector. In a fixed real inner product space, $\langle u,v\rangle$ denotes the inner product and $\|v\|=\langle v,v\rangle^{1/2}$ its induced norm. The $d$-dimensional unit sphere is $\mS^d=\{x\in\mR^{d+1}:\|x\|=1\}$, and $v^\perp=\{w:\langle v,w\rangle=0\}$ denotes the orthogonal complement of the span of $v$. For a linear map $L:V\to W$ between finite-dimensional real inner product spaces, $\|L\|=\sup_{\|v\|=1}\|Lv\|$ is its operator norm. Its adjoint $L^\ast:W\to V$ is defined by $\langle Lv,w\rangle_W=\langle v,L^\ast w\rangle_V$ for all $v\in V$ and $w\in W$. For invertible $L$, $L^\ast=L^{-1}$ if and only if $L$ preserves inner products. We define the outer product by $(u\otimes v)w=\langle v,w\rangle u$, with $u^{\otimes2}=u\otimes u$. Thus, $(Lu)^{\otimes2}=L\circ u^{\otimes2}\circ L^\ast$.~In~orthonormal bases, $u\otimes v$ is represented by $uv^\top$, and $L^\ast$ by the transpose of the matrix representing~$L$.

For $a_n, b_n\geq0$, $a_n\lesssim b_n$ or $a_n=O(b_n)$ means $a_n\leq Cb_n$ for some constant $C>0$ and all sufficiently large $n$. We write $a_n=o(b_n)$ if $a_n/b_n\to0$, and $a_n\asymp b_n$ if both $a_n\lesssim b_n$ and $b_n\lesssim a_n$. For notational simplicity, we use $C$ to denote generic constants that may take different values in different expressions. We write $\mE_n(\cdot)=\mE(\cdot\mid\mF_n)$ for the conditional expectation, where $\mF_n=\sigma(\xi_1,\ldots,\xi_n)$ is the $\sigma$-algebra generated by $\xi_1,\ldots,\xi_n$. The symbol $z_{1-\delta/2}$ denotes the $(1-\delta/2)$-quantile of the standard normal distribution. Similarly, $\chi^2_{d,1-\delta}$ denotes the $(1-\delta)$-quantile of the chi-squared distribution with $d$ degrees of freedom.

\section{Problem formulation and preliminaries}\label{sec:2}

We first introduce the geometric operations required for averaged RSGD and covariance estimation. We then state the assumptions used throughout the paper and recall the asymptotic normality result for averaged RSGD.

\subsection{Riemannian geometry}\label{sec:2.1}

We briefly review the geometric notions related to Riemannian manifold, while we point the reader to \citet{Absil2008Optimization} and \citet{Boumal2023Introduction} for a more detailed review.

\vskip4pt
\noindent\textbf{Riemannian manifolds and smooth maps.} 
Throughout the paper, $(\mM,\g)$ denotes a connected, $d$-dimensional Riemannian manifold, that is, a smooth manifold $\mM$ equipped with a Riemannian metric $\g$. For each $x\in\mM$, we use $\T_x\mM$ to denote the tangent space at $x$. The metric $\g$ smoothly assigns an inner product $\langle u,v\rangle_x:=\g_x(u,v)$ to each tangent space $\T_x\mM$, with induced norm $\|v\|_x:=\langle v,v\rangle_x^{1/2}$. We write $\T\mM:=\bigsqcup_{x\in\mM}\T_x\mM$ for the tangent bundle and $0_x$ for the zero vector in $\T_x\mM$. A map $\Phi:\mM\to\mM'$ between smooth manifolds is smooth if its representation in local coordinates is smooth. A smooth curve on $\mM$ is a smooth map $c:[0,1]\to\mM$. Its velocity at $t\in[0,1]$ is denoted by $\dot c(t):=\frac{\mathrm d}{\mathrm dt}c(t)\in\T_{c(t)}\mM$, where $\mathrm d/\mathrm dt$ denotes ordinary differentiation in local coordinates. The length of $c$ is defined by $L(c):=\int_0^1\|\dot c(t)\|_{c(t)}\,\mathrm dt$. The Riemannian distance $d(x,y)$ is the infimum~of~$L(c)$~over~all~piecewise smooth curves $c$ joining $x$ and $y$.

\vskip4pt
\noindent\textbf{Differential calculus and connections.} 
For a smooth map $\Phi:\mM\to\mM'$, the differential $D\Phi(x):\T_x\mM\to\T_{\Phi(x)}\mM'$ is the linear map characterized by $D\Phi(x)[v]=\left.\frac{\mathrm d}{\mathrm dt}\Phi(c(t))\right|_{t=0}$, for any $v\in\T_x\mM$ and any smooth curve $c:[0,1]\to\mM$ satisfying $c(0)=x$ and $\dot c(0)=v$. A connection specifies how vector fields are differentiated on a manifold. We use $\nabla$ to denote the Levi--Civita connection, the unique torsion-free connection compatible with $\g$. In particular, for a smooth vector field $Y:\mM\to\T\mM$ with $Y(x)\in\T_x\mM$, $\nabla_vY\in\T_x\mM$ with $v\in\T_x\mM$ denotes its covariant derivative at $x$ in the direction $v$. Along a smooth curve $c:[0,1]\to\mM$, the induced covariant derivative of a smooth field $V:[0,1]\to\T\mM$ with $V(t)\in\T_{c(t)}\mM$ is denoted by $\frac{D}{\mathrm dt}V(t):=\nabla_{\dot c(t)}V\in\T_{c(t)}\mM$. 
For a smooth function $h:\mM\to\mR$, its Riemannian gradient $\grad h(x)\in \T_x\mM$ is the unique vector satisfying the first relation below,~while its Riemannian Hessian $\Hess h(x):\T_x\mM\to\T_x\mM$ is the linear map defined by the second:
\begin{equation*}
\langle\grad h(x),v\rangle_x=Dh(x)[v], \qquad \Hess h(x)[v]:=\nabla_v\grad h, \qquad \text{for all } v\in\T_x\mM.
\end{equation*}
The Riemannian Hessian $\Hess h(x)$ is self-adjoint with respect to the Riemannian metric $\g$; that is,
\begin{equation*}
\langle\Hess h(x)[u],v\rangle_x = \langle u,\Hess h(x)[v]\rangle_x, \qquad \text{for all }u,v\in\T_x\mM.
\end{equation*}

\vskip4pt
\noindent\textbf{Exponential maps and retractions}. 
To turn a tangent vector into a step on the manifold, we can follow a geodesic. A smooth curve $\gamma:[0,1]\to\mM$ is a geodesic if $\frac{D}{\mathrm dt}\dot\gamma(t)=0_{\gamma(t)}$ for all $t\in[0,1]$. Given $x\in\mM$ and $v\in\T_x\mM$, let $\gamma_{x,v}[0,1]\to\mM$ denote the unique geodesic satisfying $\gamma_{x,v}(0)=x$ and $\dot\gamma_{x,v}(0)=v$, whenever it exists on the entire interval. If $\gamma_{x,v}$ exists at time one, the exponential map is defined by $\Exp_x(v):=\gamma_{x,v}(1)$. Restricted to a sufficiently small ball around $0_x$ in $\T_x\mM$, $\Exp_x$ is a diffeomorphism onto a neighborhood of $x$ and satisfies $d(x,\Exp_x(v))=\|v\|_x$. A retraction provides a computationally cheaper alternative to the exponential map for mapping tangent vectors~to~the~manifold. 
For an integer $k\geq1$, a $C^k$ map $R:\T\mM\to\mM$ is a retraction if, for every $x\in\mM$, its restriction $R_x:\T_x\mM\to\mM$ satisfies
\begin{equation*}
R_x(0_x)=x, \qquad DR_x(0_x)=I_{\T_x\mM}.
\end{equation*}
These conditions ensure that, for every $v\in\T_x\mM$, the curve $t\mapsto R_x(tv)$ starts at $x$ with initial velocity $v$. When $k\geq2$, $R$ is said to be a second-order retraction at $x$ if, in addition, $\left.\frac{D}{\mathrm dt}\dot c_\eta(t)\right|_{t=0}=0_x$ for every $\eta\in\T_x\mM$, where $c_\eta:[0,1]\to\mM$ is defined as $c_\eta(t):=R_x(t\eta)$. By the inverse function theorem, $R_x$ admits a local inverse, denoted by $R_x^{-1}$, from a neighborhood of $x$ in $\mM$ to a neighborhood of $0_x$ in $\T_x\mM$. Thus, $R_x^{-1}(y)$ represents a nearby point $y$ as a tangent vector in $\T_x\mM$. As $y\to x$, it~satisfies $\|R_x^{-1}(y)\|_x=d(x,y)+o(d(x,y))$.

\vskip4pt
\noindent\textbf{Parallel transport and vector transport}. 
Comparing tangent vectors at different points requires maps between the corresponding tangent spaces. For a smooth curve $c:[0,1]\to\mM$ and $s,t\in[0,1]$, the parallel transport $\mathcal P_{s\to t}^c:\T_{c(s)}\mM\to\T_{c(t)}\mM$ is characterized, for each $v\in\T_{c(s)}\mM$, by $\frac{D}{\mathrm dt}\mathcal P_{s\to t}^c(v)=0$ and $\mathcal P_{s\to s}^c(v)=v$. Parallel transport is a linear isometry and therefore preserves inner products and norms. Whenever $R_x^{-1}(y)$ is well defined, we consider the retraction curve $c_{x,y}^R:[0,1]\to\mM$, defined by $c_{x,y}^R(t):=R_x\{tR_x^{-1}(y)\}$ for $t\in[0,1]$, and define two maps from $\T_x\mM$ to $\T_y\mM$:
\begin{equation}\label{eq:transport-maps}
\Gamma_x^y:=\mathcal P_{0\to1}^{c_{x,y}^R}, \qquad\qquad \Lambda_x^y:=DR_x\{R_x^{-1}(y)\}.
\end{equation}
Here, $\Gamma_x^y$ is the parallel transport along $c_{x,y}^R$, while $\Lambda_x^y$ is the differential of $R_x$ with respect to its tangent vector argument at $R_x^{-1}(y)$. We denote by $\Lambda$ the family of maps $\{\Lambda_x^y\}$, defined for pairs $(x,y)$ for which $R_x^{-1}(y)$ exists. Both maps are linear, but only $\Gamma_x^y$ is guaranteed to be an isometry. The inverses of these maps should be distinguished from the corresponding maps obtained by interchanging $x$ and $y$. 
Specifically, $\Gamma_x^y$ is invertible, with $(\Gamma_x^y)^{-1}=\mathcal P_{1\to0}^{c_{x,y}^R}$, given by parallel transport back along the same curve. In general, this inverse need not coincide with $\Gamma_y^x$, since $c_{y,x}^R$ need not be the reverse of $c_{x,y}^R$. Similarly, if $R_x$ is a local diffeomorphism at $R_x^{-1}(y)$, then $\Lambda_x^y$ is invertible, with $(\Lambda_x^y)^{-1}=D(R_x^{-1})(y)$. In general, this inverse need not coincide with $\Lambda_y^x$, since the latter is obtained by differentiating $R_y$ rather than $R_x^{-1}$.

\vskip4pt
\noindent\textbf{Illustration on the sphere.}
We next illustrate these definitions on the sphere, where the maps have explicit formulas. Consider the unit sphere $\mS^d\subseteq\mR^{d+1}$ with the metric induced by the Euclidean inner product. For $x\in\mS^d$, tangent vectors are orthogonal to $x$, so $\T_x\mS^d=\{v\in\mR^{d+1}:x^\top v=0\}$. Thus, $P_x:=I-xx^\top$ is the orthogonal projection from $\mR^{d+1}$ onto $\T_x\mS^d$. As $x$ varies over the sphere, these tangent spaces form the tangent bundle $\T\mS^d=\{(x,v):\|x\|=1,\ x^\top v=0\}$. For the Euclidean-induced metric, the inner product and corresponding distance are
\begin{equation}\label{eq:sphere-metric-distance}
\langle u,v\rangle_x=u^\top v, \qquad d(x,y)=\arccos(x^\top y).
\end{equation}
For a smooth function $h:\mR^{d+1}\to\mR$, its restriction on $\mS^d$, $h|_{\mS^d}:\mS^d\to\mR$ satisfies
\begin{equation*}
\grad (h|_{\mS^d})(x)=P_x\nabla h(x), \qquad \Hess (h|_{\mS^d})(x)[v]=P_x\nabla^2 h(x)v-\{x^\top\nabla h(x)\}v.
\end{equation*}
Here $\nabla h$ and $\nabla^2h$ denote the Euclidean gradient and Hessian, respectively. On the sphere, geodesics lie on great circles. In this example, we use the second-order retraction obtained by normalizing $x+v$. For $v\in\T_x\mS^d$ with $\|v\|>0$, the exponential map and the retraction map we consider are
\begin{equation}\label{eq:sphere-maps}
\Exp_x(v)=\cos(\|v\|)x+\frac{\sin(\|v\|)}{\|v\|}v, \qquad R_x(v)=\frac{x+v}{(1+\|v\|^2)^{1/2}}.
\end{equation}
The exponential map moves from $x$ in the direction $v/\|v\|$ along a great circle, traveling an arc of length $\|v\|$. For the retraction, $x^\top v=0$ gives $\|x+v\|=(1+\|v\|^2)^{1/2}$, which explains the denominator in $R_x(v)$. In particular, $R_x(v)=\Exp_x\{\arctan(\|v\|)v/ \|v\|\}$, so the retraction follows the same great circle through an angle $\arctan(\|v\|)$ instead of $\|v\|$. Conversely, for $y\in\mS^d$, the local inverses $\Exp_x^{-1}(y)$ and $R_x^{-1}(y)$ are both well defined when $x^\top y>0$ and $y\ne x$. Formally,
\begin{equation*}
\Exp_x^{-1}(y)=\frac{d(x,y)}{\sin d(x,y)}P_xy, \qquad R_x^{-1}(y)=\frac{y}{x^\top y}-x.
\end{equation*}
Let $\theta:=d(x,y)=\arccos(x^\top y)$ denote the angle between $x$ and $y$ in radians. We can see both vectors point in the direction $P_xy$, but have lengths $\theta$ and $\tan\theta$, respectively. Finally, for $u\in\T_x\mS^d$, the parallell transition map and vector transport map from $\T_x\mS^d$ to $\T_y\mS^d$ are
\begin{equation*}
\Gamma_x^yu=u-\frac{u^\top y}{1+x^\top y}(x+y), \qquad \Lambda_x^yu=(x^\top y)(I-yy^\top)u.
\end{equation*}

\subsection{Averaged RSGD and asymptotic normality}\label{sec:2.2}

In this section, we introduce the averaged RSGD method for Problem~\eqref{eq:stochastic-optimization-problem} and state the assumptions under which the iterate exhibits asymptotic normality. Let $\{\xi_n\}_{n\geq1}$ be independent copies of $\xi\sim\mathcal D$ and $x_0\in\mM$ be the initial point. We set $\wt x_0=x_0$. For $n\geq1$, we define
\begin{equation}\label{equ:ave_RSGD}
x_n=R_{x_{n-1}}\bigl(-\gamma_n\grad f(x_{n-1},\xi_n)\bigr), \qquad \wt x_n=R_{\wt x_{n-1}}\left(\frac{1}{n}R_{\wt x_{n-1}}^{-1}(x_n)\right).
\end{equation}
Here, $\grad f$ is the Riemannian gradient with respect to $x$, and $\gamma_n$ is the step size. The second recursion moves the current average $\wt x_{n-1}$ toward $x_n$ along the retraction curve, with weight $1/n$. In the Euclidean case, $R_x(v)=x+v$, so the two recursions reduce to the standard SGD update and the recursive average $\wt x_n=\wt x_{n-1}+n^{-1}(x_n-\wt x_{n-1})$, respectively. Throughout the paper, we take step sizes $\gamma_n=\gamma n^{-\alpha}$ for some $\gamma>0$ and $\alpha\in(1/2,1)$, which satisfy the standard conditions for stochastic approximation step sizes, $\sum_{n\geq1}\gamma_n=\infty$ and $\sum_{n\geq1}\gamma_n^2<\infty$ \citep[see, e.g.,][]{Benveniste1990Adaptive}.

In what follows, we review assumptions for asymptotic normality established in \citet{Tripuraneni2018Averaging}. We begin with geometric setting in their Assumption~2, stating the conditions on a compact~set $\mX$ and retraction $R$ explicitly. Note that the update of the average requires $R_{\wt x_{n-1}}^{-1}(x_n)$ to exist at each iteration. Since this inverse need not exist globally, we assume that the iterates remain in a common set $\mX$ where the algorithm``behaves well''.

\begin{assumption}[Manifold and retraction]\label{ass:geometry}
There exists a compact, connected subset $\mX\subseteq\mM$~containing $x_\star$ such that, almost surely, $x_n\in\mX$ for every $n\geq0$. Moreover:
\begin{enumerate}[label=(\alph*)]
\item The set $\mX$ is uniformly totally retractive with respect to $R$: there exists $\rho_{\mathrm{ret}}>0$ such that for any $x\in\mX$, the restriction $R_x|_{B_x(0_x,\rho_{\mathrm{ret}})}$ is a diffeomorphism onto its image and $\mX\subseteq R_x\bigl(B_x(0_x,\rho_{\mathrm{ret}})\bigr)$, where $B_x(0_x,\rho_{\mathrm{ret}}):=\{\eta\in\T_x\mM:\|\eta\|_x<\rho_{\mathrm{ret}}\}$.
\item The set $\mX$ is retraction convex with respect to $R$: $\forall x,y\in\mX$ and $t\in[0,1]$, $R_x\bigl(tR_x^{-1}(y)\bigr)\in\mX$.
\item For every $y\in\mX$, the function $z\mapsto\frac12\|R_y^{-1}(z)\|_y^2$ is retraction strongly convex on $\mX$ with respect to $R$: for any $x\in\mX$ and unit vector $\eta\in\T_x\mM$, the function $t\mapsto\frac12\|R_y^{-1}\{R_x(t\eta)\}\|_y^2$ is strongly convex on every interval on which $R_x(t\eta)\in\mX$.
\item The retraction $R$ is second order at every $x\in\mX$.
\item The map $(x,\eta)\mapsto R_x(\eta)$ is $C^3$ on an open neighborhood in $\T\mM$ of $\bigl\{(x,R_x^{-1}(y)):x,y\in\mX\bigr\}$.
\end{enumerate}
\end{assumption}

The requirement that $\{x_n\}_n\subseteq \mX$ is common in stochastic approximation on manifolds \citep{Bonnabel2013Stochastic,Sato2019Riemannian,Tripuraneni2018Averaging}. Under this stability condition, item~(a) makes $R_{\wt x_{n-1}}^{-1}(x_n)$ well defined whenever $\wt x_{n-1}\in\mX$, while item~(b) then keeps $\wt x_n$ in $\mX$. Since $\wt x_0=x_0\in\mX$, induction shows that every average $\wt x_n$ is well defined and belongs to $\mX$. The following sphere example illustrates items~(a) and~(b).

Consider $\mM=\mathbb S^2$ with the normalization retraction in \eqref{eq:sphere-maps}. In the left panel of Figure~\ref{fig:totally-retractive-region}, let $\mX$ be the shaded compact, connected region with maximum pairwise distance $\delta:=\max_{x,y\in\mX}d(x,y)<\pi/2$. For any fixed $x\in\mX$, the inverse retraction formula in Section~\ref{sec:2.1} gives $\|R_x^{-1}(y)\|_x=\tan d(x,y)\leq\tan\delta$ for all $y\in\mX$. We can therefore choose a common radius $\rho_{\mathrm{ret}}>\tan\delta$ such that $R_x^{-1}(\mX)\subseteq B_x(0_x,\rho_{\mathrm{ret}})$ for every $x\in\mX$, which verifies item~(a). In the right panel, let $\mX$ be the shaded region, which appears elliptical in the figure. For any $x,y\in\mX$, the vectors $tR_x^{-1}(y)$, $t\in[0,1]$, form the line segment from $0_x$ to $R_x^{-1}(y)$ in $\T_x\mM$. The retraction $R_x$ maps this segment to the short great-circle arc from $x$ to $y$, which stays in the displayed region. Hence, $R_x(tR_x^{-1}(y))\in\mX$ for all $t\in[0,1]$, verifying the retraction convexity in item~(b).

\begin{figure}[thb]
\centering
\begin{minipage}[t]{0.49\textwidth}
\centering
\includegraphics[width=\linewidth,trim=45bp 70bp 30bp 0,clip]{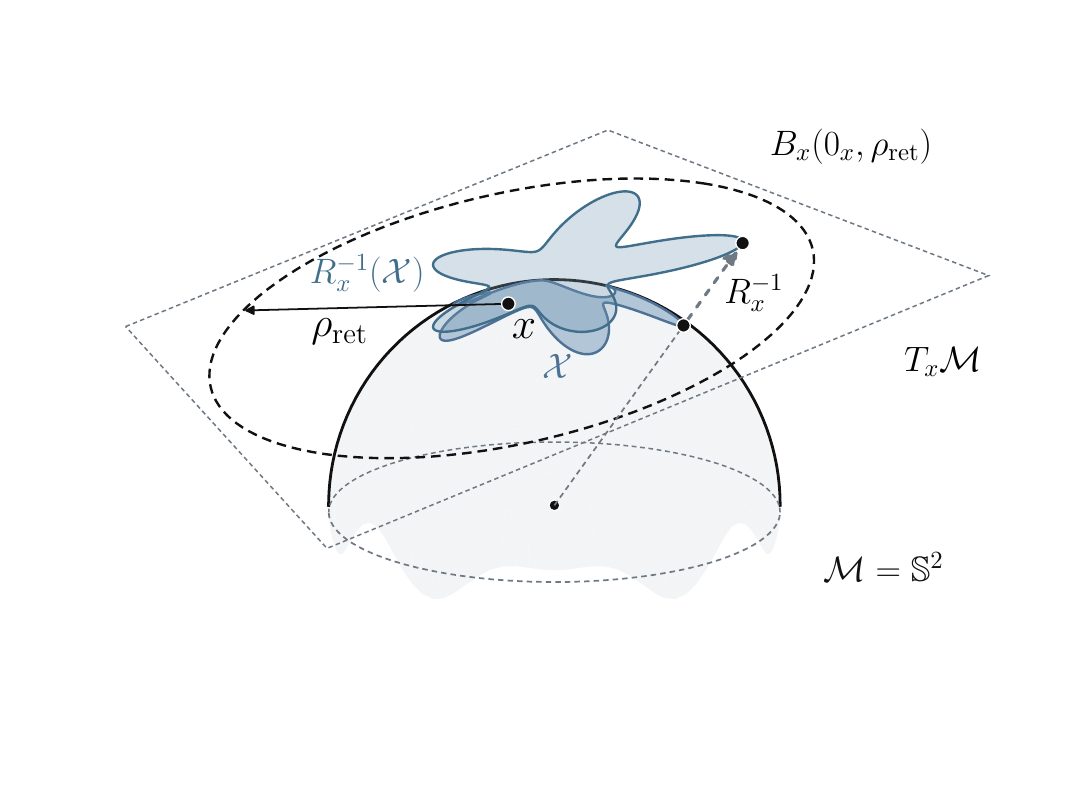}
\end{minipage}\hfill
\begin{minipage}[t]{0.49\textwidth}
\centering
\includegraphics[width=\linewidth,trim=45bp 70bp 30bp 0,clip]{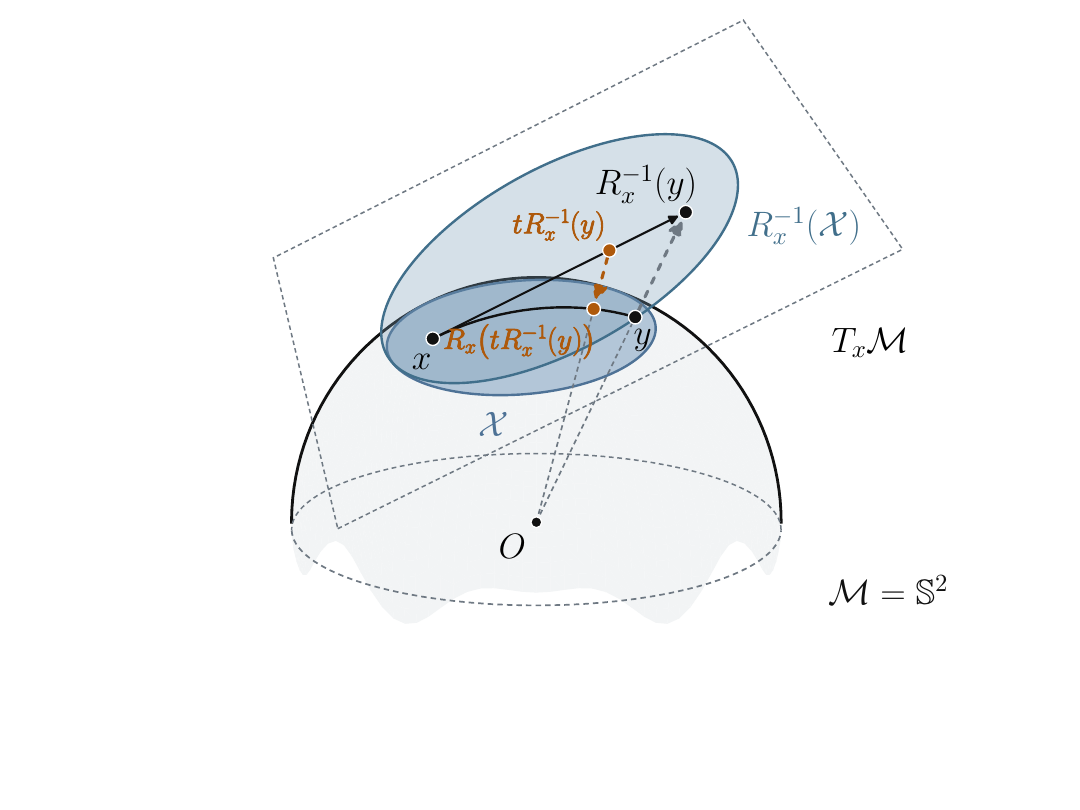}
\end{minipage}
\caption{Left: Assumption~\ref{ass:geometry}(a). The inverse image $R_x^{-1}(\mX)$ lies in the tangent ball $B_x(0_x,\rho_{\mathrm{ret}})$ for every $x\in\mX$, with a common radius $\rho_{\mathrm{ret}}>\tan (\max_{x,y\in\mX}d(x,y))$. Right: Assumption~\ref{ass:geometry}(b). For any $x,y\in\mX$, the normalization retraction maps the line segment $\{tR_x^{-1}(y):t\in[0,1]\}$ in $\T_x\mM$ onto the short great circle arc from $x$ to $y$, which remains in $\mX$, illustrating its retraction convexity.}
\label{fig:totally-retractive-region}
\end{figure}

Item~(c) is restated from \cite[Assumption~2]{Tripuraneni2018Averaging} and adopts the notion of strong convexity along retraction curves introduced by \citet{Huang2015Broyden}. Item~(d) requires the retraction to be second order, which holds for the normalization retraction on the sphere and the Grassmann polar retraction used in Section~\ref{sec:4.2} \citep{Absil2012Projection}. Finally, item~(e) provides the regularity conditions needed for the Taylor expansions used to control the geometric errors.

For the objective function $F$, we restate the smoothness and local strong convexity conditions in \citet[Assumptions~3--4]{Tripuraneni2018Averaging}.

\begin{assumption}[Objective function]\label{ass:objective}
The objective function $F$ satisfies the following conditions:
\begin{enumerate}[label=(\alph*)]
\item $F$ is twice continuously differentiable on an open neighborhood of $\mX$ and $\grad F(x_\star)=0$.
\item Let $A:=\Hess F(x_\star):\T_{x_\star}\mM\to\T_{x_\star}\mM$. There exists $\mu>0$ such that for any $\eta\in\T_{x_\star}\mM$, $\langle\eta,A\eta\rangle_{x_\star}\geq\mu\|\eta\|_{x_\star}^2$.
\item There exists $L_H<\infty$ such that for any $x\in\mX$, $\|(\Gamma_{x_\star}^{x})^{-1}\circ\Hess F(x)\circ\Gamma_{x_\star}^{x}-A\|\leq L_H\|R_{x_\star}^{-1}(x)\|_{x_\star}$.
\end{enumerate}
\end{assumption}

Items~(a) and~(b) state that $x_\star$ is a stationary point with a positive definite Riemannian Hessian. Consequently, $x_\star$ is a strict local minimizer and $A$ is invertible. Item~(c) generalizes the Lipschitz Hessian assumption to Riemannian manifolds, with one point fixed at $x_\star$.

The corresponding noise conditions in \citet[Assumptions~5--6]{Tripuraneni2018Averaging} take the following form for the stochastic gradients $\{\grad f(\cdot,\xi_n)\}_{n\geq1}$.

\begin{assumption}[Gradient noise]\label{ass:noise}
We assume the stochastic gradients satisfy the following conditions:
\begin{enumerate}[label=(\alph*)]
\item The random vector fields $\{\grad f(\cdot,\xi_n)\}_{n\ge1}$ are $\mF_n$-measurable. For every $x\in\mX$ and $n\ge1$, the random vector $\grad f(x,\xi_n)$ is unbiased: $\mE_{n-1}[\grad f(x,\xi_n)]=\grad F(x)$.
\item There exists $\tau<\infty$ such that for any $x\in\mX$ and $n\geq1$, $\mE_{n-1}[\|\grad f(x,\xi_n)\|_x^4]\leq\tau^4$.
\item There exists $L_g<\infty$ such that for any $x\in\mX$ and $n\geq1$,
\begin{equation*}
\mE_{n-1}\bigl[\|(\Gamma_{x_\star}^x)^{-1}\grad f(x,\xi_n)-\grad f(x_\star,\xi_n)\|_{x_\star}^2\bigr]\leq L_g^2\|R_{x_\star}^{-1}(x)\|_{x_\star}^2.
\end{equation*}
\item There exists a self-adjoint positive definite operator $S:\T_{x_\star}\mM\to\T_{x_\star}\mM$ such that $\forall n\geq1$,
\begin{equation*}
\mE_{n-1}[\grad f(x_\star,\xi_n)^{\otimes2}]=S.
\end{equation*}
\end{enumerate}
\end{assumption}

These conditions extend standard noise assumptions for Euclidean stochastic approximation \citep{Fabian1968Asymptotic} to the manifold setting \citep{Tripuraneni2018Averaging}. In particular, the operator $S$ in~item (d) is the covariance operator of the gradient noise at $x_\star$, since $\grad F(x_\star)=0$.

Finally, we state the moment bound conditions used in \citet[Assumption~1]{Tripuraneni2018Averaging}.

\begin{assumption}[Moment bounds for the RSGD iterates]\label{ass:raw-moments}
The RSGD iterates satisfy the following moment conditions. Let $\Delta_n:=R_{x_\star}^{-1}(x_n)\in\T_{x_\star}\mM$.
\begin{enumerate}
\item[(a)] The iterates satisfy $\mE\|\Delta_n\|_{x_\star}^2=O(\gamma_n)$.
\item[(b)] The iterates satisfy $\mE\|\Delta_n\|_{x_\star}^4=O(\gamma_n^2)$.
\end{enumerate}
\end{assumption}

Since $x_\star$ is unknown, the tangent vector $\Delta_n\in\T_{x_\star}\mM$ cannot be computed by the algorithm and is introduced only for the analysis. The second-moment bound in item~(a) is required to establish asymptotic normality, while the fourth-moment bound in item~(b) is required to estimate the covariance (a requirement needed even in the Euclidean setting). Consistent with these roles, \citet[Theorem~1]{Tripuraneni2018Averaging} uses the fourth-moment bound to establish a nonasymptotic error rate for the averaged iterates. In the Euclidean setting, the second- and fourth-moment bounds follow from strong convexity, Lipschitz continuity of the gradient, and suitable moment bounds on the stochastic noise; see \citet[Lemma~3.2]{Chen2020Statistical}, \citet[Lemma~3.1]{Zhu2023Online}, and \cite[Lemma 4.3]{Kuang2026Online}. The same moment bounds in Assumption \ref{ass:raw-moments} have been verified for a variety of manifold optimization~problems; see \citet{Arnaudon2012Stochastic,Zhou2019Fast}.

Under Assumptions~\ref{ass:geometry}--\ref{ass:noise} and the second-moment condition in Assumption~\ref{ass:raw-moments}(a), the asymptotic normality result of \citet[Theorem~1]{Tripuraneni2018Averaging} takes the following form.

\begin{theorem}[Asymptotic normality of averaged RSGD]\label{thm:rasgd-asymptotic-normality}

Consider the RSGD iterates $\{x_n\}_n$ and~recursive averages $\{\wt x_n\}_n$ in \eqref{equ:ave_RSGD}, with step sizes $\gamma_n=\gamma n^{-\alpha}$ for $\gamma>0$ and $\alpha\in(1/2,1)$. Under Assumptions~\ref{ass:geometry}--\ref{ass:noise}, \ref{ass:raw-moments}(a), define $\tDelta_n:=R_{x_\star}^{-1}(\wt x_n) \in \T_{x_\star}\mM$.
Then, with respect to any orthonormal basis of $\T_{x_\star}\mM$, we have
\begin{equation*}
\sqrt{n}\,\tDelta_n\xrightarrow{\mathrm{d}}\mathcal{N}(0,\Sigma) \qquad  \text{with}\qquad \Sigma:=A^{-1}SA^{-1}. 
\end{equation*}

\end{theorem}

Theorem~\ref{thm:rasgd-asymptotic-normality} shows that the error of the averaged RSGD iterate exhibits asymptotic normality in $\T_{x_\star}\mM$. Its limiting covariance, however, depends on the unknown operators $A$ and $S$. To enable practical inference, Section~\ref{sec:3} develops consistent online batch-means estimators of this covariance $\Sigma$ using only the RSGD iterates, without evaluating the Riemannian Hessian. Combining these estimators~with Theorem~\ref{thm:rasgd-asymptotic-normality} yields the confidence intervals and confidence regions constructed in Section~\ref{sec:4}.

\section{Online covariance estimation on Riemannian manifolds}\label{sec:3}

In this section, we develop online batch-means estimators of the asymptotic covariance $\Sigma$ in Theorem~\ref{thm:rasgd-asymptotic-normality}. The target covariance $\Sigma$ acts on $\T_{x_\star}\mM$, but this space is unknown because $x_\star$ is unknown. A natural choice is therefore to construct the estimator on $\T_{\wt x_n}\mM$, using the averaged iterate $\wt x_n$ as our current estimate of $x_\star$. Unlike in Euclidean space, both the operator and the space on which it acts change as the average $\wt x_n$ is updated, creating a central challenge for online estimation. We first present a general framework for constructing online batch-means estimators on manifolds, then specialize it to two estimators: overlapping (OL) and non-overlapping (NOL). Section~\ref{sec:3.2} presents a common online implementation. Section~\ref{sec:3.3} establishes convergence rates for both estimators, and Section~\ref{sec:3.4} outlines the proof strategy, with full details deferred to Appendix~\ref {app:D}.

\subsection{Online batch-means covariance estimators}\label{sec:3.1}

We first introduce the notation and components used to construct our batch-means estimator. Consider the RSGD iterates $\{x_j\}_{j\geq1}$ generated by \eqref{equ:ave_RSGD}. Let $\{a_m\}_{m\geq1}$ be a user-selected, strictly increasing sequence of positive integers with $a_1=1$. For each $j\geq1$, define $m_j:=\max\{m:a_m\leq j\}$ and let $t_j:=a_{m_j}$ be the largest value of $a_m$ not exceeding $j$. The construction consists of three components:
\begin{enumerate}
\item \textbf{Batches.} We adopt the same batch construction as \citet{Zhu2023Online} in the Euclidean setting. For each $j\geq1$, define the batch
\begin{equation}\label{eq:batch-indices}
B_j:=\{x_{t_j},x_{t_j+1},\ldots,x_j\}.
\end{equation}
Its length is $\ell_j:=|B_j|=j-t_j+1$. For example, with  $a_m=m^2$ we have $t_j=\lfloor\sqrt{j}\rfloor^2$ and
\begin{equation}\label{eq:batch-example}
\begin{aligned}
& B_1=\{x_1\},\quad B_2=\{x_1,x_2\},\quad B_3=\{x_1,x_2,x_3\}, \\
& B_4=\{x_4\},\quad B_5=\{x_4,x_5\},\quad \ldots ,\quad B_8=\{x_4,x_5,x_6,x_7,x_8\}, \\
& B_9=\{x_9\},\quad B_{10}=\{x_9,x_{10}\},\quad B_{11}=\{x_9,x_{10},x_{11}\},\quad B_{12}=\{x_9,x_{10},x_{11},x_{12}\},\quad\ldots.
\end{aligned}
\end{equation}
In general, for $a_m\leq j<a_{m+1}$, the batch $B_j$ contains the iterates $x_{a_m},\ldots,x_j$. Thus, successive batches within this interval share a starting index and increase in length by one. At $j=a_{m+1}$, the starting index changes and $B_{a_{m+1}}=\{x_{a_{m+1}}\}$. 
\item \textbf{Batch means.} Since arithmetic averaging is not intrinsically defined on a general Riemannian manifold, we use the same retraction-based averaging rule as for the overall average $\wt x_n$ in \eqref{equ:ave_RSGD}. We define the mean $c_j$ of batch $B_j$ recursively by
\begin{equation}\label{eq:batch-average}
c_j:=\begin{cases}
x_j, & j=t_j,\\
R_{c_{j-1}}\left(\ell_j^{-1}R_{c_{j-1}}^{-1}(x_j)\right), & j>t_j
\end{cases}.
\end{equation}
Assumption~\mbox{\ref{ass:geometry}} ensures that this recursion is well defined and that $c_j\in\mX$. In Euclidean space, with the standard retraction $R_x(v)=x+v$, it reduces to the usual arithmetic batch mean. Equation~\eqref{eq:batch-average} also provides the online update for $c_j$ used in our algorithm.
\item \textbf{Batch deviations.} We next measure the deviation of each batch mean $c_j$ from the overall average $\wt x_n$. Since subtraction is not intrinsically defined on a general Riemannian manifold, we express this deviation as a tangent vector using the inverse retraction:
\begin{equation}\label{eq:batch-vector}
b_{j,n}:=R_{\wt x_n}^{-1}(c_j)\in\T_{\wt x_n}\mM,\qquad n\geq j.
\end{equation}
In Euclidean space, with $R_x(v)=x+v$, we have $b_{j,n}=c_j -\bar x_n = \ell_j^{-1}\sum_{i=t_j}^j x_i-\bar x_n$, where $\bar x_n=n^{-1}\sum_{i=1}^n x_i$ is the arithmetic average of all the iterates. For a fixed $j$, the batch mean $c_j$ remains unchanged after time $j$. However, as $\wt x_n$ is updated, both the value of $b_{j,n}$ and the tangent space to which it belongs change. Evaluating \eqref{eq:batch-vector} directly for all selected batches at each time $n$ would require storing and revisiting their means, with memory and per-iteration costs growing with the number of selected batches. To enable an online implementation, we instead approximate $b_{j,n}$ by the recursively updated vector
\begin{equation}\label{eq:batch-vector-update}
\tb_{j,k}:=\begin{cases}
b_{j,j}, & k=j,\\
\Lambda_{\wt x_{k-1}}^{\wt x_k}\left(\tb_{j,k-1}-u_k\right), & k=j+1,\ldots,n,
\end{cases}
\end{equation}
where $u_k:=R_{\wt x_{k-1}}^{-1}(\wt x_k)=k^{-1}R_{\wt x_{k-1}}^{-1}(x_k)\in\T_{\wt x_{k-1}}\mM$ and $\Lambda_{\wt x_{k-1}}^{\wt x_k}$ is the vector transport defined in \eqref{eq:transport-maps} with $x=\wt x_{k-1}$ and $y=\wt x_k$. For $k>j$, the update first subtracts $u_k$ from $\tb_{j,k-1}$ to account for the movement of the overall average from $\wt x_{k-1}$ to $\wt x_k$. It then applies the linear map $\Lambda_{\wt x_{k-1}}^{\wt x_k}$ to transport the adjusted vector from $\T_{\wt x_{k-1}}\mM$ to $\T_{\wt x_k}\mM$. In Euclidean space, $\Lambda_{\bar x_{k-1}}^{\bar x_k}$ is the identity map and $u_k=\bar x_k-\bar x_{k-1}$. Thus, \eqref{eq:batch-vector-update} reproduces the exact identity $b_{j,k}=b_{j,k-1}-(\bar x_k-\bar x_{k-1})$, yielding $\tb_{j,n}=b_{j,n}$. On a general manifold, $\tb_{j,n}$ only approximates $b_{j,n}$, but our analysis shows that the resulting error does not affect the convergence rate of the covariance estimator.
\end{enumerate}

The covariance estimator need not use every batch. Let $E_n\subseteq\{1,\ldots,n\}$ index the batches selected at time $n$. We assume that the prescribed sequence $\{E_n\}_{n\geq1}$ satisfies $E_1=\{1\}$ and that $E_{n+1}$ is either $E_n\cup\{n+1\}$ or $(E_n\setminus\{n\})\cup\{n+1\}$ for every $n\geq1$. Using the recursively updated deviations from~\eqref{eq:batch-vector-update}, we define our general batch-means estimator as
\begin{equation}\label{eq:online-estimator}
\wt\Sigma_n:=\frac{1}{v_n}\sum_{j\in E_n}\ell_j^2\tb_{j,n}^{\otimes2},\qquad v_n:=\sum_{j\in E_n}\ell_j.
\end{equation}

We now specialize the choice of $E_n$ and introduce two main estimators we consider in this work. 

\paragraph{Overlapping estimator.} We take $E_n=E_n^{\mathrm{OL}}:=\{1,\ldots,n\}$ to obtain the OL estimator.  In this case, every batch $B_j$ is selected. The general covariance estimator in \eqref{eq:online-estimator} then becomes
\begin{equation}\label{eq:ol-estimator}
\wt\Sigma_{n,\mathrm{OL}}:=\frac{1}{v_n}\sum_{j\in E_n^{\mathrm{OL}}}\ell_j^2\tb_{j,n}^{\otimes2}.
\end{equation}
Here $v_n=\sum_{j\in E_n^{\mathrm{OL}}}\ell_j$ is the sum of the batch lengths, including overlaps. In Euclidean space, with $R_x(v)=x+v$, the estimator recovers the overlapping estimator introduced in~\citet{Zhu2023Online}.

\paragraph{Non-overlapping estimator.} 
For the non-overlapping version, we take
\begin{equation}\label{eq:nol-endpoints}
E_n^{\mathrm{NOL}}:=\{a_{m+1}-1:m\geq1,\ a_{m+1}\leq n\}\cup\{n\}.
\end{equation}
In contrast, the NOL estimator selects only the past batches that end immediately before a new starting index, together with the current batch $B_n$. The batches we select are pairwise disjoint and partition $\{x_1,\ldots,x_n\}$, so their total length is $v_n=n$. With this choice, the covariance estimator in \eqref{eq:online-estimator} becomes
\begin{equation}\label{eq:nol-estimator}
\wt\Sigma_{n,\mathrm{NOL}}:=\frac{1}{n}\sum_{j\in E_n^{\mathrm{NOL}}}\ell_j^2\tb_{j,n}^{\otimes2}.
\end{equation}
In Euclidean space, with $R_x(v)=x+v$, the estimator recovers the non-overlapping estimator introduced in~\citet{Zhu2023Online}.

For comparison, consider the example with $a_m=m^2$ in \eqref{eq:batch-example}. At time $9$, the OL estimator selects $B_1,\ldots,B_9$, while the NOL estimator selects only $B_3,B_8,B_9$. At time $10$, OL keeps $B_9$ and adds $B_{10}$, whereas NOL replaces $B_9$ with $B_{10}$. In the next section, we present a concrete algorithm to calculate $\tilde \Sigma_n$ in an online fashion.

\subsection{Online implementation}\label{sec:3.2}

To illustrate the method,  we separate the current batch $B_n$ from the other selected batches. At time $n$, let $H_n:=E_n\setminus\{n\}$ denote the indices of these other  batches. We call the batches $B_j$ with $j\in H_n$ retained batches. Under the assumed selection rule that $E_{n+1}$ is either $E_n\cup\{n+1\}$ or $(E_n\setminus\{n\})\cup\{n+1\}$, only the current index may be removed at each update. Hence, $j\in H_n$ implies $j\in E_k$ for every $k\geq n$, so a retained batch remains selected at all subsequent times. The algorithm handles the retained batches  and  the current batch separately.

Alongside the iterate $x_n$ and overall average $\wt x_n$, the algorithm maintains $m_n$, which identifies the starting index $a_{m_n}$, and the current batch length $\ell_n$, mean $c_n$, and deviation approximation $\tb_{n,n}$. For the retained batches, it maintains their total length $s_n:=\sum_{j\in H_n}\ell_j=v_n-\ell_n$ and the weighted~sums
\begin{equation*}
q_n:=\sum_{j\in H_n}\ell_j^2,\qquad P_n:=\sum_{j\in H_n}\ell_j^2\tb_{j,n},\qquad V_n:=\sum_{j\in H_n}\ell_j^2\tb_{j,n}^{\otimes2}.
\end{equation*}
Here, $V_n$ collects the contributions of the retained batches to the numerator of \eqref{eq:online-estimator}, while $s_n$ accounts for their contribution to the normalizing factor. The auxiliary sums $P_n$ and $q_n$ can be maintained recursively and allow $V_n$ to be updated without revisiting individual retained batches. Algorithm~\ref{alg:online-covariance} combines these updates with those of the current batch to compute the estimator in an online fashion.

\begin{algorithm}[H]
\small
\caption{Online batch-means covariance estimation with averaged RSGD}
\label{alg:online-covariance}
\textbf{Input:} stochastic gradient oracle $\grad f(\cdot,\xi)$; retraction $R$ and its associated vector transport $\Lambda$;\\
\hphantom{\textbf{Input:} }initial point $x_0\in\mM$; step sizes $\{\gamma_n\}_{n\geq1}$; batch starts $\{a_m\}_{m\geq1}$; batch selection rule $E_n$.\\
\textbf{Initialize:} $\wt x_0=x_0$; $m_0=\ell_0=s_0=q_0=0$; $P_0=0_{x_0}$; $V_0=0:\T_{x_0}\mM\to\T_{x_0}\mM$.\\
\textbf{For} $n=0,1,2,\ldots$, receive $\xi_{n+1}$ and perform the following updates.\\
\hspace*{2ex}1. \textbf{Update the RSGD iterate and overall average:}\\
\hspace*{5ex}$x_{n+1}=R_{x_n}\{-\gamma_{n+1}\grad f(x_n,\xi_{n+1})\}$.\\
\hspace*{5ex}$u_{n+1}=(n+1)^{-1}R_{\wt x_n}^{-1}(x_{n+1})$; $\wt x_{n+1}=R_{\wt x_n}(u_{n+1})$; $\Lambda_{n+1}=\Lambda_{\wt x_n}^{\wt x_{n+1}}$.\\
\hspace*{2ex}2. \textbf{Retain the previous batch contribution if required:}\\
\hspace*{5ex}\textbf{If} $n\geq1$ and $n\in E_{n+1}$, set\\
\hspace*{8ex}$s_n\gets s_n+\ell_n$; $q_n\gets q_n+\ell_n^2$;\\
\hspace*{8ex}$P_n\gets P_n+\ell_n^2\tb_{n,n}$; $V_n\gets V_n+\ell_n^2\tb_{n,n}^{\otimes2}$.\\
\hspace*{2ex}3. \textbf{Update the current batch average:}\\
\hspace*{5ex}\textbf{If} $n+1=a_{m_n+1}$, set $m_{n+1}=m_n+1$, $\ell_{n+1}=1$, and $c_{n+1}=x_{n+1}$;\\
\hspace*{5ex}\textbf{else}, set $m_{n+1}=m_n$, $\ell_{n+1}=\ell_n+1$, and $c_{n+1}=R_{c_n}\{\ell_{n+1}^{-1}R_{c_n}^{-1}(x_{n+1})\}$.\\
\hspace*{5ex}Compute and retain $\tb_{n+1,n+1}=R_{\wt x_{n+1}}^{-1}(c_{n+1})$.\\
\hspace*{2ex}4. \textbf{Recenter and transport the retained sums:}\\
\hspace*{5ex}$s_{n+1}=s_n$; $q_{n+1}=q_n$.\\
\hspace*{5ex}$P_{n+1}=\Lambda_{n+1}(P_n-q_nu_{n+1})$.\\
\hspace*{5ex}$V_{n+1}=\Lambda_{n+1}\circ\bigl\{V_n-P_n\otimes u_{n+1}-u_{n+1}\otimes P_n+q_nu_{n+1}^{\otimes2}\bigr\}\circ\Lambda_{n+1}^\ast$.\\
\hspace*{2ex}\textbf{Output:} averaged RSGD iterate $\wt x_{n+1}$; covariance estimate $\wt\Sigma_{n+1}=(V_{n+1}+\ell_{n+1}^2\tb_{n+1,n+1}^{\otimes2})/(s_{n+1}+\ell_{n+1})$.
\end{algorithm}

For $n\geq1$, step~2 determines whether $B_n$ becomes a retained batch. If $n\in E_{n+1}$, then $H_{n+1}=H_n\cup\{n\}$, and its contribution is added to the retained-batch sums. Otherwise, $H_{n+1}=H_n$, and these sums remain unchanged.
Step~3 then computes the mean $c_{n+1}$ of $B_{n+1}$. If $n+1=a_{m_n+1}$, this batch contains only the iterate $x_{n+1}$, so $c_{n+1}=x_{n+1}$. Otherwise, $B_{n+1}$ extends $B_n$ by one iterate, and $c_{n+1}$ is computed from $c_n$ using \eqref{eq:batch-average}. The algorithm then evaluates the deviation $\tb_{n+1,n+1}=R_{\wt x_{n+1}}^{-1}(c_{n+1})$. 

The deviation of $B_{n+1}$ already lies in $\T_{\wt x_{n+1}}\mM$, but the retained-batch sums $P_n$ and $V_n$ assembled in step~2 still use the deviations $\tb_{j,n}\in\T_{\wt x_n}\mM$. Step~4 adjusts these sums for the movement of the overall average through the common shift and transport in \eqref{eq:batch-vector-update}. For the weighted sum $P_n$, linearity implies
\begin{equation*}
\begin{aligned}
P_{n+1}
&=\sum_{j\in H_{n+1}}\ell_j^2\tb_{j,n+1}\\
&=\Lambda_{n+1}\Bigl\{\sum_{j\in H_{n+1}}\ell_j^2\tb_{j,n}
-\Bigl(\sum_{j\in H_{n+1}}\ell_j^2\Bigr)u_{n+1}\Bigr\} =\Lambda_{n+1}\bigl(P_n-q_nu_{n+1}\bigr).
\end{aligned}
\end{equation*}
For the second-order sum, expanding the weighted outer products yields
\begin{equation*}
\begin{aligned}
\sum_{j\in H_{n+1}}\ell_j^2(\tb_{j,n}-u_{n+1})^{\otimes2}
&=\sum_{j\in H_{n+1}}\ell_j^2\tb_{j,n}^{\otimes2}
-\Bigl(\sum_{j\in H_{n+1}}\ell_j^2\tb_{j,n}\Bigr)\otimes u_{n+1}\\
&\quad-u_{n+1}\otimes\Bigl(\sum_{j\in H_{n+1}}\ell_j^2\tb_{j,n}\Bigr)
+\Bigl(\sum_{j\in H_{n+1}}\ell_j^2\Bigr)u_{n+1}^{\otimes2}\\
&=V_n-P_n\otimes u_{n+1}-u_{n+1}\otimes P_n+q_nu_{n+1}^{\otimes2}.
\end{aligned}
\end{equation*}
For $z\in\T_{\wt x_n}\mM$, the transport identity $
(\Lambda_{n+1}z)^{\otimes2}
=\Lambda_{n+1}\circ z^{\otimes2}\circ\Lambda_{n+1}^{\ast}$
implies
\begin{equation*}
\begin{aligned}
V_{n+1}
&=\sum_{j\in H_{n+1}}\ell_j^2\tb_{j,n+1}^{\otimes2}\\
&=\Lambda_{n+1}\circ\Bigl\{\sum_{j\in H_{n+1}}\ell_j^2(\tb_{j,n}-u_{n+1})^{\otimes2}\Bigr\}\circ\Lambda_{n+1}^{\ast}\\
&=\Lambda_{n+1}\circ\bigl\{V_n-P_n\otimes u_{n+1}-u_{n+1}\otimes P_n+q_nu_{n+1}^{\otimes2}\bigr\}\circ\Lambda_{n+1}^{\ast}.
\end{aligned}
\end{equation*}
The updates of $P_n$ and $V_n$ require only the aggregate sums, allowing the estimator to be updated online without revisiting individual retained batches.

Finally, at each time $n\geq1$, the output combines the retained-batch sum $V_n$ with the contribution of $B_n$. The decomposition $E_n=H_n\cup\{n\}$ and the identity $v_n=s_n+\ell_n$ imply
\begin{equation*}
\frac{V_n+\ell_n^2\tb_{n,n}^{\otimes2}}{s_n+\ell_n}
=\frac{1}{v_n}\sum_{j\in E_n}\ell_j^2\tb_{j,n}^{\otimes2}
=\wt\Sigma_n.
\end{equation*}
Thus the output is exactly the estimator in \eqref{eq:online-estimator}; updating the weighted sums introduces no further approximation beyond the batch vectors $\tb_{j,n}$ already used in its definition. We evaluate the numerical performance of the estimators in Section~\ref{sec:5}.

\subsection{Convergence rates}\label{sec:3.3}

In this section, we establish convergence guarantees for the OL and NOL estimators defined in~\eqref{eq:ol-estimator} and \eqref{eq:nol-estimator}, respectively. Recall that the estimators act on $\T_{\wt x_n}\mM$, whereas $\Sigma$ acts on $\T_{x_\star}\mM$. We therefore use parallel transport to express the estimators on $\T_{x_\star}\mM$ for comparison with $\Sigma$. This transport is used only in the analysis; computing the estimators requires neither parallel transport nor knowledge of $x_\star$. We now state our main convergence guarantee.

\begin{theorem}[Convergence rates of the OL and NOL estimators]\label{thm:covariance-rate}
Consider the RSGD iterates $\{x_n\}_n$ and~recursive averages $\{\wt x_n\}_n$ in \eqref{equ:ave_RSGD}, with step sizes $\gamma_n=\gamma n^{-\alpha}$ for $\gamma>0$ and $\alpha\in(1/2,1)$. Under Assumptions~\ref{ass:geometry}--\ref{ass:raw-moments}, let $\wt\Sigma_{n,\mathrm{OL}}$ and $\wt\Sigma_{n,\mathrm{NOL}}$ be the estimators defined in~\eqref{eq:ol-estimator} and \eqref{eq:nol-estimator}, respectively, with batch starts $a_1=1$ and $a_m=\lfloor cm^\beta\rfloor$ for $m\geq2$, where $c\geq1$ and $\beta>1/(1-\alpha)$.
Then, for $\diamond\in\{\mathrm{OL},\mathrm{NOL}\}$, we have
\begin{equation*}
\mE\left\|\left[(\Gamma_{x_\star}^{\wt x_n})^{-1}\circ\wt\Sigma_{n,\diamond}\circ\Gamma^{\wt x_n}_{x_\star}\right]-\Sigma\right\|\lesssim n^{-1/(2\beta)}+n^{(\alpha-1)/2+1/(2\beta)}.
\end{equation*}
\end{theorem}

We defer the proof for OL and NOL to Appendices~\ref{app:D}. Theorem~\ref{thm:covariance-rate} shows that, as $n\to\infty$, the OL and NOL estimators converge to the asymptotic covariance $\Sigma$ of averaged RSGD. The convergence rate is associated with the step-size parameter $\alpha$ and the batch parameter $\beta$. For each fixed $\alpha$, optimizing the choice of $\beta$ leads to the following corollary.

\begin{corollary}[Choice of batch exponent]\label{cor:recommended-beta}
Consider the RSGD iterates $\{x_n\}_n$ and~recursive averages $\{\wt x_n\}_n$ in \eqref{equ:ave_RSGD}, with step sizes $\gamma_n=\gamma n^{-\alpha}$ for $\gamma>0$ and $\alpha\in(1/2,1)$. Under Assumptions~\ref{ass:geometry}--\ref{ass:raw-moments}, let $\wt\Sigma_{n,\mathrm{OL}}$ and $\wt\Sigma_{n,\mathrm{NOL}}$ be the estimators defined in~\eqref{eq:ol-estimator} and \eqref{eq:nol-estimator}, respectively, with batch starts $a_1=1$ and $a_m=\lfloor cm^\beta\rfloor$ for $m\geq2$, where $c\geq1$ and $\beta=2/(1-\alpha)$.
Then, for $\diamond\in\{\mathrm{OL},\mathrm{NOL}\}$, we have
\begin{equation*}
\mE\left\|\left[(\Gamma_{x_\star}^{\wt x_n})^{-1}\circ\wt\Sigma_{n,\diamond}\circ\Gamma^{\wt x_n}_{x_\star}\right]-\Sigma\right\|\lesssim n^{-(1-\alpha)/4}.
\end{equation*}
\end{corollary}

The convergence rates in Theorem~\ref{thm:covariance-rate} and Corollary~\ref{cor:recommended-beta} match those established by \citet{Zhu2023Online} for the corresponding Euclidean batch-means estimators. Thus, under our assumptions, the manifold geometry does not slow the convergence of the covariance estimators . In particular, the optimized rate approaches $n^{-1/8}$ as $\alpha$ approaches $1/2$ from above.

\subsection{Proof strategy}\label{sec:3.4}

In this section, we provide a brief overview of the proof of Theorem~\ref{thm:covariance-rate} for the OL estimator; the NOL argument follows the same strategy. The key is to separate the additional errors introduced by the manifold geometry from the batch-means estimation error already present in Euclidean space. To this end, we introduce the following notation and intermediate quantities.

Recall the errors of the RSGD iterates, $\Delta_k=R_{x_\star}^{-1}(x_k)$, and the error of the averaged iterate, $\tDelta_n=R_{x_\star}^{-1}(\wt x_n)$. For each batch $B_j$ in \eqref{eq:batch-indices} and its recursive mean $c_j$ in \eqref{eq:batch-average}, define
\begin{equation}\label{eq:tangent-averages}
\barDelta_n:=\frac1n\sum_{k=1}^n\Delta_k, \qquad \barDelta_j^B:=\frac1{\ell_j}\sum_{k=t_j}^j\Delta_k, \qquad \tDelta_j^B:=R_{x_\star}^{-1}(c_j),
\end{equation}
where $\barDelta_n$ and $\barDelta_j^B$ are arithmetic averages of the tangent vectors $\Delta_k$, while $\tDelta_n$ and $\tDelta_j^B$ are obtained by averaging on the manifold before mapping the averages to $\T_{x_\star}\mM$. Using the batch indices $E_n^{\mathrm{OL}}$~and normalizing factor $v_n$ from \eqref{eq:ol-estimator}, we define four operators on this fixed tangent space:
\begin{equation}\label{eq:ol-comparison-operators}
\begin{aligned}
\hSigma_{n,\mathrm{OL}}^{(0)}&:=\frac{1}{v_n}\sum_{j\in E_n^{\mathrm{OL}}}\ell_j^2\bigl((\Gamma_{x_\star}^{\wt x_n})^{-1}\tb_{j,n}\bigr)^{\otimes 2},\\
\hSigma_{n,\mathrm{OL}}^{(1)}&:=\frac{1}{v_n}\sum_{j\in E_n^{\mathrm{OL}}}\ell_j^2\bigl((\Gamma_{x_\star}^{\wt x_n})^{-1}b_{j,n}\bigr)^{\otimes 2},\\
\hSigma_{n,\mathrm{OL}}^{(2)}&:=\frac{1}{v_n}\sum_{j\in E_n^{\mathrm{OL}}}\ell_j^2\bigl(\tDelta_j^B-\tDelta_n\bigr)^{\otimes 2},\\
\hSigma_{n,\mathrm{OL}}^{(3)}&:=\frac{1}{v_n}\sum_{j\in E_n^{\mathrm{OL}}}\ell_j^2\bigl(\barDelta_j^B-\barDelta_n\bigr)^{\otimes 2}.
\end{aligned}
\end{equation}
In \eqref{eq:ol-comparison-operators}, $\hSigma_{n,\mathrm{OL}}^{(0)}=(\Gamma_{x_\star}^{\wt x_n})^{-1}\circ\wt\Sigma_{n,\mathrm{OL}}\circ\Gamma_{x_\star}^{\wt x_n}$ is the OL estimator transported to $\T_{x_\star}\mM$. The last operator, $\hSigma_{n,\mathrm{OL}}^{(3)}$, has the Euclidean OL batch-means form applied to the vectors $\Delta_k$. The triangle inequality separates the total error into the three differences between successive operators and the remaining batch-means error:
\begin{align*}
\mE\|\hSigma_{n,\mathrm{OL}}^{(0)}-\Sigma\|
&\leq \mE\|\hSigma_{n,\mathrm{OL}}^{(0)}-\hSigma_{n,\mathrm{OL}}^{(1)}\|
+\mE\|\hSigma_{n,\mathrm{OL}}^{(1)}-\hSigma_{n,\mathrm{OL}}^{(2)}\|\\
&\quad+\mE\|\hSigma_{n,\mathrm{OL}}^{(2)}-\hSigma_{n,\mathrm{OL}}^{(3)}\|
+\mE\|\hSigma_{n,\mathrm{OL}}^{(3)}-\Sigma\|.
\end{align*}
The first three terms measure the additional errors introduced by the manifold and vanish in the~standard Euclidean setting. The last term is the batch-means estimation error that must still be controlled even in Euclidean space. We begin with the first three terms.
\begin{itemize}
\item From $\hSigma_{n,\mathrm{OL}}^{(0)}$ to $\hSigma_{n,\mathrm{OL}}^{(1)}$, we replace the recursively updated vector $\tb_{j,n}$ with the exact deviation $b_{j,n}=R_{\wt x_n}^{-1}(c_j)$. This isolates the approximation introduced by the online update in \eqref{eq:batch-vector-update}. Proposition~\ref{prop:ol-online-vector-update} bounds the resulting covariance error.

\item From $\hSigma_{n,\mathrm{OL}}^{(1)}$ to $\hSigma_{n,\mathrm{OL}}^{(2)}$, we compare two ways of expressing the deviation of $c_j$ from $\wt x_n$ in the fixed tangent space $\T_{x_\star}\mM$:
\begin{equation*}
(\Gamma_{x_\star}^{\wt x_n})^{-1}b_{j,n}=(\Gamma_{x_\star}^{\wt x_n})^{-1}R_{\wt x_n}^{-1}(c_j),
\qquad
\tDelta_j^B-\tDelta_n=R_{x_\star}^{-1}(c_j)-R_{x_\star}^{-1}(\wt x_n).
\end{equation*}
The first forms the deviation using the inverse retraction at $\wt x_n$ and then transports it to $\T_{x_\star}\mM$. The second maps both averages to $\T_{x_\star}\mM$ using $R_{x_\star}^{-1}$ before subtracting the resulting vectors. Proposition~\ref{prop:ol-coordinate-error} bounds the resulting covariance error.

\item From $\hSigma_{n,\mathrm{OL}}^{(2)}$ to $\hSigma_{n,\mathrm{OL}}^{(3)}$, we replace the mapped manifold averages $(\tDelta_j^B,\tDelta_n)$ with the arithmetic averages $(\barDelta_j^B,\barDelta_n)$ in \eqref{eq:tangent-averages}. This isolates the difference between averaging on the manifold and averaging the tangent vectors $\Delta_k$. Proposition~\ref{prop:ol-manifold-tangent-error} bounds the resulting covariance error.
\end{itemize}

The propositions above bound each of these three errors by a constant multiple of $n^{1/2-1/(2\beta)-\alpha}$. It remains to control the batch-means error $\mE\|\hSigma_{n,\mathrm{OL}}^{(3)}-\Sigma\|$. We modify the Euclidean batch-means analysis of \citet{Zhu2023Online}  and show in Proposition~\ref{prop:ol-euclidean-bm} that 
\begin{equation*}
\mE\|\hSigma_{n,\mathrm{OL}}^{(3)}-\Sigma\|\lesssim n^{-1/(2\beta)}+n^{(\alpha-1)/2+1/(2\beta)}.
\end{equation*}
Since $\alpha>1/2$, we have $n^{1/2-1/(2\beta)-\alpha}=o(n^{-1/(2\beta)})$, so the errors introduced by the manifold do not affect the convergence rate. Combining the four bounds establishes the OL result in Theorem~\ref{thm:covariance-rate}.

\section{Online statistical inference on manifolds}\label{sec:4}

In this section, we combine the asymptotic normality result in Theorem~\ref{thm:rasgd-asymptotic-normality} with our  covariance estimator to perform online statistical inference on manifolds. We first present a general strategy for constructing confidence intervals and regions when operating on a Riemannian manifold, then specialize it to streaming principal component analysis on a Grassmann manifold.

\subsection{Confidence intervals and regions}\label{sec:4.1}

We first construct a confidence interval for a scalar quantity $h(x_\star)$, where $h:\mM\to\mR$ is a smooth function. For example, in Euclidean setting where $\cM = \RR^d$, the function $h(x)=e_i^\top x$ selects the $i$-th coordinate of $x_\star$. We estimate $h(x_\star)$ by $h(\wt x_n)$. For a prescribed level $1-\delta$, with $\delta\in(0,1)$, we seek an interval whose probability of containing $h(x_\star)$ approaches $1-\delta$ as $n$ increases.

For either covariance estimator $\wt\Sigma_{n,\diamond}$, with $\diamond\in\{\mathrm{OL},\mathrm{NOL}\}$, we estimate the asymptotic variance~by
\begin{equation}\label{eq:scalar-variance}
\widehat\sigma_{n,h,\diamond}^2:=\langle\grad h(\wt x_n),\wt\Sigma_{n,\diamond}\grad h(\wt x_n)\rangle_{\wt x_n}.
\end{equation}
The gradient accounts for the first-order change in $h$ along tangent directions. Let $\widehat\sigma_{n,h,\diamond}\geq0$ denote the square root of this variance estimate. Then $\widehat\sigma_{n,h,\diamond}/\sqrt n$ is the estimated standard error of $h(\wt x_n)$, which determines the width of the confidence interval
\begin{equation}\label{eq:scalar-confidence-interval}
\mathcal I_{n,h,\diamond}(1-\delta) := \left[ h(\wt x_n)-z_{1-\delta/2}\frac{\widehat\sigma_{n,h,\diamond}}{\sqrt n}, \ h(\wt x_n)+z_{1-\delta/2}\frac{\widehat\sigma_{n,h,\diamond}}{\sqrt n} \right],
\end{equation}
where $z_{1-\delta/2}$ is the $(1-\delta/2)$-quantile of the standard normal distribution.

For joint inference, we seek a confidence region for $x_\star$ itself, consisting of points on $\mM$. Recall that $\wt\Sigma_{n,\diamond}$ acts on $\T_{\wt x_n}\mM$. We therefore first construct an ellipsoid of tangent vectors around $0_{\wt x_n}$, with its shape determined by the covariance estimate. To obtain a region on the manifold, we apply the retraction $R_{\wt x_n}$ to this ellipsoid. When $\wt\Sigma_{n,\diamond}$ is positive definite, the two sets are
\begin{equation}\label{eq:confidence-sets}
\begin{aligned}
\mathcal E_{n,\diamond}(1-\delta)&:=\left\{v\in\T_{\wt x_n}\mM:n\langle v,\wt\Sigma_{n,\diamond}^{-1}v\rangle_{\wt x_n}\leq\chi^2_{d,1-\delta}\right\},\\
\mathcal C_{n,\diamond}(1-\delta)&:=R_{\wt x_n}\{\mathcal E_{n,\diamond}(1-\delta)\}.
\end{aligned}
\end{equation}
Here, $d=\dim\mM$ and $\chi^2_{d,1-\delta}$ is the $(1-\delta)$-quantile of the chi-squared distribution with $d$ degrees of freedom. Figure~\ref{fig:retraction-confidence-region} illustrates this construction.

\begin{figure}[H]
\centering
\includegraphics[width=\textwidth]{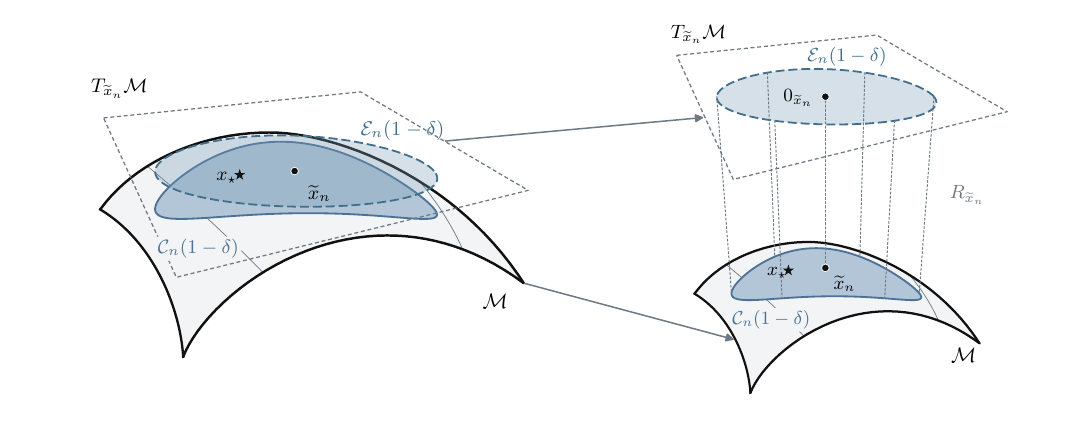}
\caption{Construction of a confidence region on the manifold. The retraction $R_{\wt x_n}$ maps the tangent ellipsoid $\mathcal E_{n,\diamond}(1-\delta)$ to the shaded region $\mathcal C_{n,\diamond}(1-\delta)$, and its center $0_{\wt x_n}$ to $\wt x_n$. The right panel separates the tangent space from the manifold to show corresponding points, connected by the light gray dashed lines. The estimator subscript $\diamond$ is omitted in the figure.}
\label{fig:retraction-confidence-region}
\end{figure}

In Euclidean space, with $R_x(v)=x+v$ and $h(x)=e_i^\top x$, the interval in \eqref{eq:scalar-confidence-interval} and the region in \eqref{eq:confidence-sets} reduce to the coordinatewise interval and ellipsoidal region in \citet[Corollary~3.2]{Zhu2023Online}.

The following proposition combines the asymptotic normality in Theorem~\ref{thm:rasgd-asymptotic-normality} with the covariance consistency in Theorem~\ref{thm:covariance-rate} to establish asymptotically correct coverage for both constructions.The proof is deferred to Appendix~\ref{app:E}.

\begin{proposition}\label{prop:confidence-procedures}
Let $\mM$ be a Riemannian manifold with dimension $d$. Under the conditions of Theorem~\ref{thm:covariance-rate}, fix $\delta\in(0,1)$ and let $h:\mM\to\mR$ be continuously differentiable on an open neighborhood of $\mX$, with $\grad h(x_\star)\ne0_{x_\star}$. For $\diamond\in\{\mathrm{OL},\mathrm{NOL}\}$, consider the interval in \eqref{eq:scalar-confidence-interval} and the region in \eqref{eq:confidence-sets}, setting $\mathcal C_{n,\diamond}(1-\delta)=\{\wt x_n\}$ if $\wt\Sigma_{n,\diamond}$ is singular. Then, as $n\to\infty$,
\begin{equation*}
\begin{aligned}
\mP\{h(x_\star)\in\mathcal I_{n,h,\diamond}(1-\delta)\}&\longrightarrow1-\delta,\\
\mP\{x_\star\in\mathcal C_{n,\diamond}(1-\delta)\}&\longrightarrow1-\delta.
\end{aligned}
\end{equation*}
\end{proposition}

We next specialize these inference procedures to streaming principal component analysis.

\subsection{Streaming principal component analysis (PCA)}\label{sec:4.2}

Streaming $k$-PCA estimates a leading $k$-dimensional principal subspace from sequentially observed data \citep{Jain2016Streaming,AllenZhu2017First}. Given independent and identically distributed symmetric positive semidefinite matrices $\mathsf H_n\in\mR^{p\times p}$, the goal is to estimate the subspace spanned by the eigenvectors corresponding to the $k$ largest eigenvalues of $H=\mE\mathsf H_n$, where $1\leq k<p$. We focus on the standard outer-product case $\mathsf H_n=\xi_n\xi_n^\top$, where $\{\xi_n\}_{n\geq1}$ are independent samples from a common distribution on $\mR^p$ with finite second moment, so $H=\mE(\xi_n\xi_n^\top)$. Write $H=Q\diag(\lambda_1,\ldots,\lambda_p)Q^\top$, where $Q$ is orthogonal and $\lambda_1\geq\cdots\geq\lambda_p\geq0$, and let $Q_1\in\mR^{p\times k}$ contain the first $k$ columns of $Q$. The columns of $Q_1$ span the subspace we seek to estimate. Different orthonormal bases can represent the same subspace, and the particular choice does not matter. So, this problem can be formulated on the Grassmann manifold $G_{p,k}$, whose points are the $k$-dimensional subspaces of $\mR^p$; see \citet[Section~4.1]{Bonnabel2013Stochastic} and \citet[Section~6.2]{Tripuraneni2018Averaging}.

The dimension of $G_{p,k}$ is $k(p-k)$. We write a point in  $G_{p,k}$ as $[X]$, where $X\in\mR^{p\times k}$ has orthonormal columns and spans $[X]$. The orthogonal projection of an observation $\xi_n$ onto $[X]$ is $XX^\top\xi_n$, with squared length $\|X^\top\xi_n\|^2$. PCA seeks the subspace that minimizes the expected squared norm of the residual $\xi_n-XX^\top\xi_n$. Equivalently, it maximizes $\mE\|X^\top\xi_n\|^2=\tr(X^\top HX)$. The corresponding population optimization problem is
\begin{equation}\label{eq:pca-objective}
\min_{[X]\in G_{p,k}}F([X]):=-\frac12\tr(X^\top HX) =\mE\left[-\frac12\|X^\top\xi_n\|^2\right].
\end{equation}
Under the eigengap condition $\lambda_k>\lambda_{k+1}$, the leading principal subspace is unique, and the minimizer is $x_\star=[Q_1]$.

We now specify the geometric maps and Riemannian gradient needed to implement RSGD and Algorithm~\ref{alg:online-covariance} on $G_{p,k}$.
\begin{enumerate}
\item \textbf{Tangent space and Riemannian metric.} We equip $G_{p,k}$ with its standard Riemannian metric; see \citet[Section~9.16]{Boumal2023Introduction}. For a fixed orthonormal representative $X$, the tangent space, inner product, and orthogonal projection $\Pi_X:\mR^{p\times k}\to\T_{[X]}G_{p,k}$ have the following matrix representations:
\begin{equation}\label{eq:grassmann-metric}
\begin{aligned}
\T_{[X]}G_{p,k}&=\{V\in\mR^{p\times k}:X^\top V=0\},\\
\langle V,W\rangle_{[X]}&=\tr(V^\top W), \qquad \Pi_XZ=(I_p-XX^\top)Z.
\end{aligned}
\end{equation}
Here $V,W\in\T_{[X]}G_{p,k}$ and $Z\in\mR^{p\times k}$. Although the matrix representations depend on the choice of $X$, the inner product of the corresponding tangent vectors does not.
\item \label{item:grassmann-retraction}\textbf{Retraction and inverse retraction.} For a tangent vector $V\in\T_{[X]}G_{p,k}$, the columns of $X+V$ need not be orthonormal. We orthonormalize them using the polar retraction $R_{[X]}:\T_{[X]}G_{p,k}\to G_{p,k}$:
\begin{equation}\label{eq:grassmann-retraction}
R_{[X]}(V)=\left[(X+V)(I_k+V^\top V)^{-1/2}\right].
\end{equation}
The factor $(I_k+V^\top V)^{-1/2}$ changes the basis but not the span of $X+V$. As established in \citet[Example~9.59]{Boumal2023Introduction}, this retraction is second order.

The averaging and covariance updates also require the local inverse retraction. For a nearby subspace $[Y]$ with orthonormal representative $Y_0$, write $X^\top Y_0=KO$ in polar form, where $K\succ0$ and $O$ is orthogonal. The factor $K$ is unique because it is the unique positive definite square root of $X^\top Y_0Y_0^\top X$. With the aligned representative $Y=Y_0O^\top$, the inverse retraction takes the form
\begin{equation*}
R_{[X]}^{-1}([Y])=YK^{-1}-X.
\end{equation*}
\item \textbf{Vector transport and its adjoint.} The covariance updates require transporting tangent vectors between the tangent spaces of successive averages. With the aligned representatives $X,Y$ and the matrix $K$ specified in item~\ref{item:grassmann-retraction}, the vector transport in \eqref{eq:transport-maps} and its adjoint are represented~by
\begin{equation*}
\Lambda_{[X]}^{[Y]}(W)=\Pi_YWK, \qquad (\Lambda_{[X]}^{[Y]})^\ast(Z)=\Pi_XZK.
\end{equation*}
Here $W\in\T_{[X]}G_{p,k}$ and $Z\in\T_{[Y]}G_{p,k}$. The transport maps tangent vectors at $[X]$ to tangent vectors at $[Y]$, while its adjoint acts in the opposite direction.
\item \textbf{Riemannian gradient.} For the sample loss $f([X],\xi):=-\|X^\top\xi\|^2/2$ in \eqref{eq:pca-objective}, the Euclidean gradient with respect to $X$ is $-\xi\xi^\top X$. Under the metric above, the Riemannian gradient is its projection onto the tangent space: $\grad f([X],\xi)=-\Pi_X\xi\xi^\top X$.
\end{enumerate}

We write the RSGD iterate and its recursive average as $x_n=[X_n]$ and $\wt x_n=[\wt X_n]$, respectively. With the components above, the RSGD and averaging updates in \eqref{equ:ave_RSGD} become
\begin{equation}\label{eq:pca-updates}
[X_n]=R_{[X_{n-1}]}\left\{\gamma_n\Pi_{X_{n-1}}\xi_n\xi_n^\top X_{n-1}\right\},\qquad [\wt X_n]=R_{[\wt X_{n-1}]}\left\{\frac1nR_{[\wt X_{n-1}]}^{-1}([X_n])\right\}.
\end{equation}
Algorithm~\ref{alg:online-covariance} uses these iterates and the geometric maps above to compute the OL and NOL covariance estimators on $\T_{[\wt X_n]}G_{p,k}$. We next use these estimates to construct confidence regions for the principal subspace and confidence intervals for its scalar features.

\paragraph{Confidence regions for the principal subspace.} 
To construct a confidence region for $x_\star=[Q_1]$, we form the tangent ellipsoid $\mathcal E_{n,\diamond}(1-\delta)$ in \eqref{eq:confidence-sets} at $\wt x_n=[\wt X_n]$, with $d=k(p-k)$. The polar retraction maps this ellipsoid to a set of subspaces. Since the orthonormalization in \eqref{eq:grassmann-retraction} does not change the column span, the resulting region is
\begin{equation*}
\mathcal C_{n,\diamond}(1-\delta) = \left\{ \operatorname{span}(\wt X_n+V): V\in\mathcal E_{n,\diamond}(1-\delta) \right\}.
\end{equation*}
Here we represent each tangent vector $V$ by a matrix in $\mR^{p\times k}$ satisfying $\wt X_n^\top V=0$, and $\operatorname{span}(\wt X_n+V)$ denotes the subspace spanned by the columns of $\wt X_n+V$. As $V$ ranges over the tangent ellipsoid, the subspaces $\operatorname{span}(\wt X_n+V)$ form the confidence region. This region is independent of the choice of orthonormal representative $\wt X_n$.

Figure~\ref{fig:grassmann-g31-confidence-region} illustrates this construction on $G_{3,1}$, whose points are lines through the origin in $\mR^3$. As shown in the left panel, each line intersects the unit sphere $\mS^2$ at two opposite points $z$ and $-z$, either of which represents the same subspace. For lines near the estimate $\wt x_n=[\wt X_n]$, we choose the representative satisfying $\wt X_n^\top z>0$. This represents each such line by a single point in the hemisphere centered at $\wt X_n$. The right panel shows the confidence-region construction using these representatives. The tangent ellipsoid $\mathcal E_{n,\diamond}(1-\delta)$ is an ellipse in the tangent plane $\T_{\wt X_n}\mS^2$. For each vector $v$ in this ellipse, the polar retraction in \eqref{eq:grassmann-retraction} divides $\wt X_n+v$ by its length to obtain the unit vector $z=(\wt X_n+v)/\sqrt{1+\|v\|^2}$ and returns the line $[z]=[\wt X_n+v]$. As $v$ ranges over the ellipse, these representatives form the shaded spherical cap. The corresponding lines form the confidence region $\mathcal C_{n,\diamond}(1-\delta)$.

\begin{figure}[H]
\centering
\includegraphics[width=\textwidth]{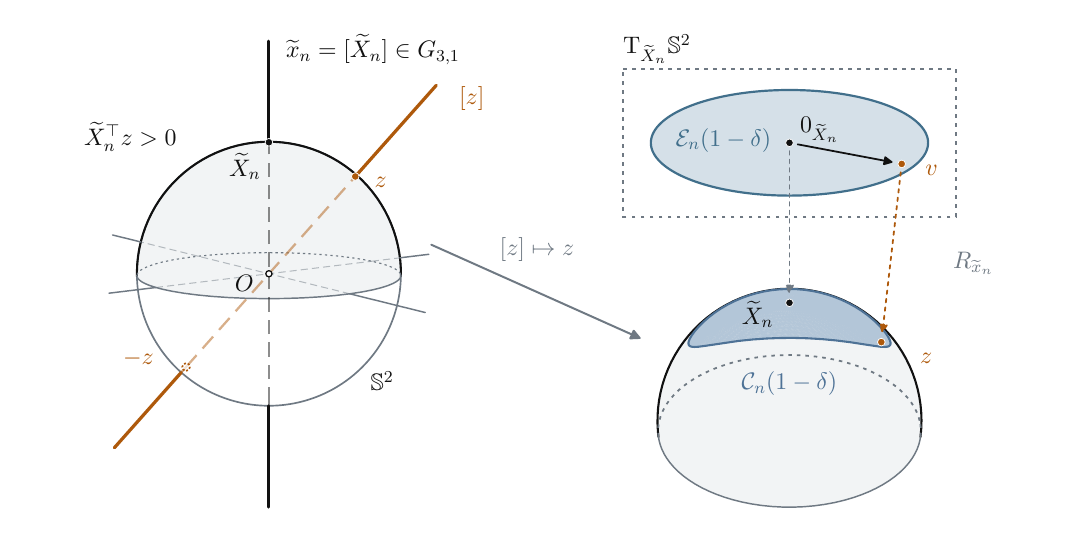}
\caption{Spherical representation of a confidence region on $G_{3,1}$. Left: the black and orange lines represent the estimate $[\wt X_n]$ and a subspace $[z]$ in the confidence region, respectively. Right: the polar retraction maps each tangent vector $v$ to the line represented by $z=(\wt X_n+v)/\sqrt{1+\|v\|^2}$. These unit representatives form the shaded spherical cap depicting the confidence region. The subscript $\diamond$ is omitted.}
\label{fig:grassmann-g31-confidence-region}
\end{figure}

\paragraph{Confidence intervals for a scalar feature.} 
To illustrate the scalar confidence-interval construction in Section~\ref{sec:4.1}, we consider how closely a fixed unit vector $q\in\mR^p$ aligns with the unknown principal subspace $x_\star=[Q_1]$. Our target is $\|Q_1^\top q\|^2$, the squared length of the projection of $q$ onto this subspace. We express it as a smooth function on $G_{p,k}$ by defining
\begin{equation}\label{eq:pca-target}
h_q([X]):=q^\top XX^\top q=\|X^\top q\|^2.
\end{equation}
Because $XX^\top$ is the orthogonal projector onto $[X]$, the function $h_q$ does not depend on the choice of representative $X$. Its value lies in $[0,1]$, with larger values indicating closer alignment. The target  $\|Q_1^\top q\|^2$ can be written as $h_q(x_\star)$, and we estimate it by $h_q([\wt X_n])$. To quantify the uncertainty of this estimate, we use the Riemannian gradient of $h_q$, represented by
\begin{equation}\label{eq:pca-target-gradient}
\grad h_q([X])=2(I_p-XX^\top)qq^\top X.
\end{equation}
Evaluating \eqref{eq:pca-target-gradient} at $[\wt X_n]$ and combining it with either the OL or NOL covariance estimator in \eqref{eq:scalar-variance} yields the estimated asymptotic variance. This determines the width of the confidence interval in \eqref{eq:scalar-confidence-interval}, centered at $h_q([\wt X_n])$. In Section~\ref{sec:5}, we will evaluate the empirical coverage and length of these intervals through numerical experiments on Gaussian streaming PCA.

\section{Numerical experiments}\label{sec:5}

This section examines the finite-sample performance of the OL and NOL covariance estimators in \eqref{eq:ol-estimator} and \eqref{eq:nol-estimator} using the streaming $k$-PCA model from Section~\ref{sec:4.2} with Gaussian observations. 
In this model, the limiting covariance is available in closed form, allowing us to quantify covariance estimation errors and compare the resulting confidence intervals with oracle intervals constructed using the true asymptotic variance. We first describe the experimental setup in Section~\ref{sec:5.1}, then compare covariance estimation errors in Section~\ref{sec:5.2}, and finally examine confidence interval coverage and length in Section~\ref{sec:5.3}. 
We defer additional covariance estimation and statistical inference results for more settings to Appendix~\ref{app:F}.

\subsection{Experimental setup}\label{sec:5.1}

Recall from Section~\ref{sec:4.2} that streaming $k$-PCA uses observations $\xi_n\in\mR^p$ that arrive sequentially to estimate the subspace spanned by the eigenvectors associated with the $k$ largest eigenvalues of $H=\mE(\xi_n\xi_n^\top)$. The parameter space is the Grassmann manifold $G_{p,k}$, whose dimension is $d=k(p-k)$. In our experiments, we generate i.i.d. observations $\xi_1,\xi_2,\ldots\sim\mathcal N(0,H)$, where $H=Q\diag(\lambda_1,\ldots,\lambda_p)Q^\top$ with a fixed orthogonal matrix $Q$, and set $k=\lfloor\sqrt p\rfloor$ in the streaming $k$-PCA problem. Following the choice of eigenvalues in \citet[Section~7]{Tripuraneni2018Averaging}, we choose $\Delta_\lambda>0$ and set
\begin{equation*}
\lambda_j=j^{-1}+\Delta_\lambda,\quad 1\leq j\leq k, \qquad \lambda_j=(j-1)^{-1},\quad k<j\leq p.
\end{equation*}
Thus, the eigengap is $\lambda_k-\lambda_{k+1}=\Delta_\lambda$. A smaller eigengap makes the principal subspace more sensitive to sampling noise and therefore harder to estimate accurately.

At iteration $n$, we use the new observation $\xi_n$ to update the RSGD iterate and its recursive average using \eqref{eq:pca-updates}, with the polar retraction in \eqref{eq:grassmann-retraction}. Following \citet[Section~4]{Zhu2023Online}, we set $\alpha=0.505$ and use step size $\gamma_n=0.5n^{-\alpha}$. The batch starting points are $a_1=1$ and $a_m=\lfloor cm^\beta\rfloor$ for $m\geq2$, with $\beta=2/(1-\alpha)$ as recommended in Corollary~\ref{cor:recommended-beta}. The constant $c$ controls the spacing between successive batch starts. We vary $c\in\{1,2,5\}$ for covariance estimation to examine the effect of this~choice. Each run contains $N=3\times 10^5$ observations. All results are based on $200$ independent runs for each setting. For each method and setting, the initialization is fixed across runs.

We first give an explicit representation of the limiting covariance $\Sigma=A^{-1}SA^{-1}$ from Theorem~\ref{thm:rasgd-asymptotic-normality}. Recall from Section~\ref{sec:4.2} that $x_\star=[Q_1]$ is the population principal subspace, where $Q_1$ contains the first $k$ columns of $Q$ and $[Q_1]$ denotes their span. At $x_\star=[Q_1]$, $\Sigma$ acts on the tangent space $\T_{x_\star}G_{p,k}$, which has dimension $d=k(p-k)$. We therefore represent $\Sigma$ by a $d\times d$ matrix in an orthonormal basis. We choose the orthonormal basis matrices $q_ie_j^\top$, indexed by $1\leq j\leq k<i\leq p$, where $q_i$ is column $i$ of $Q$ and $e_j$ is the $j$th coordinate vector in $\mR^k$. Since the columns of $Q$ and the vectors $e_j$ are orthonormal, these $d$ matrices lie in the tangent~space and form an orthonormal basis under the inner product in \eqref{eq:grassmann-metric}.

In this basis, we compute the Hessian $A=\Hess F(x_\star)$ from the objective in \eqref{eq:pca-objective}. Differentiating its Riemannian gradient at the matrix representative $X=Q_1$ of $x_\star$ and projecting the result onto $\T_{x_\star}G_{p,k}$ gives $A[V]=V(Q_1^\top HQ_1)-HV$ for $V\in\T_{x_\star}G_{p,k}$. Since $Hq_i=\lambda_iq_i$, we have $A[q_ie_j^\top]=(\lambda_j-\lambda_i)q_ie_j^\top$. This identity shows that each basis vector is an eigenvector of $A$, so the matrix of $A$ is diagonal in this basis, with diagonal entries~$\lambda_j-\lambda_i$.

We next compute the gradient-noise covariance $S$ in the same basis. The gradient noise $\grad f(x_\star,\xi_n)$ has coordinate $-(q_i^\top\xi_n)(q_j^\top\xi_n)$ along $q_ie_j^\top$. Since $Q^\top\xi_n$ has independent centered Gaussian entries with variances $\lambda_1,\ldots,\lambda_p$, this coordinate has mean zero and variance $\mE[(q_i^\top\xi_n)^2(q_j^\top\xi_n)^2]=\lambda_i\lambda_j$. For two distinct noise coordinates, at least one centered Gaussian factor appears only once in their product; thus, independence then implies zero covariance between them. Finally, the gradient-noise covariance $S$ is diagonal with entries $\lambda_i\lambda_j$. Substituting the diagonal entries of $A$ and $S$ into $\Sigma=A^{-1}SA^{-1}$,~we~know
\begin{equation}\label{eq:experimental-oracle-covariance}
\Sigma=\diag(\sigma_{ij}^2),\qquad \sigma_{ij}^2 =\frac{\lambda_i\lambda_j}{(\lambda_j-\lambda_i)^2}, \qquad 1\leq j\leq k<i\leq p.
\end{equation}
Note that this covariance is used only to evaluate estimation errors and construct oracle intervals. It is not used in the OL or NOL updates.

\subsection{Covariance estimation}\label{sec:5.2}

We first compare the OL and NOL covariance estimators on $G_{3,1}$, $G_{7,2}$, and $G_{10,3}$, using eigengaps $\Delta_\lambda\in\{1,0.25\}$ and batch constants $c\in\{1,2,5\}$. To compare each estimate with the limiting covariance $\Sigma$ in \eqref{eq:experimental-oracle-covariance}, we parallel transport the estimate from $\T_{\wt x_n}G_{p,k}$ to $\T_{x_\star}G_{p,k}$, as in Theorem~\ref{thm:covariance-rate}, and compute the operator-norm error in the orthonormal basis described above. Figure~\ref{fig:final-covariance-convergence} plots the logarithm of the average error over $200$ runs against $\log n$. The black dashed lines have slope $-1/8$, which is close to the exponent $-(1-\alpha)/4=-0.12375$ in Corollary~\ref{cor:recommended-beta}, and are included as reference lines.

\begin{figure}[thp]
\centering
\includegraphics[width=\textwidth]{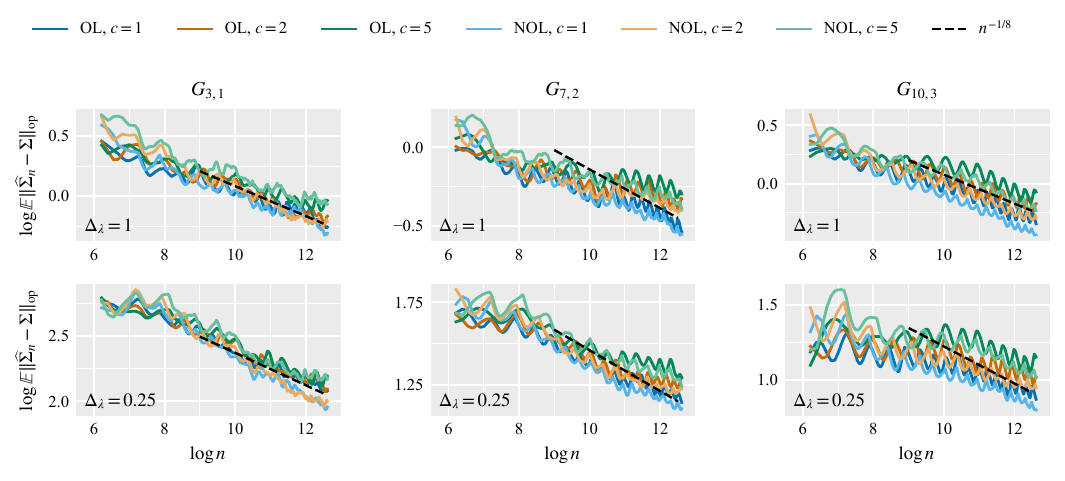}
\caption{Gaussian streaming PCA: log mean covariance error (operator norm) against log sample size. Columns correspond to $G_{3,1}$, $G_{7,2}$, and $G_{10,3}$, and the top and bottom rows use eigengaps $1$ and $0.25$, respectively. The colored curves compare OL and NOL with $c\in\{1,2,5\}$. Each mean is based on $200$ independent runs. The black dashed lines indicate the reference slope $-1/8$.}
\label{fig:final-covariance-convergence}
\end{figure}

Figure~\ref{fig:final-covariance-convergence} shows that the errors of both estimators generally decrease as more observations are received, despite some fluctuations along the trajectories. The overall slopes of the log-log curves are close to the reference slope $-1/8$, consistent with the convergence rate in Corollary~\ref{cor:recommended-beta}. Comparing the two rows of Figure~\ref{fig:final-covariance-convergence}, we also observe larger absolute errors for the smaller eigengap. However, the magnitude of $\Sigma$ in \eqref{eq:experimental-oracle-covariance} also increases as the eigengap decreases, so the larger absolute errors do not necessarily indicate lower relative accuracy.

We next compare the different choices of $c$ and the two covariance estimators. At $n=3\times 10^5$, $c=1$ yields the smallest mean error among the three batch constants for both estimators across all settings considered. Based on these results, we use $c=1$ for both estimators in the confidence interval experiments below.

\subsection{Statistical inference}\label{sec:5.3}

We next examine the empirical coverage and mean length of the nominal $95\%$ confidence intervals in \eqref{eq:scalar-confidence-interval}. We consider $G_{10,3}$ and $G_{20,4}$, with $\Delta_\lambda=0.25$ and $c=1$. For each manifold, we draw $v\sim\mathcal N(0,H)$ independently of the RSGD streams, set $q=v/\|v\|$, and hold $q$ fixed across runs. Recall from Section~\ref{sec:4.2} that $h_q(x_\star)=\|Q_1^\top q\|^2$ in \eqref{eq:pca-target} is the squared length of the projection of $q$ onto the principal subspace $x_\star$. We estimate this quantity by $h_q(\wt x_n)$ and construct its confidence interval using either the OL or NOL covariance estimate.

For comparison, we also construct oracle intervals by using the true scalar asymptotic variance $\sigma_{h_q}^2=\langle\grad h_q(x_\star),\Sigma\grad h_q(x_\star)\rangle_{x_\star}$ in the confidence interval formula \eqref{eq:scalar-confidence-interval}, where $\Sigma$ is given by \eqref{eq:experimental-oracle-covariance}. We report the empirical coverage, defined as the fraction of the $200$ intervals that contain $h_q(x_\star)$, together with the mean interval length. Figure~\ref{fig:final-scalar-inference} displays both quantities as the sample size increases. Table~\ref{tab:scalar-inference} reports the coverage at four sample sizes, with Monte Carlo standard errors $\sqrt{\wh p(1-\wh p)/200}$~shown in parentheses, where $\wh p$ denotes the empirical coverage.

\begin{figure}[t!]
\centering
\includegraphics[width=\textwidth]{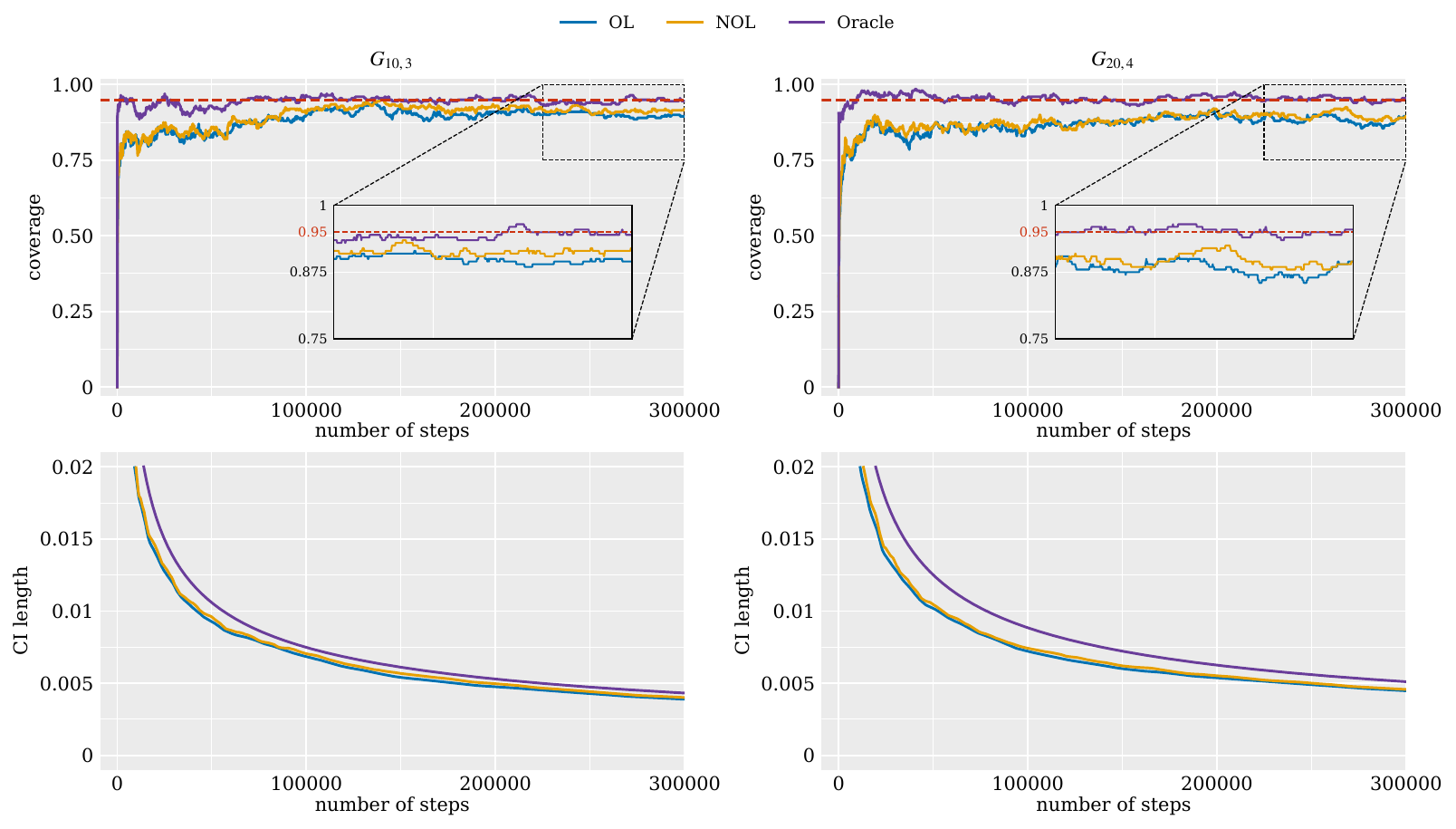}
\caption{Gaussian streaming PCA: empirical coverage (top) and mean length (bottom) of $95\%$ confidence intervals for $h_q(x_\star)$. Columns correspond to $G_{10,3}$ and $G_{20,4}$, with $\Delta_\lambda=0.25$ and $c=1$. The blue, orange, and purple curves show OL, NOL, and Oracle, respectively, with results based on $200$ independent runs. The small plots inside the top panels show the later coverage values, and the red dashed lines mark the nominal coverage level $0.95$.}
\label{fig:final-scalar-inference}
\end{figure}

\begin{table}[thp]
\centering
\small
\caption{Empirical coverage of nominal $95\%$ confidence intervals for $h_q(x_\star)$ in Gaussian streaming PCA, with $\Delta_\lambda=0.25$ and $c=1$. Results are based on 200 independent runs, with Monte Carlo standard errors in parentheses. Here $d=k(p-k)$ is the dimension of $G_{p,k}$.}
\label{tab:scalar-inference}
\setlength{\tabcolsep}{5pt}
\renewcommand{\arraystretch}{1.20}
\begin{tabular}{@{}l|rrrr@{}}
\hline
\multicolumn{5}{c}{Empirical coverage} \\
\hline
$G_{10,3}$ ($d=21$) & $n=50{,}000$ & $n=100{,}000$ & $n=200{,}000$ & $n=300{,}000$ \\
\hline
OL & 0.835 (0.02625) & 0.905 (0.02073) & 0.910 (0.02024) & 0.895 (0.02168) \\
NOL & 0.860 (0.02454) & 0.905 (0.02073) & 0.930 (0.01804) & 0.920 (0.01918) \\
Oracle & 0.930 (0.01804) & 0.965 (0.01300) & 0.955 (0.01466) & 0.945 (0.01612) \\
\hline
$G_{20,4}$ ($d=64$) & $n=50{,}000$ & $n=100{,}000$ & $n=200{,}000$ & $n=300{,}000$ \\
\hline
OL & 0.860 (0.02454) & 0.860 (0.02454) & 0.890 (0.02212) & 0.895 (0.02168) \\
NOL & 0.865 (0.02416) & 0.865 (0.02416) & 0.915 (0.01972) & 0.895 (0.02168) \\
Oracle & 0.955 (0.01466) & 0.960 (0.01386) & 0.960 (0.01386) & 0.955 (0.01466) \\
\hline
\end{tabular}
\end{table}

Figure~\ref{fig:final-scalar-inference} shows that OL and NOL achieve similar coverage on both manifolds. For example, at $n=3\times 10^5$, Table~\ref{tab:scalar-inference} shows that the OL and NOL coverage rates are $0.895$ and $0.920$ on $G_{10,3}$, respectively, while both methods attain coverage of $0.895$ on $G_{20,4}$. Consistent with the observations in \citet{Zhu2023Online}, although these coverage rates are below the nominal $95\%$ level, both methods achieve coverage of about $90\%$ at this sample size, indicating reasonable finite-sample performance in these higher-dimensional settings with the small eigengap $\Delta_\lambda=0.25$. The mean interval lengths decrease as more observations are received and are similar for OL and NOL. Additional experimental results for different choices of dimension and eigengap are reported in Appendix~\ref{app:F}.

\section{Conclusion and future directions}\label{sec:6}

In this paper, we studied asymptotic covariance estimation for averaged Riemannian stochastic gradient descent (RSGD). Specifically, we extended the Euclidean online batch-means estimator of \citet{Zhu2023Online} to Riemannian manifolds, developing both overlapping (OL) and non-overlapping (NOL) versions. To our knowledge, these are the first covariance estimators of this type for averaged RSGD. Both estimators use only the RSGD iterates and require neither Riemannian Hessian calculations nor a prespecified terminal sample size. Our analysis further shows that the geometric errors induced by the manifold structure are of smaller order than the leading batch-means error. Consequently, under standard assumptions, an appropriate choice of the batch-size control sequence yields the convergence rate $n^{-(1-\alpha)/4}$, matching that of the Euclidean estimator of \citet{Zhu2023Online}. Combining this consistency result with the asymptotic normality established by \citet{Tripuraneni2018Averaging}, we obtained asymptotically valid confidence intervals for smooth scalar functions of $x_\star$ and joint confidence regions for $x_\star$. We illustrated these online inference procedures through streaming PCA on Grassmann manifolds. In this setting, our Gaussian experiments support the consistency of the proposed covariance estimators~and demonstrate their use in constructing confidence intervals in higher-dimensional~settings.

One direction for future research is to develop online covariance estimators for averaged RSGD with faster convergence rates. For example, extending the bias-reduced estimator of \citet{Wei2026Refining} to averaged RSGD could yield a faster rate, as the Euclidean estimator achieves $n^{-(1-\alpha)/2}\sqrt{\log n}$ under stronger moment conditions. As $\alpha\downarrow1/2$, the polynomial factor in this rate approaches $n^{-1/4}$, compared with $n^{-1/8}$ in our bound. Faster covariance estimation may also improve the finite-sample performance of the resulting inference procedures. Another direction is to establish lower bounds for covariance estimation based on RSGD iterates and characterize the minimax optimal convergence~rate on Riemannian manifolds.

\section*{Disclosure of AI usage}

The methodology, results, and proofs were developed by the authors. All writing was done by the authors. OpenAI's GPT 5.6 Sol and some earlier models were used to assist with the proofs of some technical results (Lemmas~\ref{lem:block-calculus} and~\ref{lem:D-online-endpoint}), to help write code for the simulations, and to proofread the final draft of the paper.

\bibliographystyle{my-plainnat}
\bibliography{ref}

\appendix
\numberwithin{equation}{section}

\input{appendix}

\end{document}

%% file: appendix.tex
\newpage
\begin{center}
\textbf{\large Supplemental Materials: \\ Online Covariance Estimation for Averaged Stochastic Gradient Descent\\ on Riemannian Manifolds}
\end{center}

The supplementary material is organized as follows. In Appendix~\ref{app:A}, we present some useful technical lemmas. Appendices~\ref{app:B} and~\ref{app:C} provide estimates for geometric maps and the RSGD iterates, respectively. In Appendix~\ref{app:D}, we prove Theorem~\ref{thm:covariance-rate}, our main theorem for the OL and NOL covariance estimators, and derive Corollary~\ref{cor:recommended-beta}. Appendix~\ref{app:E} proves the statistical inference results in Proposition~\ref{prop:confidence-procedures}. Additional numerical experiments are presented in Appendix~\ref{app:F}. Throughout the appendices, $C$ denotes a finite positive constant whose value may change between displays.

\section{Technical lemmas}\label{app:A}

\begin{lemma}\label{lem:linear-recursion-products}
Let $A\in\mathbb R^{d\times d}$ be symmetric and positive definite. For $i\in\mathbb{Z}_{\ge0}$, define $Y_i^i:=I_d$ and, for $j>i$, set
\begin{equation*}
Y_i^j:=\prod_{k=i+1}^{j}(I_d-\gamma_kA),
\end{equation*}
where $\gamma_k:=\gamma k^{-\alpha}$, $\gamma>0$ and $\alpha\in(1/2,1)$. There exists a constant $C>0$ such that, for all $0\le i\le j$,
\begin{equation}\label{eq:A-product-bound}
\|Y_i^j\|\le C\exp\left(-\gamma\lambda_{\min}(A)\sum_{k=i+1}^{j}k^{-\alpha}\right)\le C\exp\left[-\frac{\gamma\lambda_{\min}(A)}{1-\alpha}\bigl((j+1)^{1-\alpha}-(i+1)^{1-\alpha}\bigr)\right].
\end{equation}
\end{lemma}

\begin{proof}
Set $\eta:=\gamma\lambda_{\min}(A)$ and choose $k_A:=\left\lceil (\gamma\lambda_{\max}(A))^{1/\alpha}\right\rceil+1$. Let $\lambda_1,\ldots,\lambda_d$ denote the eigenvalues of $A$. Then, for every $k\ge k_A$,
\begin{equation*}
\|I_d-\gamma_kA\|=\max_{1\le r\le d}|1-\gamma k^{-\alpha}\lambda_r|\le1-\eta k^{-\alpha}\le e^{-\eta k^{-\alpha}}.
\end{equation*}
The first inequality uses $0\le\gamma_k\lambda_r\le1$ and $\lambda_r\ge\lambda_{\min}(A)$; the last uses $1-x\le e^{-x}$ for $x\in[0,1]$. Set $C:=\left(\prod_{k=1}^{k_A-1}\max\{1,\|I_d-\gamma_kA\|\}\right)\exp\left(\eta\sum_{k=1}^{k_A-1}k^{-\alpha}\right)$. By submultiplicativity, for $j>i$, we have
\begin{align*}
\|Y_i^j\|&\le \prod_{k=i+1}^{j}\|I_d-\gamma_kA\|=\exp\left(-\eta\sum_{k=i+1}^{j}k^{-\alpha}\right)\prod_{k=i+1}^{j}\left(e^{\eta k^{-\alpha}}\|I_d-\gamma_kA\|\right) \le C\exp\left(-\eta\sum_{k=i+1}^{j}k^{-\alpha}\right).
\end{align*}
For $j=i$, the bound follows from $C\ge1$. Finally, integral comparison yields
\begin{equation*}
\sum_{k=i+1}^{j}k^{-\alpha}\ge\int_{i+1}^{j+1}x^{-\alpha}\,dx=\frac{(j+1)^{1-\alpha}-(i+1)^{1-\alpha}}{1-\alpha}.
\end{equation*}
Substituting this lower bound into the exponential proves the result.\qedhere
\end{proof}

\begin{lemma}\label{lem:product-sum}
Under the assumptions of Lemma~\ref{lem:linear-recursion-products}, for each fixed $q>0$, there exists $C(q)<\infty$ such that the matrix products $Y_i^j$ defined there satisfy
\begin{equation}\label{eq:A-product-sum}
\sum_{i=1}^n\|Y_i^n\|i^{-q}\le C(q)n^{\alpha-q},\qquad n\ge1.
\end{equation}
\end{lemma}

\begin{proof}
Set $\eta:=\gamma\lambda_{\min}(A)>0$. The case $n=1$ is covered by choosing $C(q)\ge1$. For $n\ge2$, split the sum at $\lfloor n/2\rfloor$.
If $i\le\lfloor n/2\rfloor$, then $\sum_{k=i+1}^n k^{-\alpha}\ge(n-i)n^{-\alpha}\ge n^{1-\alpha}/2$ and $i^{-q}\le1$, so
\begin{equation*}
\sum_{i=1}^{\lfloor n/2\rfloor}\|Y_i^n\|i^{-q} \;\stackrel{\mathclap{\eqref{eq:A-product-bound}}}{\le}\; C n e^{-\eta n^{1-\alpha}/2}.
\end{equation*}
Dividing by $n^{\alpha-q}$ leaves the bounded factor $z^{(1-\alpha+q)/(1-\alpha)}e^{-\eta z/2}$, where $z=n^{1-\alpha}\ge1$, so
\begin{equation*}
\sum_{i=1}^{\lfloor n/2\rfloor}\|Y_i^n\|i^{-q}\le C n^{\alpha-q}.
\end{equation*}
If $i>\lfloor n/2\rfloor$, then $i^{-q}\le2^qn^{-q}$ and $\sum_{k=i+1}^n k^{-\alpha}\ge(n-i)n^{-\alpha}$. Extending the nonnegative sum to a geometric series, we obtain
\begin{equation*}
\sum_{i=\lfloor n/2\rfloor+1}^n\|Y_i^n\|i^{-q} \;\stackrel{\mathclap{\eqref{eq:A-product-bound}}}{\le}\; C2^qn^{-q}\sum_{r=0}^\infty e^{-\eta r n^{-\alpha}} =\frac{C2^qn^{-q}}{1-e^{-\eta n^{-\alpha}}} \le\frac{C2^q}{1-e^{-\eta}}n^{\alpha-q}.
\end{equation*}
The last inequality uses $1-e^{-\eta n^{-\alpha}}\ge n^{-\alpha}(1-e^{-\eta})$, by concavity of $x\mapsto1-e^{-\eta x}$ on $[0,1]$. Combining the two parts proves the claim.
\end{proof}

\begin{lemma}\label{lem:linear-recursion-moment}
Under the assumptions of Lemma~\ref{lem:linear-recursion-products}, let $(\mathcal F_n)_{n\ge0}$ be a filtration. Let $U_0,\zeta_1,\zeta_2,\ldots\in\mathbb R^d$ be random vectors such that $U_0$ is $\mathcal F_0$-measurable and $\zeta_n$ is $\mathcal F_n$-measurable for every $n\ge1$. Assume $\mE\|U_0\|^2<\infty$, $\mE[\zeta_n\mid\mathcal F_{n-1}]=0$, and $\mE\|\zeta_n\|^2\le K_\zeta<\infty$ for every $n\ge1$. Consider
\begin{equation*}
U_n=(I_d-\gamma_nA)U_{n-1}+\gamma_n\zeta_n, \qquad \gamma_n=\gamma n^{-\alpha}.
\end{equation*}
Then there exists a constant $C>0$ depending on $\mE\|U_0\|^2$ such that, for every $n\ge1$,
\begin{equation*}
\mE\|U_n\|^2\le C\gamma_n.
\end{equation*}
\end{lemma}

\begin{proof}
With $Y_i^n$ from Lemma~\ref{lem:linear-recursion-products}, iterating the recursion yields
\begin{equation*}
U_n=Y_0^nU_0+\sum_{i=1}^nY_i^n\gamma_i\zeta_i.
\end{equation*}
For $i<j$, the $i$th summand is $\mathcal F_{j-1}$-measurable and $\mE(\zeta_j\mid\mathcal F_{j-1})=0$, so
\begin{equation*}
\mE\langle Y_i^n\gamma_i\zeta_i,Y_j^n\gamma_j\zeta_j\rangle =\mE\langle Y_i^n\gamma_i\zeta_i,Y_j^n\gamma_j\mE(\zeta_j\mid\mathcal F_{j-1})\rangle=0.
\end{equation*}
The initial term is also orthogonal to each increment. Thus, using $\mE\|\zeta_i\|^2\le K_\zeta$,
\begin{equation*}
\mE\|U_n\|^2 \le\|Y_0^n\|^2\mE\|U_0\|^2+ K_\zeta\sum_{i=1}^n\gamma_i^2\|Y_i^n\|^2.
\end{equation*}
Since $\sum_{k=1}^n k^{-\alpha}\ge n^{1-\alpha}$, \eqref{eq:A-product-bound} bounds the first term by $C^2e^{-2\gamma\lambda_{\min}(A) n^{1-\alpha}}\allowbreak\mE\|U_0\|^2$. For the second term, we have
\begin{equation*}
\sum_{i=1}^n\gamma_i^2\|Y_i^n\|^2\;\stackrel{\mathclap{\eqref{eq:A-product-bound}}}{\le}\;\gamma^2 C\sum_{i=1}^n\|Y_i^n\|i^{-2\alpha}.
\end{equation*}
Applying \eqref{eq:A-product-sum} with $q=2\alpha$ and absorbing fixed factors into $C$, we obtain
\begin{equation*}
\mE\|U_n\|^2 \le C e^{-2\gamma\lambda_{\min}(A) n^{1-\alpha}}\mE\|U_0\|^2+C K_\zeta n^{-\alpha}.
\end{equation*}
Since $\sup_{z\ge1}z^{\alpha/(1-\alpha)}e^{-2\gamma\lambda_{\min}(A)z}<\infty$, setting $z=n^{1-\alpha}$ bounds the first term by $Cn^{-\alpha}\mE\|U_0\|^2$. This proves the claim.
\end{proof}

\begin{lemma}\label{lem:linear-interval-remainder}
Under the assumptions of Lemma~\ref{lem:linear-recursion-moment}, let $(U_n)_{n\ge0}$ be the sequence defined there. For $1\le s\le t$, define
\begin{equation*}
r^U_{s,t}:= U_t-U_{s-1} +A^{-1}\left( \gamma_s^{-1}U_{s-1} -\gamma_t^{-1}U_t +\sum_{k=s}^{t-1} (\gamma_{k+1}^{-1}-\gamma_k^{-1})U_k \right).
\end{equation*}
Here and below, the sum from $s$ to $t-1$ is understood to be zero when $s=t$. Then there exists a constant $C>0$ depending on $\mE\|U_0\|^2$ such that, for every pair of integers $1\le s\le t$,
\begin{equation*}
\mE\|r^U_{s,t}\|^2\le Ct^\alpha.
\end{equation*}
\end{lemma}

\begin{proof}
Note that the recursion implies $U_{k-1}=A^{-1}\zeta_k+A^{-1}\gamma_k^{-1}(U_{k-1}-U_k)$. Since $\sum_{k=s}^{t}U_k=\sum_{k=s}^{t}U_{k-1}+U_t-U_{s-1}$, discrete summation by parts yields
\begin{equation*}
\sum_{k=s}^{t}\gamma_k^{-1}(U_{k-1}-U_k) =\gamma_s^{-1}U_{s-1}-\gamma_t^{-1}U_t +\sum_{k=s}^{t-1}(\gamma_{k+1}^{-1}-\gamma_k^{-1})U_k.
\end{equation*}
This also holds for $s=t$, with an empty sum. Substitution yields
\begin{equation*}
\sum_{k=s}^t U_k =A^{-1}\sum_{k=s}^t\zeta_k+A^{-1}\sum_{k=s}^t\gamma_k^{-1}(U_{k-1}-U_k)+U_t-U_{s-1} =A^{-1}\sum_{k=s}^t\zeta_k+r^U_{s,t}.
\end{equation*}
By Lemma~\ref{lem:linear-recursion-moment}, including $k=0$ by enlarging $C$, we have
\begin{equation*}
(\mE\|U_k\|^2)^{1/2}\le C(k+1)^{-\alpha/2}, \qquad k\ge0.
\end{equation*}
Also, concavity of $x\mapsto x^\alpha$ implies
\begin{equation*}
0\le\gamma_{k+1}^{-1}-\gamma_k^{-1}\le\alpha\gamma^{-1}k^{\alpha-1}.
\end{equation*}
Minkowski's inequality and these two bounds yield
\begin{align*}
\bigl(\mE\|r^U_{s,t}\|^2\bigr)^{1/2}&\le(\mE\|U_t\|^2)^{1/2}+(\mE\|U_{s-1}\|^2)^{1/2} +\|A^{-1}\|\biggl(\gamma_s^{-1}(\mE\|U_{s-1}\|^2)^{1/2}+\gamma_t^{-1}(\mE\|U_t\|^2)^{1/2}\\
&\qquad+\sum_{k=s}^{t-1}(\gamma_{k+1}^{-1}-\gamma_k^{-1})(\mE\|U_k\|^2)^{1/2}\biggr)\\
&\le 2C+\|A^{-1}\|\gamma^{-1}C\left(s^{\alpha/2}+t^{\alpha/2}+\alpha\sum_{k=s}^{t-1}k^{\alpha/2-1}\right).
\end{align*}
Since $x\mapsto x^{\alpha/2-1}$ is decreasing,
\begin{equation*}
\sum_{k=s}^{t-1}k^{\alpha/2-1} \le 1+\int_s^t x^{\alpha/2-1}\,dx \le 1+\frac{2}{\alpha}t^{\alpha/2}.
\end{equation*}
Using $1\le s\le t$ in the preceding two displays, we obtain
\begin{equation*}
\bigl(\mE\|r^U_{s,t}\|^2\bigr)^{1/2} \le Ct^{\alpha/2}.
\end{equation*}
Squaring proves the claim.
\end{proof}

The following inequality compares two finite weighted sums of outer products.
\begin{lemma}\label{lem:quadratic-comparison}
Let $m\ge1$ be an integer, let $a_1,\ldots,a_m\ge0$ and $D>0$ be deterministic, and let $\sigma_i,e_i\in\mathbb R^d$, $1\le i\le m$, be square-integrable random vectors. Define
\begin{equation*}
Q:=D^{-1}\sum_{i=1}^m a_i(\sigma_i+e_i)^{\otimes2}, \qquad Q_0:=D^{-1}\sum_{i=1}^m a_i\sigma_i^{\otimes2}, \qquad \delta:=D^{-1}\sum_{i=1}^m a_i\mE\|e_i\|^2.
\end{equation*}
Then
\begin{equation}\label{eq:A-quadratic-comparison}
\mE\|Q-Q_0\|\le2(\mE\|Q_0\|)^{1/2}\delta^{1/2}+\delta.
\end{equation}
\end{lemma}

\begin{proof}
Let $\mathcal S,\mathcal R\in\mathbb R^{d\times m}$ have columns $D^{-1/2}a_i^{1/2}\sigma_i$ and $D^{-1/2}a_i^{1/2}e_i$, respectively. Then $Q_0=\mathcal S\mathcal S^\top$ and
\begin{equation*}
Q-Q_0=\mathcal S\mathcal R^\top+\mathcal R\mathcal S^\top+\mathcal R\mathcal R^\top.
\end{equation*}
The trace bound and submultiplicativity imply
\begin{align*}
\mE\|\mathcal R\mathcal R^\top\|&\le\mE\tr(\mathcal R\mathcal R^\top)=D^{-1}\sum_{i=1}^m a_i\mE\|e_i\|^2=\delta,\\
\|\mathcal S\mathcal R^\top\|&\le\|\mathcal S\|\|\mathcal R\|=\|Q_0\|^{1/2}\|\mathcal R\mathcal R^\top\|^{1/2}.
\end{align*}
Since $\|\mathcal R\mathcal S^\top\|=\|\mathcal S\mathcal R^\top\|$, Cauchy--Schwarz yields the second inequality below:
\begin{align*}
\mE\|Q-Q_0\|&\le2\mE[\|Q_0\|^{1/2}\|\mathcal R\mathcal R^\top\|^{1/2}]+\mE\|\mathcal R\mathcal R^\top\| \le2(\mE\|Q_0\|)^{1/2}(\mE\|\mathcal R\mathcal R^\top\|)^{1/2}+\mE\|\mathcal R\mathcal R^\top\|\\
&\le2(\mE\|Q_0\|)^{1/2}\delta^{1/2}+\delta.
\end{align*}
This proves the quadratic comparison bound.
\end{proof}

The following lemma collects the deterministic block bounds used in the OL and NOL proofs. Recall the batches $B_j=\{x_{t_j},x_{t_j+1},\ldots,x_j\}$ in \eqref{eq:batch-indices}, with lengths $\ell_j=|B_j|=j-t_j+1$. The set $I_j$ below contains their indices.

\begin{lemma}\label{lem:block-calculus}
Let $\alpha\in(1/2,1)$ and $\beta>1/(1-\alpha)$. Let $(a_m)_{m\ge1}$ be an increasing integer sequence with $a_1=1$. Suppose that there are constants $0<c_a\le C_a<\infty$ and $0<c_g\le C_g<\infty$ such that, for all $m\ge1$,
\begin{equation*}
c_am^\beta\le a_m\le C_am^\beta, \qquad c_gm^{\beta-1}\le a_{m+1}-a_m\le C_gm^{\beta-1}.
\end{equation*}
For each integer $j\ge1$, write
\begin{equation*}
\begin{aligned}
&m_j:=\max\{m:a_m\le j\}, \qquad t_j:=a_{m_j},\\
&I_j:=\{t_j,t_j+1,\ldots,j\}, \qquad
\ell_j:=|B_j|=j-t_j+1.
\end{aligned}
\end{equation*}
At time $n\ge1$, let $L_m:=\min\{a_{m+1}-1,n\}-a_m+1$ for $1\le m\le m_n$, and define
\begin{alignat*}{2}
&E_n^{\mathrm{OL}}:=\{1,\ldots,n\}, \qquad &&E_n^{\mathrm{NOL}}:=\{\min\{a_{m+1}-1,n\}:1\le m\le m_n\},\\
&v_{n,\mathrm{OL}}:=\sum_{j\in E_n^{\mathrm{OL}}}\ell_j, &&v_{n,\mathrm{NOL}}:=n.
\end{alignat*}
Then there exists $C\ge1$ such that, for every integer $n\ge1$ and $\diamond\in\{\mathrm{OL},\mathrm{NOL}\}$,
\begin{enumerate}[label=\textnormal{(\roman*)},itemsep=1ex]
\item $\displaystyle \frac1{v_{n,\diamond}}\sum_{j\in E_n^\diamond}\ell_j^2 \le C n^{1-1/\beta}$.
\item $\displaystyle \frac1{v_{n,\mathrm{OL}}^2}\sum_{m=1}^{m_n}L_m^4 \le Cn^{-1/\beta}$.
\item $\displaystyle \frac1{v_{n,\diamond}}\sum_{j\in E_n^\diamond}\ell_j^2t_j^{-2\alpha} \le C n^{1-1/\beta-2\alpha}$.
\item $\displaystyle \frac1{v_{n,\diamond}}\sum_{j\in E_n^\diamond}j^\alpha \le C n^{\alpha-1+1/\beta}$.
\end{enumerate}
These bounds apply in particular to $a_1=1$ and $a_m=\lfloor cm^\beta\rfloor$ for $m\ge2$ with $c\ge1$, as in Theorem~\ref{thm:covariance-rate}.
\end{lemma}

\begin{proof}
For $a_m=\lfloor cm^\beta\rfloor$ and $m\ge2$, the floor error is less than one, so $a_m\asymp m^\beta$. The mean-value theorem also implies
\begin{equation*}
c\beta m^{\beta-1}-1 \le a_{m+1}-a_m \le c\beta(m+1)^{\beta-1}+1.
\end{equation*}
Since $\beta>2$ and $c\ge1$, these differences are positive and comparable to $m^{\beta-1}$. Adjusting the constants covers $m=1$ and verifies the hypotheses. For any sequence satisfying the hypotheses, we use
\begin{equation}\label{eq:A-power-sum}
\sum_{m=1}^N m^r\asymp N^{r+1},\qquad r\ge0,\quad N\ge1.
\end{equation}
Indeed, since $x^r$ is nondecreasing, we have
\begin{equation*}
\frac{N^{r+1}}{r+1}=\int_0^N x^r\,\mathrm dx\le\sum_{m=1}^N m^r\le N\cdot N^r=N^{r+1}.
\end{equation*}
First, $a_{m_j}\le j<a_{m_j+1}$ and $m_j+1\le2m_j$ imply
\begin{equation*}
c_a m_j^\beta\le j<C_a(2m_j)^\beta.
\end{equation*}
Thus $m_j\asymp j^{1/\beta}$. The NOL indices lie in distinct blocks, so $|E_n^{\mathrm{NOL}}|=m_n\asymp n^{1/\beta}$. Equivalently,
\begin{equation*}
E_n^{\mathrm{NOL}} = \{a_{m+1}-1:m\ge1,\ a_{m+1}\le n\}\cup\{n\},
\end{equation*}
as in \eqref{eq:nol-endpoints}. The first set excludes $n$, so $n$ is counted once even at a block endpoint. Since $t_j=a_{m_j}\ge c_a m_j^\beta$,
\begin{equation*}
\ell_j\le C_gm_j^{\beta-1}\le C_g c_a^{-(\beta-1)/\beta}t_j^{1-1/\beta}\le C_g c_a^{-(\beta-1)/\beta}j^{1-1/\beta}.
\end{equation*}

For the normalizing factors, use $L_m\le C_gm^{\beta-1}$ for all blocks and $L_m\ge c_gm^{\beta-1}$ only for completed blocks. The OL batches in block $m$ have lengths $1,\ldots,L_m$, so
\begin{equation*}
v_{n,\mathrm{OL}}=\sum_{m=1}^{m_n}\frac{L_m(L_m+1)}2 \lesssim\sum_{m=1}^{m_n}m^{2\beta-2}\;\stackrel{\mathclap{\eqref{eq:A-power-sum}}}{\asymp}\; m_n^{2\beta-1}.
\end{equation*}
The inequality uses $L_m(L_m+1)/2\le L_m^2$. For $m_n\ge2$, the completed blocks and $m_n-1\ge m_n/2$ yield
\begin{equation*}
v_{n,\mathrm{OL}}\ge(c_g^2/2)\sum_{m=1}^{m_n-1}m^{2\beta-2}\;\stackrel{\mathclap{\eqref{eq:A-power-sum}}}{\gtrsim}\;m_n^{2\beta-1}.
\end{equation*}
Adjusting the constants covers the finitely many times with $m_n=1$. Hence $v_{n,\mathrm{OL}}\asymp m_n^{2\beta-1}\asymp n^{2-1/\beta}$. For NOL, the index sets $\{I_j:j\in E_n^{\mathrm{NOL}}\}$ partition $\{1,\ldots,n\}$, so $\sum_{j\in E_n^{\mathrm{NOL}}}\ell_j=n$.

For (i), both batch rules satisfy $\sum_{j\in E_n^\diamond}\ell_j=v_{n,\diamond}$. The batch-length bound above therefore implies
\begin{equation*}
\frac1{v_{n,\diamond}}\sum_{j\in E_n^\diamond}\ell_j^2 \le\max_{j\le n}\ell_j \le C_g c_a^{-(\beta-1)/\beta}n^{1-1/\beta}.
\end{equation*}

For (ii), we have
\begin{equation*}
\frac{\sum_{m=1}^{m_n}L_m^4}{v_{n,\mathrm{OL}}^2} \lesssim\frac{m_n^{4\beta-3}}{m_n^{4\beta-2}} =m_n^{-1}\lesssim n^{-1/\beta}.
\end{equation*}
The first inequality uses $L_m\le C_gm^{\beta-1}$, \eqref{eq:A-power-sum} with exponent $4\beta-4$, and $v_{n,\mathrm{OL}}\asymp m_n^{2\beta-1}$. The last uses $m_n\asymp n^{1/\beta}$.

For (iii), $t_j=a_m\asymp m^\beta$ throughout block $m$. Therefore,
\begin{align*}
\sum_{j\in E_n^{\mathrm{OL}}}\ell_j^2t_j^{-2\alpha}
&\lesssim\sum_{m=1}^{m_n}m^{3\beta-3-2\alpha\beta}
\;\stackrel{\mathclap{\eqref{eq:A-power-sum}}}{\lesssim}\; m_n^{3\beta-2-2\alpha\beta},\\
\sum_{j\in E_n^{\mathrm{NOL}}}\ell_j^2t_j^{-2\alpha}
&\lesssim\sum_{m=1}^{m_n}m^{2\beta-2-2\alpha\beta}
\;\stackrel{\mathclap{\eqref{eq:A-power-sum}}}{\lesssim}\; m_n^{2\beta-1-2\alpha\beta}.
\end{align*}
The first inequalities use $L_m\le C_gm^{\beta-1}$ and $a_m\ge c_am^\beta$, with $\sum_{r=1}^{L_m}r^2\le L_m^3$ for OL and one batch of length $L_m$ for NOL. Both exponents are positive: $2\beta(1-\alpha)-2>0$, and the OL exponent is larger by $\beta-1$. Using $v_{n,\mathrm{OL}}\asymp m_n^{2\beta-1}$ and $v_{n,\mathrm{NOL}}=n\asymp m_n^\beta$ yields
\begin{align*}
\frac1{v_{n,\mathrm{OL}}}\sum_{j\in E_n^{\mathrm{OL}}}\ell_j^2t_j^{-2\alpha}
&\lesssim\frac{m_n^{3\beta-2-2\alpha\beta}}{m_n^{2\beta-1}}
=m_n^{\beta-1-2\alpha\beta}\asymp n^{1-1/\beta-2\alpha},\\
\frac1{v_{n,\mathrm{NOL}}}\sum_{j\in E_n^{\mathrm{NOL}}}\ell_j^2t_j^{-2\alpha}
&\lesssim\frac{m_n^{2\beta-1-2\alpha\beta}}{m_n^\beta}
=m_n^{\beta-1-2\alpha\beta}\asymp n^{1-1/\beta-2\alpha}.
\end{align*}

Finally, for (iv), $\sum_{j=1}^n j^\alpha\lesssim n^{\alpha+1}$ and $|E_n^{\mathrm{NOL}}|\lesssim n^{1/\beta}$ imply
\begin{equation*}
\frac{\sum_{j\in E_n^{\mathrm{OL}}}j^\alpha}{v_{n,\mathrm{OL}}} \lesssim n^{\alpha-1+1/\beta}, \qquad \frac{\sum_{j\in E_n^{\mathrm{NOL}}}j^\alpha}{n} \le n^{\alpha-1}|E_n^{\mathrm{NOL}}| \lesssim n^{\alpha-1+1/\beta}.
\end{equation*}
The bounds use $j^\alpha\le n^\alpha$, $v_{n,\mathrm{OL}}\asymp n^{2-1/\beta}$, and $|E_n^{\mathrm{NOL}}|\asymp n^{1/\beta}$. Choosing a common $C\ge1$ proves all four claims.
\end{proof}

\section{Manifold lemmas}\label{app:B}

This appendix collects the deterministic geometric estimates used in the proofs. Under Assumption~\ref{ass:geometry}(a), $R_x^{-1}(y)$ denotes the unique vector in $B_x(0_x,\rho_{\mathrm{ret}})$ such that $R_x\bigl(R_x^{-1}(y)\bigr)=y$. For the maps used below, $D$ denotes the total differential. If the domain is a product, $D_i$ denotes the partial differential with respect to the $i$th argument, and $D_{ij}^2:=D_jD_i$ denotes the corresponding second partial differential. For a map or a parametrized family of maps with one displayed variable, $D$ and $D^2$ denote the first and second differentials with respect to that variable. Covariant differentiation along a curve is denoted by $D/\mathrm dt$, as in Section~\ref{sec:2.1}. The base point of each vector norm is indicated explicitly, and $\|\cdot\|$ denotes the operator norm induced by the Riemannian metrics on the indicated tangent spaces. For a bilinear map $B:V_1\times V_2\to W$ between finite-dimensional inner-product spaces, $\|B\|:=\sup_{\|u\|_{V_1}\leq1,\,\|v\|_{V_2}\leq1}\|B[u,v]\|_W$.

\begin{lemma}\label{lem:inverse-retraction-C2}
Let $\mX\subseteq\mM$ be compact and suppose that Assumption~\ref{ass:geometry}(a) holds. Assume that $(x,v)\mapsto R_x(v)$ is $C^2$ on an open neighborhood of $\{(x,R_x^{-1}(y)):x,y\in\mX\}$ in $\T\mM$. Then the following conclusions hold.
\begin{enumerate}[label=\textnormal{(\roman*)}]
\item The inverse retraction $(x,y)\mapsto R_x^{-1}(y)$ has a jointly $C^2$ extension to an open neighborhood of $\mX\times\mX$. More precisely, there exist an open set $\mathcal U_R\subseteq\mM\times\mM$ containing $\mX\times\mX$ and a $C^2$ map $\mathcal I:\mathcal U_R\to\T\mM$ such that
\begin{equation*}
\mathcal I(x,y)\in\T_x\mM,\qquad R_x\bigl(\mathcal I(x,y)\bigr)=y,\qquad (x,y)\in\mathcal U_R,
\end{equation*}
and $\mathcal I(x,y)=R_x^{-1}(y)$ for every $x,y\in\mX$.
\item Using the extension in (i), the map $\Phi_{x,y}:=R_y^{-1}\circ R_x$ is $C^2$ on a neighborhood of $R_x^{-1}(\mX)$ for every $x,y\in\mX$. Moreover, there exists $C<\infty$ such that
\begin{equation*}
\|D^2\Phi_{x,y}(v)\|\le C, \qquad x,y\in\mX,\quad v\in R_x^{-1}(\mX).
\end{equation*}
Here the derivatives are taken with respect to $v\in\T_x\mM$.
\end{enumerate}
If $R$ also has the $C^3$ regularity in Assumption~\ref{ass:geometry}(e), both extensions can be chosen to be $C^3$.
\end{lemma}

\begin{proof}
For (i), consider $(x,v)\mapsto(x,R_x(v))$. At $(x,R_x^{-1}(y))$, $x,y\in\mX$, its differential has the block form
\begin{equation*}
\begin{pmatrix}
I&0\\
*&DR_x\bigl(R_x^{-1}(y)\bigr)
\end{pmatrix}.
\end{equation*}
The lower-right block is invertible by Assumption~\ref{ass:geometry}(a), so the inverse function theorem supplies local $C^2$ inverses. Shrink their domains so that the tangent components have norm less than $\rho_{\mathrm{ret}}$. Uniqueness in Assumption~\ref{ass:geometry}(a) identifies these components with $R_x^{-1}(y)$ whenever $x,y\in\mX$. Thus $(x,y)\mapsto R_x^{-1}(y)$ is continuous on $\mX\times\mX$, and $\{(x,R_x^{-1}(y)):x,y\in\mX\}$ is compact.

By the inverse function theorem, $(x,v)\mapsto(x,R_x(v))$ is locally invertible near this compact set. Since it is one-to-one on the set, compactness implies that it is one-to-one on a sufficiently small open neighborhood. Its image is an open set $\mathcal U_R$ containing $\mX\times\mX$, and its $C^2$ inverse has the form $(x,y)\mapsto(x,\mathcal I(x,y))$. This proves (i).

For (ii), the set $\{(x,y,R_x^{-1}(z)):x,y,z\in\mX\}$ is compact by continuity. On this set, $(y,R_x(v))\in\mathcal U_R$, so $\Phi_{x,y}(v)=\mathcal I\bigl(y,R_x(v)\bigr)\in\T_y\mM$ is jointly $C^2$ on a neighborhood. Compactness therefore bounds $D^2\Phi_{x,y}$ uniformly on the set. Under the $C^3$ condition, the same argument yields $C^3$ extensions. This completes the proof.
\end{proof}

\begin{lemma}\label{lem:distance-equivalence}
Under the assumptions of Lemma~\ref{lem:inverse-retraction-C2}, there exist constants $0<c_M\le C_M<\infty$ such that, for every $x,y\in\mX$,
\begin{equation}\label{eq:B-distance-equivalence}
c_Md(x,y)\le \|R_x^{-1}(y)\|_x\le C_Md(x,y).
\end{equation}
Consequently, for fixed $x_\star\in\mX$ and $\forall x,y\in\mX$, $\|R_x^{-1}(y)\|_x \le C_Mc_M^{-1} \bigl( \|R_{x_\star}^{-1}(x)\|_{x_\star} +\|R_{x_\star}^{-1}(y)\|_{x_\star} \bigr).$
\end{lemma}

\begin{proof}
For $x\ne y$, set $h(x,y):=\|R_x^{-1}(y)\|_x/d(x,y)$, which is continuous by Lemma~\ref{lem:inverse-retraction-C2}. We show that $h(x,y)\to1$ as $y\to x$, uniformly in $x\in\mX$. The local properties of the exponential map and compactness of $\mX$ supply a common $\rho>0$: for every $x\in\mX$, $d(x,y)\le\rho$ implies $y=\Exp_x(v)$ for a unique $v$ with $\|v\|_x\le\rho$, and
\begin{equation}\label{eq:B-normal-distance}
d\bigl(x,\Exp_x(v)\bigr)=\|v\|_x, \qquad \|v\|_x\le\rho.
\end{equation}
Here we use $R_x^{-1}$ also for its $C^2$ extension in Lemma~\ref{lem:inverse-retraction-C2}. Reducing $\rho$ if necessary, $(x,v)\mapsto R_x^{-1}\bigl(\Exp_x(v)\bigr)$ is $C^2$ on an open neighborhood of the compact set $\{(x,v)\in\T\mM:x\in\mX,\ \|v\|_x\le\rho\}$, and
\begin{equation*}
R_x^{-1}\bigl(\Exp_x(0_x)\bigr)=0_x, \qquad D(R_x^{-1}\circ\Exp_x)(0_x)=I_{\T_x\mM}.
\end{equation*}
Applying the fundamental theorem of calculus twice to $t\mapsto R_x^{-1}\bigl(\Exp_x(tv)\bigr)$ yields, for $\|v\|_x\le\rho$,
\begin{equation*}
R_x^{-1}\bigl(\Exp_x(v)\bigr)-v = \int_0^1(1-t)D^2(R_x^{-1}\circ\Exp_x)(tv)[v,v]\,\mathrm dt.
\end{equation*}
Consequently,  we obtain
\begin{equation}\label{eq:B-normal-coordinate-remainder}
\|R_x^{-1}\bigl(\Exp_x(v)\bigr)-v\|_x \le \frac12 \sup_{\substack{x\in\mX,\ \|u\|_x\le\rho}} \|D^2(R_x^{-1}\circ\Exp_x)(u)\| \|v\|_x^2.
\end{equation}
The supremum is finite by compactness. For $0<\|v\|_x\le\rho$,
\begin{equation*}
\left|\frac{\|R_x^{-1}\bigl(\Exp_x(v)\bigr)\|_x}{\|v\|_x}-1\right| \le\frac{\|R_x^{-1}\bigl(\Exp_x(v)\bigr)-v\|_x}{\|v\|_x}\;\stackrel{\mathclap{\eqref{eq:B-normal-coordinate-remainder}}}{\le}\; C\|v\|_x.
\end{equation*}
Together with \eqref{eq:B-normal-distance}, this uniform estimate shows that $h(x,x):=1$ extends $h$ continuously to $\mX\times\mX$. Since $R_x^{-1}(y)=0_x$ only when $x=y$, this extension is positive. Compactness then yields $0<c_M:=\min h\le C_M:=\max h<\infty$. Finally, we have
\begin{equation*}
\|R_x^{-1}(y)\|_x \;\stackrel{\mathclap{\eqref{eq:B-distance-equivalence}}}{\le}\; C_Md(x,y) \le C_M\bigl(d(x,x_\star)+d(x_\star,y)\bigr) \;\stackrel{\mathclap{\eqref{eq:B-distance-equivalence}}}{\le}\; C_Mc_M^{-1}\bigl(\|R_{x_\star}^{-1}(x)\|_{x_\star}+\|R_{x_\star}^{-1}(y)\|_{x_\star}\bigr),
\end{equation*}
which proves the second claim.
\end{proof}

Recall from Section~\ref{sec:2.1} that $\Lambda_x^y=DR_x\bigl(R_x^{-1}(y)\bigr)$ for $x,y\in\mX$.

\begin{lemma}\label{lem:lambda-metric-distortion}
Under the assumptions of Lemma~\ref{lem:inverse-retraction-C2}, suppose additionally that items (b), (d), and (e) of Assumption~\ref{ass:geometry} hold. Then there exists a constant $C_L<\infty$ such that, for every $x,y\in\mX$,
\begin{equation}\label{eq:lambda-uniform-bounds}
\|\Lambda_x^y\| +\|[\Lambda_x^y]^{-1}\| \le C_L.
\end{equation}
Moreover,
\begin{equation}\label{eq:B-lambda-distortion}
\|\Lambda_x^y\| \le 1+C_L\|R_x^{-1}(y)\|_x^2.
\end{equation}
\end{lemma}

\begin{proof}
Let $\mathcal I$ be the $C^2$ extension in Lemma~\ref{lem:inverse-retraction-C2}. For $u=R_x^{-1}(y)$, the identity $\mathcal I\bigl(x,R_x(u)\bigr)=u$ and the chain rule imply
\begin{equation}\label{eq:B-lambda-inverse-identity}
D_2\mathcal I(x,y)\circ DR_x(u) = I_{\T_x\mM}.
\end{equation}
Since $DR_x(u)=\Lambda_x^y$ by definition, we have $[\Lambda_x^y]^{-1}=D_2\mathcal I(x,y)$. Both $\Lambda_x^y$ and its inverse are continuous on the compact set $\mX\times\mX$, proving \eqref{eq:lambda-uniform-bounds} for some $C_L$.

For the second bound, the second covariant derivative of $R_x$ at $0_x$ is symmetric because the Levi--Civita connection is torsion-free. Its values on equal arguments vanish by the second-order retraction property, so polarization implies
\begin{equation}\label{eq:B-retraction-zero-derivative}
\left.\frac{D}{\mathrm dt}\bigl(DR_x(th)[v]\bigr)\right|_{t=0}=0_x, \qquad h,v\in\T_x\mM,
\end{equation}
where the derivative is taken along $t\mapsto R_x(th)$. For $\|v\|_x=1$, the retraction property, metric compatibility, and \eqref{eq:B-retraction-zero-derivative} imply
\begin{equation*}
\|DR_x(0_x)[v]\|_x^2=1, \qquad \left.\frac{d}{ds}\|DR_x(su)[v]\|_{R_x(su)}^2\right|_{s=0}=0.
\end{equation*}
Retraction convexity and injectivity imply $R_x(su)\in\mX$ and $su=R_x^{-1}\bigl(R_x(su)\bigr)$ for $s\in[0,1]$. These points lie in the compact set $\{(x,R_x^{-1}(z)):x,z\in\mX\}$. By $C^3$ regularity, $\bigl|\frac{d^2}{ds^2}\|DR_x(su)[v]\|_{R_x(su)}^2\bigr|\le C\|u\|_x^2$ uniformly for $x,y\in\mX$, $s\in[0,1]$, and $\|v\|_x=1$. Taylor's formula therefore yields
\begin{equation}\label{eq:B-pullback-metric-bound}
\|DR_x(u)[v]\|_y^2 =1+\int_0^1(1-s)\frac{d^2}{ds^2}\|DR_x(su)[v]\|_{R_x(su)}^2\,\mathrm ds \le 1+\frac{C}{2}\|u\|_x^2.
\end{equation}
Taking the supremum over $\|v\|_x=1$, we obtain
\begin{equation*}
\|\Lambda_x^y\|\;\stackrel{\mathclap{\eqref{eq:B-pullback-metric-bound}}}{\le}\;\left(1+\frac{C}{2}\|u\|_x^2\right)^{1/2} \le1+\frac{C}{4}\|u\|_x^2.
\end{equation*}
Choosing $C_L$ large enough proves both bounds.
\end{proof}

\begin{lemma}\label{lem:taylor-coordinate-change}
Under the assumptions of Lemma~\ref{lem:inverse-retraction-C2}, suppose additionally that items (b) and (e) of Assumption~\ref{ass:geometry} hold. Then there exists a constant $C_\Psi<\infty$ such that, for every $x\in\mX$ and every $u,w\in\T_x\mM$ with $u,u+w\in R_x^{-1}(\mX)$, we have
\begin{equation}\label{eq:B-coordinate-change}
\left\| [\Lambda_x^{R_x(u)}]^{-1} R_{R_x(u)}^{-1}\bigl(R_x(u+w)\bigr)-w \right\|_x \le C_\Psi\|u\|_x\|w\|_x.
\end{equation}
\end{lemma}

\begin{proof}
For $x\in\mX$ and $u,v\in R_x^{-1}(\mX)$, define
\begin{equation*}
\Psi_x(u,v):=[\Lambda_x^{R_x(u)}]^{-1}R_{R_x(u)}^{-1}\bigl(R_x(v)\bigr)-(v-u)\in\T_x\mM.
\end{equation*}
These triples form a compact set by Lemma~\ref{lem:inverse-retraction-C2}. Since $\Lambda_x^{R_x(u)}=DR_x(u)$, the inverse extension in that lemma and Assumption~\ref{ass:geometry}(e) make $\Psi$ jointly $C^2$ on an open neighborhood of this set. Moreover,
\begin{equation*}
\Psi_x(0_x,v)=0_x, \qquad \Psi_x(u,u)=0_x.
\end{equation*}
Retraction convexity and injectivity imply $su\in R_x^{-1}(\mX)$ for $s\in[0,1]$. Thus $(x,su,u+w)$ and $(x,su,u)$ lie in the compact set above. Since $D_1\Psi$ is $C^1$ on an open neighborhood, it is Lipschitz in the second variable on this compact set, uniformly in $x$ and the first variable. Hence there exists $C<\infty$ such that
\begin{equation*}
\|D_1\Psi_x(su,u+w)-D_1\Psi_x(su,u)\|\le C\|w\|_x, \qquad s\in[0,1].
\end{equation*}
The fundamental theorem of calculus and the two identities above yield
\begin{align*}
\Psi_x(u,u+w)&=\Psi_x(u,u+w)-\Psi_x(0_x,u+w)-\Psi_x(u,u)+\Psi_x(0_x,u)\\
&=\int_0^1\bigl(D_1\Psi_x(su,u+w)-D_1\Psi_x(su,u)\bigr)[u]\,\mathrm ds.
\end{align*}
The derivative bound then implies
\begin{equation*}
\|\Psi_x(u,u+w)\|_x\le\int_0^1 C\|w\|_x\|u\|_x\,\mathrm ds=C\|u\|_x\|w\|_x.
\end{equation*}
Taking $C_\Psi=C$ completes the proof.
\end{proof}

We next compare $\Lambda_x^y$ with parallel transport $\Gamma_x^y$ along the curve $t\mapsto R_x\bigl(tR_x^{-1}(y)\bigr)$, $t\in[0,1]$. Its inverse $(\Gamma_x^y)^{-1}$ denotes parallel transport along the same curve in reverse.

\begin{lemma}\label{lem:transport-lambda}
Under Assumption~\ref{ass:geometry}, there exists a constant $C_{\Gamma,\Lambda}<\infty$ such that, for every $x,y\in\mX$,
\begin{equation}\label{eq:B-transport-comparison}
\|(\Gamma_x^y)^{-1}\Lambda_x^y-I_{\T_x\mM}\| \le C_{\Gamma,\Lambda}\|R_x^{-1}(y)\|_x^2.
\end{equation}
In particular, for every $y\in\mX$,
\begin{equation}\label{eq:B-inverse-transport-comparison}
\|(\Gamma_{x_\star}^y)^{-1}-[\Lambda_{x_\star}^y]^{-1}\| \le C_{\Gamma,\Lambda}\|R_{x_\star}^{-1}(y)\|_{x_\star}^2.
\end{equation}
\end{lemma}

\begin{proof}
Set $u=R_x^{-1}(y)$. Retraction convexity and injectivity imply $R_x(su)\in\mX$ and $su=R_x^{-1}\bigl(R_x(su)\bigr)$ for $s\in[0,1]$. Thus $(x,su)$ remains in the compact set $\{(x,R_x^{-1}(z)):x,z\in\mX\}$ from Lemma~\ref{lem:inverse-retraction-C2}. For nearby $(x,v)$, let $\Gamma_x^{R_x(v)}$ denote transport along $t\mapsto R_x(tv)$. It satisfies a linear differential equation with coefficients that are $C^2$ in $(x,v)$ by Assumption~\ref{ass:geometry}(e), so its solution and inverse are $C^2$. Consequently, $(\Gamma_x^{R_x(v)})^{-1}DR_x(v)$ is jointly $C^2$ near this compact set. Since $\Gamma_x^{R_x(su)}$ follows the curve $s\mapsto R_x(su)$ up to time $s$, \eqref{eq:B-retraction-zero-derivative} implies, for every $v\in\T_x\mM$,
\begin{equation*}
\left.\frac{d}{ds}\left((\Gamma_x^{R_x(su)})^{-1}DR_x(su)[v]\right)\right|_{s=0}
=\left.\frac{D}{\mathrm ds}\bigl(DR_x(su)[v]\bigr)\right|_{s=0}=0_x.
\end{equation*}
The operator at $s=0$ is $I_{\T_x\mM}$, and its second derivative in $s$ has norm at most $C\|u\|_x^2$ by the $C^2$ regularity and compactness above. Taylor's formula then yields
\begin{equation*}
\|(\Gamma_x^y)^{-1}\Lambda_x^y-I_{\T_x\mM}\|
=\left\|\int_0^1(1-s)\frac{d^2}{ds^2}\left((\Gamma_x^{R_x(su)})^{-1}DR_x(su)\right)\,\mathrm ds\right\|
\le\frac{C}{2}\|u\|_x^2.
\end{equation*}
This proves the first bound. For the inverse transports,
\begin{equation*}
[\Lambda_{x_\star}^y]^{-1}-(\Gamma_{x_\star}^y)^{-1}
=[\Lambda_{x_\star}^y]^{-1}\bigl(\Gamma_{x_\star}^y-\Lambda_{x_\star}^y\bigr)(\Gamma_{x_\star}^y)^{-1}.
\end{equation*}
Submultiplicativity and the isometry of parallel transport imply
\begin{equation*}
\|[\Lambda_{x_\star}^y]^{-1}-(\Gamma_{x_\star}^y)^{-1}\|
\;\stackrel{\mathclap{\eqref{eq:lambda-uniform-bounds}}}{\le}\;
C_L\|\Lambda_{x_\star}^y-\Gamma_{x_\star}^y\|
=C_L\|(\Gamma_{x_\star}^y)^{-1}\Lambda_{x_\star}^y-I_{\T_{x_\star}\mM}\|.
\end{equation*}
Combining this with the first bound and choosing $C_{\Gamma,\Lambda}$ large enough proves both claims.
\end{proof}

\section{Algorithmic estimates}\label{app:C}

Recall $\Delta_n=R_{x_\star}^{-1}(x_n)$ and $\tDelta_n=R_{x_\star}^{-1}(\wt x_n)$, both in $\T_{x_\star}\mM$. In Appendices~\ref{app:C}--\ref{app:D}, we use a fixed orthonormal basis of $\T_{x_\star}\mM$, with $\|\cdot\|$ denoting the vector or operator norm and $\|\cdot\|_{\mathrm F}$ the Frobenius norm. Write $g_n:=\grad f(x_{n-1},\xi_n)\in\T_{x_{n-1}}\mM$ and $D_\mX:=\sup_{x\in\mX}\|R_{x_\star}^{-1}(x)\|$.

\begin{lemma}\label{lem:C-retraction-coordinate-taylor}
Let $\gamma_n=\gamma n^{-\alpha}$, where $\gamma>0$ and $\alpha\in(1/2,1)$. Then, under Assumptions~\ref{ass:geometry} and~\ref{ass:noise}, there exist a constant $C<\infty$ and $\mF_n$-measurable vectors $r_{\Delta,n}\in\T_{x_\star}\mM$, $n\ge1$, such that, for every $n\ge1$,
\begin{equation*}
\Delta_n=\Delta_{n-1}-\gamma_n[\Lambda_{x_\star}^{x_{n-1}}]^{-1}g_n+\gamma_n r_{\Delta,n},
\end{equation*}
and, almost surely,
\begin{equation*}
\mE_{n-1}\|r_{\Delta,n}\|\le C\gamma_n, \qquad \mE_{n-1}\|r_{\Delta,n}\|^2\le C\gamma_n^2.
\end{equation*}
\end{lemma}

\begin{proof}
Following the linearization argument of \citet[Appendix~C.1]{Tripuraneni2018Averaging}, define $r_{\Delta,n}$ by
\begin{equation*}
r_{\Delta,n}:=\gamma_n^{-1}(\Delta_n-\Delta_{n-1})+[\Lambda_{x_\star}^{x_{n-1}}]^{-1}g_n.
\end{equation*}
This remainder is $\mF_n$-measurable and satisfies the recursion. Recall $\Phi_{x,x_\star}=R_{x_\star}^{-1}\circ R_x$ from Lemma~\ref{lem:inverse-retraction-C2}. Retraction convexity and injectivity imply $tR_x^{-1}(z)=R_x^{-1}\bigl(R_x(tR_x^{-1}(z))\bigr)$ with $R_x(tR_x^{-1}(z))\in\mX$. The lemma therefore bounds $\|D^2\Phi_{x,x_\star}(tR_x^{-1}(z))\|$ uniformly over $x,z\in\mX$ and $t\in[0,1]$ by some $C<\infty$.
The chain rule and the inverse function theorem yield
\begin{equation*}
D\Phi_{x,x_\star}(0_x)=D(R_{x_\star}^{-1})(x)\circ DR_x(0_x)=[DR_{x_\star}\bigl(R_{x_\star}^{-1}(x)\bigr)]^{-1}=[\Lambda_{x_\star}^x]^{-1}.
\end{equation*}
If $\gamma_n\|g_n\|_{x_{n-1}}<\rho_{\mathrm{ret}}$, injectivity implies $R_{x_{n-1}}^{-1}(x_n)=-\gamma_ng_n$. Taylor's formula at $0_{x_{n-1}}$ then yields
\begin{equation*}
\Delta_n =\Delta_{n-1}-\gamma_n[\Lambda_{x_\star}^{x_{n-1}}]^{-1}g_n +\gamma_n^2\int_0^1(1-t)D^2\Phi_{x_{n-1},x_\star}(-t\gamma_ng_n)[g_n,g_n]\,\mathrm dt.
\end{equation*}
Dividing the remainder by $\gamma_n$, we obtain
\begin{equation*}
\|r_{\Delta,n}\| \le\frac{C}{2}\gamma_n\|g_n\|_{x_{n-1}}^2.
\end{equation*}
If $\gamma_n\|g_n\|_{x_{n-1}}\ge\rho_{\mathrm{ret}}$, compactness and Lemma~\ref{lem:inverse-retraction-C2} give $D_\mX<\infty$, and the definition of $r_{\Delta,n}$ yields
\begin{equation*}
\|r_{\Delta,n}\|\;\stackrel{\mathclap{\eqref{eq:lambda-uniform-bounds}}}{\le}\;2D_\mX\gamma_n^{-1}+C_L\|g_n\|_{x_{n-1}} \le\left(\frac{2D_\mX}{\rho_{\mathrm{ret}}^2}+\frac{C_L}{\rho_{\mathrm{ret}}}\right)\gamma_n\|g_n\|_{x_{n-1}}^2.
\end{equation*}
Combining the two cases, we obtain
\begin{equation}\label{eq:C-pointwise-retraction-remainder}
\|r_{\Delta,n}\|\le C\gamma_n\|g_n\|_{x_{n-1}}^2.
\end{equation}
The loss is jointly measurable by continuity in $x$ and measurability in the data. Its partial derivatives are limits of measurable difference quotients, so the smooth metric makes the stochastic gradient jointly measurable. Since $x_{n-1}$ is $\mF_{n-1}$-measurable and $\xi_n$ is independent of $\mF_{n-1}$, Assumption~\ref{ass:noise}(b) implies
\begin{equation}\label{eq:C-gradient-fourth-moment}
\mE_{n-1}\|g_n\|_{x_{n-1}}^4=\mE_{\xi\sim\mathcal D}\bigl[\|\grad f(x_{n-1},\xi)\|_{x_{n-1}}^4\bigr]\le\tau^4
\end{equation}
almost surely, where the expectation over $\xi$ holds $x_{n-1}$ fixed. Taking conditional expectations in the pointwise bound and its square yields, respectively,
\begin{align*}
\mE_{n-1}\|r_{\Delta,n}\|\;&\stackrel{\mathclap{\eqref{eq:C-pointwise-retraction-remainder}}}{\le}\;C\gamma_n\mE_{n-1}\|g_n\|_{x_{n-1}}^2,\\
\mE_{n-1}\|r_{\Delta,n}\|^2\;&\stackrel{\mathclap{\eqref{eq:C-pointwise-retraction-remainder}}}{\le}\;C^2\gamma_n^2\mE_{n-1}\|g_n\|_{x_{n-1}}^4.
\end{align*}
The first bound uses only the conditional second moment of $g_n$, whereas the second uses its conditional fourth moment. Conditional Jensen's inequality and \eqref{eq:C-gradient-fourth-moment} bound these moments by $\tau^2$ and~$\tau^4$, respectively. Choosing a common $C$ proves both claims.
\end{proof}

\begin{lemma}\label{lem:C-stochastic-gradient-expansion}
Recall that $A=\Hess F(x_\star)$. Under Assumptions~\ref{ass:geometry}--\ref{ass:noise}, there exist a constant $C_{r,F}<\infty$, $\mF_n$-measurable vectors $\widetilde\epsilon_n\in\T_{x_\star}\mM$, and $\mF_{n-1}$-measurable vectors $r_{F,n}\in\T_{x_\star}\mM$, $n\ge1$, such that, for every $n\ge1$,
\begin{equation*}
(\Gamma_{x_\star}^{x_{n-1}})^{-1}g_n=A\Delta_{n-1}+\grad f(x_\star,\xi_n)+\widetilde\epsilon_n+r_{F,n},
\end{equation*}
and, almost surely,
\begin{equation*}
\mE_{n-1}\widetilde\epsilon_n=0, \qquad \mE_{n-1}\|\widetilde\epsilon_n\|^2\le L_g^2\|\Delta_{n-1}\|^2, \qquad \|r_{F,n}\|\le C_{r,F}\|\Delta_{n-1}\|^2.
\end{equation*}
\end{lemma}

\begin{proof}
Define the centered perturbation of the stochastic gradient by
\begin{equation*}
\widetilde\epsilon_n :=(\Gamma_{x_\star}^{x_{n-1}})^{-1}g_n-\grad f(x_\star,\xi_n) -(\Gamma_{x_\star}^{x_{n-1}})^{-1}\grad F(x_{n-1}).
\end{equation*}
Conditioning on $\mF_{n-1}$ as in the preceding proof, Assumption~\ref{ass:noise}(a) yields
\begin{equation}\label{eq:random-point-conditioning}
\mE_{n-1}g_n=\mE_{\xi\sim\mathcal D}\bigl[\grad f(x_{n-1},\xi)\bigr]=\grad F(x_{n-1}).
\end{equation}
The fourth-moment bound ensures integrability, and the inverse transport is $\mF_{n-1}$-measurable and isometric. Thus \eqref{eq:random-point-conditioning} and $\grad F(x_\star)=0$ imply
\begin{align*}
\mE_{n-1}\widetilde\epsilon_n&= (\Gamma_{x_\star}^{x_{n-1}})^{-1}\mE_{n-1}g_n -\mE_{n-1}\grad f(x_\star,\xi_n) -(\Gamma_{x_\star}^{x_{n-1}})^{-1}\grad F(x_{n-1})\\
&= (\Gamma_{x_\star}^{x_{n-1}})^{-1}\grad F(x_{n-1}) -\grad F(x_\star) -(\Gamma_{x_\star}^{x_{n-1}})^{-1}\grad F(x_{n-1}) =0.
\end{align*}
For the conditional second moment, we have
\begin{align*}
\mE_{n-1}\|\widetilde\epsilon_n\|^2&=\mE_{n-1}\|(\Gamma_{x_\star}^{x_{n-1}})^{-1}g_n-\grad f(x_\star,\xi_n)\|^2-\|(\Gamma_{x_\star}^{x_{n-1}})^{-1}\grad F(x_{n-1})\|^2\\
&\le\mE_{n-1}\|(\Gamma_{x_\star}^{x_{n-1}})^{-1}g_n-\grad f(x_\star,\xi_n)\|^2 \le L_g^2\|\Delta_{n-1}\|^2.
\end{align*}
The equality is the conditional variance identity; the last inequality uses Assumption~\ref{ass:noise}(c) at $x_{n-1}$. To expand $\grad F$, set $y(t)=y_t:=R_{x_\star}(t\Delta_{n-1})$ for $t\in[0,1]$. Retraction convexity and injectivity imply $y_t\in\mX$ and $R_{x_\star}^{-1}(y_t)=t\Delta_{n-1}$. Since $\Gamma_{x_\star}^{y_t}$ follows the subcurve $s\mapsto y(ts)$, differentiation yields
\begin{equation*}
\frac{\mathrm d}{\mathrm dt}\left((\Gamma_{x_\star}^{y_t})^{-1}\grad F(y_t)\right)=(\Gamma_{x_\star}^{y_t})^{-1}\circ\Hess F(y_t)\circ\Lambda_{x_\star}^{y_t}\Delta_{n-1}.
\end{equation*}
Since $\grad F(x_\star)=0$ and $y_1=x_{n-1}$, the fundamental theorem of calculus yields
\begin{equation*}
(\Gamma_{x_\star}^{x_{n-1}})^{-1}\grad F(x_{n-1}) = \int_0^1(\Gamma_{x_\star}^{y_t})^{-1}\circ\Hess F(y_t)\circ\Lambda_{x_\star}^{y_t}\Delta_{n-1}\,dt.
\end{equation*}
Define the $\mF_{n-1}$-measurable remainder $r_{F,n}:=(\Gamma_{x_\star}^{x_{n-1}})^{-1}\grad F(x_{n-1})-A\Delta_{n-1}$. Adding and subtracting $\Gamma_{x_\star}^{y_t}$ yields
\begin{align*}
r_{F,n}&=\int_0^1\bigl((\Gamma_{x_\star}^{y_t})^{-1}\circ\Hess F(y_t)\circ\Gamma_{x_\star}^{y_t}-A\bigr)\Delta_{n-1}\,dt\\
&\quad+\int_0^1(\Gamma_{x_\star}^{y_t})^{-1}\circ\Hess F(y_t)\circ(\Lambda_{x_\star}^{y_t}-\Gamma_{x_\star}^{y_t})\Delta_{n-1}\,dt.
\end{align*}
Assumption~\ref{ass:objective}(c) bounds the transported Hessian by $\|A\|+L_HD_\mX$ and the first integrand by $L_Ht\|\Delta_{n-1}\|^2$. For the second integrand,
\begin{equation*}
\|(\Gamma_{x_\star}^{y_t})^{-1}\Lambda_{x_\star}^{y_t}-I\| \;\stackrel{\mathclap{\eqref{eq:B-transport-comparison}}}{\le}\; C_{\Gamma,\Lambda}t^2\|\Delta_{n-1}\|^2,
\end{equation*}
Thus, using $\|\Delta_{n-1}\|\le D_\mX$,
\begin{align*}
\|r_{F,n}\|
&\le\int_0^1 L_Ht\|\Delta_{n-1}\|^2\,\mathrm dt
+\int_0^1 (\|A\|+L_HD_\mX) C_{\Gamma,\Lambda}t^2\|\Delta_{n-1}\|^3\,\mathrm dt\\
&\le\left(\frac{L_H}{2}+\frac{(\|A\|+L_HD_\mX) C_{\Gamma,\Lambda}D_\mX}{3}\right)\|\Delta_{n-1}\|^2
=:C_{r,F}\|\Delta_{n-1}\|^2.
\end{align*}
Combining the definitions of $\widetilde\epsilon_n$ and $r_{F,n}$ proves the expansion.
\end{proof}

\begin{theorem}\label{thm:C-perturbed-tangent-recursion}
Under the assumptions of Lemma~\ref{lem:C-stochastic-gradient-expansion}, let $\gamma_n=\gamma n^{-\alpha}$, where $\gamma>0$ and $\alpha\in(1/2,1)$. Then there exist $\mF_n$-measurable sequences $\{\epsilon_n\}_{n\ge1}$ and $\{e_n\}_{n\ge1}$ in $\T_{x_\star}\mM$ and a constant $C<\infty$ such that, for every $n\ge1$,
\begin{equation}\label{eq:C-perturbed-tangent-recursion}
\Delta_n=(I-\gamma_n A)\Delta_{n-1}+\gamma_n(\zeta_n+\epsilon_n+e_n),
\end{equation}
where $\zeta_n:=-\grad f(x_\star,\xi_n)$.
For every $n\ge1$, the three terms satisfy
\begin{equation}\label{eq:C-zeta-properties}
\mE_{n-1}\zeta_n=0, \qquad \mE_{n-1}(\zeta_n^{\otimes2})=S, \qquad \mE_{n-1}\|\zeta_n\|^4\le \tau^4,
\end{equation}
\begin{equation}\label{eq:C-epsilon-properties}
\mE_{n-1}\epsilon_n=0, \qquad \mE_{n-1}\|\epsilon_n\|^2\le C\|\Delta_{n-1}\|^2,
\end{equation}
and
\begin{equation}\label{eq:C-e-remainder-properties}
\mE_{n-1}\|e_n\|\le C\bigl(\|\Delta_{n-1}\|^2+\gamma_n\bigr), \qquad \mE_{n-1}\|e_n\|^2\le C\bigl(\|\Delta_{n-1}\|^4+\gamma_n^2\bigr).
\end{equation}
\end{theorem}

\begin{proof}
Since $\Gamma_{x_\star}^{x_{n-1}}$ is an isometry,
\begin{equation*}
\left\|[\Lambda_{x_\star}^{x_{n-1}}]^{-1}\Gamma_{x_\star}^{x_{n-1}}-I\right\| \;\stackrel{\mathclap{\eqref{eq:B-inverse-transport-comparison}}}{\le}\; C_{\Gamma,\Lambda}\|\Delta_{n-1}\|^2.
\end{equation*}
Define
\begin{equation*}
\epsilon_n := -\widetilde\epsilon_n -\left( [\Lambda_{x_\star}^{x_{n-1}}]^{-1} \Gamma_{x_\star}^{x_{n-1}}-I \right) \bigl(\grad f(x_\star,\xi_n)+\widetilde\epsilon_n\bigr),
\end{equation*}
and
\begin{equation*}
e_n := r_{\Delta,n} -r_{F,n} -\left( [\Lambda_{x_\star}^{x_{n-1}}]^{-1} \Gamma_{x_\star}^{x_{n-1}}-I \right) (A\Delta_{n-1}+r_{F,n}).
\end{equation*}
Both sequences are $\mF_n$-measurable. Substituting
\begin{equation*}
[\Lambda_{x_\star}^{x_{n-1}}]^{-1} =[\Lambda_{x_\star}^{x_{n-1}}]^{-1}\Gamma_{x_\star}^{x_{n-1}}(\Gamma_{x_\star}^{x_{n-1}})^{-1},
\end{equation*}
into the recursion for $\Delta_n$ and using the gradient expansion in Lemma~\ref{lem:C-stochastic-gradient-expansion}, we obtain
\begin{align*}
\Delta_n&=\Delta_{n-1}+\gamma_n r_{\Delta,n} -\gamma_n[\Lambda_{x_\star}^{x_{n-1}}]^{-1}\Gamma_{x_\star}^{x_{n-1}} \bigl(A\Delta_{n-1}+\grad f(x_\star,\xi_n)+\widetilde\epsilon_n+r_{F,n}\bigr)\\
&=(I-\gamma_n A)\Delta_{n-1}-\gamma_n\grad f(x_\star,\xi_n) +\gamma_n(r_{\Delta,n}-\widetilde\epsilon_n-r_{F,n})\\
&\quad-\gamma_n\bigl([\Lambda_{x_\star}^{x_{n-1}}]^{-1}\Gamma_{x_\star}^{x_{n-1}}-I\bigr) \bigl(A\Delta_{n-1}+\grad f(x_\star,\xi_n)+\widetilde\epsilon_n+r_{F,n}\bigr)\\
&=(I-\gamma_n A)\Delta_{n-1}+\gamma_n(\zeta_n+\epsilon_n+e_n),
\end{align*}
which proves \eqref{eq:C-perturbed-tangent-recursion}. Assumption~\ref{ass:noise}(a), (b), and (d) at $x_\star$, together with $\grad F(x_\star)=0$ from Assumption~\ref{ass:objective}(a), imply \eqref{eq:C-zeta-properties}. For $\epsilon_n$, the operator
\begin{equation*}
[\Lambda_{x_\star}^{x_{n-1}}]^{-1} \Gamma_{x_\star}^{x_{n-1}}-I
\end{equation*}
is $\mF_{n-1}$-measurable. Assumption~\ref{ass:noise}(a) and Lemma~\ref{lem:C-stochastic-gradient-expansion} therefore imply
\begin{align*}
\mE_{n-1}\epsilon_n&= -\mE_{n-1}\widetilde\epsilon_n -\left( [\Lambda_{x_\star}^{x_{n-1}}]^{-1} \Gamma_{x_\star}^{x_{n-1}}-I \right) \mE_{n-1}[\grad f(x_\star,\xi_n)+\widetilde\epsilon_n]\\
&= -0 -\left( [\Lambda_{x_\star}^{x_{n-1}}]^{-1} \Gamma_{x_\star}^{x_{n-1}}-I \right) \bigl(\grad F(x_\star)+0\bigr) =0.
\end{align*}
Conditional Cauchy--Schwarz bounds the noise second moment by $\tau^2$. Using Lemma~\ref{lem:C-stochastic-gradient-expansion} and $\|\Delta_{n-1}\|\le D_\mX$, we obtain
\begin{equation*}
\mE_{n-1}\|\grad f(x_\star,\xi_n)+\widetilde\epsilon_n\|^2\le 2\tau^2+2L_g^2D_\mX^2.
\end{equation*}
The transport bound then yields
\begin{equation*}
\mE_{n-1}\|\epsilon_n\|^2 \le 2L_g^2\|\Delta_{n-1}\|^2+C\|\Delta_{n-1}\|^4 \le (2L_g^2+CD_\mX^2)\|\Delta_{n-1}\|^2.
\end{equation*}
For $e_n$, the remainder and transport bounds, together with $\|\Delta_{n-1}\|\le D_\mX$, imply
\begin{equation*}
\|e_n-r_{\Delta,n}\| \le C_{r,F}\|\Delta_{n-1}\|^2+C_{\Gamma,\Lambda}\|\Delta_{n-1}\|^2(\|A\|\|\Delta_{n-1}\|+C_{r,F}\|\Delta_{n-1}\|^2) \le C\|\Delta_{n-1}\|^2.
\end{equation*}
Lemma~\ref{lem:C-retraction-coordinate-taylor} now implies
\begin{align*}
\mE_{n-1}\|e_n\|&\le C\bigl(\|\Delta_{n-1}\|^2+\gamma_n\bigr),\\
\mE_{n-1}\|e_n\|^2&\le C\bigl(\|\Delta_{n-1}\|^4+\gamma_n^2\bigr).
\end{align*}
Choosing a common $C$ completes the proof.
\end{proof}

We now compare the perturbed recursion with its linear part. Under Assumption~\ref{ass:raw-moments}(a) and (b), respectively, choose $C_{\Delta,2},C_{\Delta,4}<\infty$, enlarged to cover the finitely many initial indices, such that, for every $n\ge1$,
\begin{equation}\label{eq:C-raw-moment-bounds}
\mE\|\Delta_n\|^2\le C_{\Delta,2}\gamma_n, \qquad \mE\|\Delta_n\|^4\le C_{\Delta,4}\gamma_n^2.
\end{equation}

\begin{lemma}\label{lem:C-linear-coupling}
Under the assumptions of Theorem~\ref{thm:C-perturbed-tangent-recursion} and Assumption~\ref{ass:raw-moments}, define the linear recursion in $\T_{x_\star}\mM$ by $U_0:=\Delta_0$ and
\begin{equation}\label{eq:C-linear-recursion}
U_n=(I-\gamma_n A)U_{n-1}+\gamma_n\zeta_n, \qquad n\ge1.
\end{equation}
Set $d_n:=\Delta_n-U_n$ for $n\ge0$. Then there exists a constant $C<\infty$ such that, for every $n\ge1$,
\begin{equation*}
(\mE\|d_n\|^2)^{1/2}\le C n^{-\alpha}.
\end{equation*}
\end{lemma}

\begin{proof}
Subtracting the linear recursion \eqref{eq:C-linear-recursion} from the perturbed recursion \eqref{eq:C-perturbed-tangent-recursion}, we have $d_0=0$ and
\begin{equation*}
d_n=(I-\gamma_n A)d_{n-1}+\gamma_n(\epsilon_n+e_n).
\end{equation*}
Taking expectations of the conditional bounds yields, for $n\ge2$,
\begin{equation*}
\begin{aligned}
\mE\|\epsilon_n\|^2&=\mE[\mE_{n-1}\|\epsilon_n\|^2]\;\stackrel{\mathclap{\eqref{eq:C-epsilon-properties}}}{\le}\; C\mE\|\Delta_{n-1}\|^2,\\
\mE\|e_n\|^2&=\mE[\mE_{n-1}\|e_n\|^2]\;\stackrel{\mathclap{\eqref{eq:C-e-remainder-properties}}}{\le}\; C\bigl(\mE\|\Delta_{n-1}\|^4+\gamma_n^2\bigr).
\end{aligned}
\end{equation*}
Combining these bounds with \eqref{eq:C-raw-moment-bounds} and $\gamma_{n-1}\le2^\alpha\gamma_n$ yields
\begin{equation}\label{eq:C-linear-remainder-bounds}
\mE\|\epsilon_n\|^2\le Cn^{-\alpha},\qquad \mE\|e_n\|^2\le Cn^{-2\alpha}.
\end{equation}
The case $n=1$ is covered by $\|\Delta_0\|\le D_\mX$. Let us recall $Y_i^n=\prod_{k=i+1}^n(I-\gamma_kA)$ and $Y_n^n=I$; then the iteration yields
\begin{equation*}
d_n=\sum_{i=1}^nY_i^n\gamma_i\epsilon_i+ \sum_{i=1}^nY_i^n\gamma_i e_i.
\end{equation*}
For $i<j$, past measurability and conditional centering imply
\begin{equation*}
\mE\langle Y_i^n\gamma_i\epsilon_i,Y_j^n\gamma_j\epsilon_j\rangle =\mE\langle Y_i^n\gamma_i\epsilon_i,Y_j^n\gamma_j\mE_{j-1}\epsilon_j\rangle=0.
\end{equation*}
Using this orthogonality, we obtain
\begin{equation*}
\mE\left\|\sum_{i=1}^nY_i^n\gamma_i\epsilon_i\right\|^2 =\sum_{i=1}^n\gamma_i^2\mE\|Y_i^n\epsilon_i\|^2 \le C\sum_{i=1}^n\|Y_i^n\|i^{-3\alpha}.
\end{equation*}
The inequality combines \eqref{eq:C-linear-remainder-bounds} with \eqref{eq:A-product-bound}, which implies $\|Y_i^n\|^2\le C\|Y_i^n\|$. For the remainder sum, Minkowski's inequality yields
\begin{equation*}
\left(\mE\left\|\sum_{i=1}^nY_i^n\gamma_i e_i\right\|^2\right)^{1/2} \le\sum_{i=1}^n\|Y_i^n\|\gamma_i(\mE\|e_i\|^2)^{1/2} \;\stackrel{\mathclap{\eqref{eq:C-linear-remainder-bounds}}}{\le}\; C\sum_{i=1}^n\|Y_i^n\|i^{-2\alpha}.
\end{equation*}
Absorbing the fixed factors for $q=3\alpha$ and $q=2\alpha$ into $C$, we obtain
\begin{align*}
\mE\left\|\sum_{i=1}^nY_i^n\gamma_i\epsilon_i\right\|^2\; \stackrel{\mathclap{\eqref{eq:A-product-sum}}}{\le}\; Cn^{-2\alpha},\qquad \qquad\left(\mE\left\|\sum_{i=1}^nY_i^n\gamma_i e_i\right\|^2\right)^{1/2}\; \stackrel{\mathclap{\eqref{eq:A-product-sum}}}{\le}\; Cn^{-\alpha}.
\end{align*}
Taking the square root of the first bound and adding the second proves the claim.
\end{proof}

We next compare arithmetic and recursive averages over an interval. For each $s\ge1$, set $c_{s:s}:=x_s$ and, for $j>s$, define
\begin{equation*}
c_{s:j}:=R_{c_{s:j-1}}\left(\frac1{j-s+1}R_{c_{s:j-1}}^{-1}(x_j)\right).
\end{equation*}
For $1\le s\le j$, write
\begin{equation*}
\barDelta_{s:j}:=\frac1{j-s+1}\sum_{k=s}^j\Delta_k, \qquad \tDelta_{s:j}:=R_{x_\star}^{-1}(c_{s:j}).
\end{equation*}
Thus $\barDelta_{1:n}=\barDelta_n$ and $\tDelta_{1:n}=\tDelta_n$ for the global averages, while $\barDelta_{t_j:j}=\barDelta_j^B$ and $\tDelta_{t_j:j}=\tDelta_j^B$ for the batch averages in \eqref{eq:tangent-averages}.

\begin{lemma}\label{lem:C-rough-average-moments}
Let $\gamma_n=\gamma n^{-\alpha}$, where $\gamma>0$ and $\alpha\in(0,1)$. Then, under Assumptions~\ref{ass:geometry} and~\ref{ass:raw-moments}(a), for every $n\ge1$ and $1\le s\le j$,
\begin{align*}
\mE\|\barDelta_n\|^2&\le\frac{C_{\Delta,2}\gamma_n}{1-\alpha}, \qquad \mE\|\tDelta_n\|^2\le\frac{C_{\Delta,2}\gamma_n}{1-\alpha},\\
\mE\|\barDelta_{s:j}\|^2&\le C_{\Delta,2}\gamma_s, \qquad \mE\|\tDelta_{s:j}\|^2\le C_{\Delta,2}\gamma_s.
\end{align*}
If Assumption~\ref{ass:raw-moments}(b) also holds, then
\begin{align*}
\mE\|\barDelta_n\|^4&\le\frac{C_{\Delta,4}\gamma_n^2}{(1-\alpha)^2}, \qquad \mE\|\tDelta_n\|^4\le\frac{C_{\Delta,4}\gamma_n^2}{(1-\alpha)^2},\\
\mE\|\barDelta_{s:j}\|^4&\le C_{\Delta,4}\gamma_s^2, \qquad \mE\|\tDelta_{s:j}\|^4\le C_{\Delta,4}\gamma_s^2.
\end{align*}
The constants $C_{\Delta,2}$ and $C_{\Delta,4}$ are those in \eqref{eq:C-raw-moment-bounds}.
\end{lemma}

\begin{proof}
By Assumption~\ref{ass:geometry}(b), $c_{s:j}\in\mX$. Convexity in Assumption~\ref{ass:geometry}(c) gives, for $j>s$,
\begin{equation*}
\|\tDelta_{s:j}\|^2\le\frac{j-s}{j-s+1}\|\tDelta_{s:j-1}\|^2+\frac1{j-s+1}\|\Delta_j\|^2.
\end{equation*}
Starting from $\tDelta_{s:s}=\Delta_s$, iteration yields
\begin{equation}\label{eq:C-interval-average-domination}
\|\tDelta_{s:j}\|^2\le\frac1{j-s+1}\sum_{k=s}^j\|\Delta_k\|^2.
\end{equation}
Jensen's inequality gives the same bound for $\|\barDelta_{s:j}\|^2$. Taking expectations and using \eqref{eq:C-raw-moment-bounds}, we obtain
\begin{equation*}
\mE\|\barDelta_{s:j}\|^2\le\frac{C_{\Delta,2}}{j-s+1}\sum_{k=s}^j\gamma_k, \qquad \mE\|\tDelta_{s:j}\|^2\le\frac{C_{\Delta,2}}{j-s+1}\sum_{k=s}^j\gamma_k.
\end{equation*}
The interval bounds follow from $\gamma_k\le\gamma_s$ for $k\ge s$. Taking $s=1$ and $j=n$ gives the global bounds because
\begin{equation}\label{eq:C-step-size-sum}
\frac1n\sum_{k=1}^n\gamma_k\le\frac\gamma n\int_0^n t^{-\alpha}\,\mathrm dt=\frac{\gamma_n}{1-\alpha}.
\end{equation}
If Assumption~\ref{ass:raw-moments}(b) also holds, Minkowski's inequality yields
\begin{equation*}
\bigl(\mE\|\tDelta_{s:j}\|^4\bigr)^{1/2}\;\,\stackrel{\mathclap{\eqref{eq:C-interval-average-domination}}}{\le}\;\,\frac1{j-s+1}\sum_{k=s}^j\bigl(\mE\|\Delta_k\|^4\bigr)^{1/2}\;\,\stackrel{\mathclap{\eqref{eq:C-raw-moment-bounds}}}{\le}\;\,\frac{C_{\Delta,4}^{1/2}}{j-s+1}\sum_{k=s}^j\gamma_k.
\end{equation*}
The same argument applies to $\barDelta_{s:j}$. Squaring and using $\gamma_k\le\gamma_s$ for the interval averages and \eqref{eq:C-step-size-sum} for the global averages proves the fourth-moment bounds. This completes the proof.
\end{proof}

\begin{lemma}\label{lem:C-retraction-arithmetic-discrepancy}
Under the assumptions of Lemma~\ref{lem:C-rough-average-moments}, there exists a constant $C<\infty$ such that, for every $n\ge1$ and $1\le s\le j$,
\begin{equation*}
\mE\|\tDelta_n-\barDelta_n\|\le C\gamma_n, \qquad \mE\|\tDelta_{s:j}-\barDelta_{s:j}\|\le C\gamma_s.
\end{equation*}
If Assumption~\ref{ass:raw-moments}(b) also holds, then, after enlarging $C$ if necessary,
\begin{equation*}
\mE\|\tDelta_n-\barDelta_n\|^2\le C\gamma_n^2, \qquad \mE\|\tDelta_{s:j}-\barDelta_{s:j}\|^2\le C\gamma_s^2.
\end{equation*}
\end{lemma}

\begin{proof}
Following \citet[Appendix~C.3]{Tripuraneni2018Averaging}, we first bound the error in one averaging step. Fix $x,z\in\mX$ and set $v=R_x^{-1}(z)$. Recall $\Phi_{x,x_\star}=R_{x_\star}^{-1}\circ R_x$ from Lemma~\ref{lem:inverse-retraction-C2}. Retraction convexity gives $tv\in R_x^{-1}(\mX)$ for $t\in[0,1]$, so that lemma bounds $D^2\Phi_{x,x_\star}(tv)$ uniformly. Taylor's formula therefore gives
\begin{equation*}
\|\Phi_{x,x_\star}(tv)-\Phi_{x,x_\star}(0_x)-tD\Phi_{x,x_\star}(0_x)[v]\|\le\frac C2 t^2\|v\|_x^2, \qquad 0\le t\le1.
\end{equation*}
For $\theta\in[0,1]$, subtracting $\theta$ times the remainder at $t=1$ from the remainder at $t=\theta$ cancels the linear terms and yields
\begin{equation*}
\|\Phi_{x,x_\star}(\theta v)-(1-\theta)\Phi_{x,x_\star}(0_x)-\theta\Phi_{x,x_\star}(v)\|\le\frac C2(\theta^2+\theta)\|v\|_x^2\le C\theta\|v\|_x^2.
\end{equation*}
Assumption~\ref{ass:geometry}(a), Lemma~\ref{lem:inverse-retraction-C2}, and compactness give $D_\mX<\rho_{\mathrm{ret}}$, so $DR_{x_\star}$ is bounded on $\|u\|\le D_\mX$. The curve $t\mapsto R_{x_\star}\bigl((1-t)R_{x_\star}^{-1}(x)+tR_{x_\star}^{-1}(z)\bigr)$, $t\in[0,1]$, then has length at most $C\|R_{x_\star}^{-1}(z)-R_{x_\star}^{-1}(x)\|$. Combining this with \eqref{eq:B-distance-equivalence} gives
\begin{equation}\label{eq:C-inverse-coordinate-distance}
\|R_x^{-1}(z)\|_x\le C\|R_{x_\star}^{-1}(z)-R_{x_\star}^{-1}(x)\|.
\end{equation}
Since $\Phi_{x,x_\star}(0_x)=R_{x_\star}^{-1}(x)$ and $\Phi_{x,x_\star}(v)=R_{x_\star}^{-1}(z)$, the Taylor estimate and \eqref{eq:C-inverse-coordinate-distance} imply
\begin{equation}
\left\|R_{x_\star}^{-1}\!\left(R_x\bigl(\theta R_x^{-1}(z)\bigr)\right)-(1-\theta)R_{x_\star}^{-1}(x)-\theta R_{x_\star}^{-1}(z)\right\| \le C\theta\|R_{x_\star}^{-1}(z)-R_{x_\star}^{-1}(x)\|^2.\label{eq:C-one-step-average-comparison}
\end{equation}
We now accumulate these errors over an interval. For $k\ge s$, take $x=c_{s:k}$, $z=x_{k+1}$, and $\theta=(k-s+2)^{-1}$. Subtracting the arithmetic update gives
\begin{equation*}
(k-s+2)\|\tDelta_{s:k+1}-\barDelta_{s:k+1}\| \;\stackrel{\mathclap{\eqref{eq:C-one-step-average-comparison}}}{\le}\;(k-s+1)\|\tDelta_{s:k}-\barDelta_{s:k}\|+C\|\Delta_{k+1}-\tDelta_{s:k}\|^2.
\end{equation*}
Since $\tDelta_{s:s}=\barDelta_{s:s}=\Delta_s$, iteration yields
\begin{equation}\label{eq:C-batch-discrepancy}
\|\tDelta_{s:j}-\barDelta_{s:j}\|\le\frac C{j-s+1}\sum_{k=s}^{j-1}\|\Delta_{k+1}-\tDelta_{s:k}\|^2.
\end{equation}
For $s=j$, the sum is empty and both sides are zero. The second-moment bounds in \eqref{eq:C-raw-moment-bounds} and Lemma~\ref{lem:C-rough-average-moments} give, for $1\le s\le k$,
\begin{equation*}
\mE\|\Delta_{k+1}-\tDelta_k\|^2\le C\gamma_k, \qquad \mE\|\Delta_{k+1}-\tDelta_{s:k}\|^2\le C\gamma_s.
\end{equation*}
Taking expectations in \eqref{eq:C-batch-discrepancy}, first with $s=1$, $j=n$ and then with arbitrary $s,j$, yields
\begin{align*}
\mE\|\tDelta_n-\barDelta_n\|&\le\frac Cn\sum_{k=1}^{n-1}\gamma_k\;\stackrel{\mathclap{\eqref{eq:C-step-size-sum}}}{\le}\;\frac C{1-\alpha}\gamma_n,\\
\mE\|\tDelta_{s:j}-\barDelta_{s:j}\|&\le C\frac{j-s}{j-s+1}\gamma_s\le C\gamma_s.
\end{align*}
This proves the first-moment bounds. If Assumption~\ref{ass:raw-moments}(b) also holds, the fourth-moment bounds in \eqref{eq:C-raw-moment-bounds} and Lemma~\ref{lem:C-rough-average-moments} give
\begin{equation*}
\bigl(\mE\|\Delta_{k+1}-\tDelta_k\|^4\bigr)^{1/2}\le C\gamma_k, \qquad \bigl(\mE\|\Delta_{k+1}-\tDelta_{s:k}\|^4\bigr)^{1/2}\le C\gamma_s.
\end{equation*}
Minkowski's inequality yields
\begin{equation*}
\bigl(\mE\|\tDelta_{s:j}-\barDelta_{s:j}\|^2\bigr)^{1/2}\;\,\stackrel{\mathclap{\eqref{eq:C-batch-discrepancy}}}{\le}\;\,\frac C{j-s+1}\sum_{k=s}^{j-1}\bigl(\mE\|\Delta_{k+1}-\tDelta_{s:k}\|^4\bigr)^{1/2}.
\end{equation*}
Substituting the preceding fourth-moment bounds and using \eqref{eq:C-step-size-sum} for the global average gives
\begin{equation*}
\bigl(\mE\|\tDelta_n-\barDelta_n\|^2\bigr)^{1/2}\le C\gamma_n, \qquad \bigl(\mE\|\tDelta_{s:j}-\barDelta_{s:j}\|^2\bigr)^{1/2}\le C\gamma_s.
\end{equation*}
Squaring and enlarging $C$ proves the second-moment bounds and completes the proof.
\end{proof}

\section{Proof of Theorem~\ref{thm:covariance-rate}}\label{app:D}

We prove the result for both OL and NOL using the four-term decomposition in Section~\ref{sec:3.4}. Fix $\diamond\in\{\mathrm{OL},\mathrm{NOL}\}$. Recall the selected indices $E_n^\diamond$ from \eqref{eq:ol-estimator} and \eqref{eq:nol-endpoints}, and write
\begin{equation*}
v_{n,\diamond}:=\sum_{j\in E_n^\diamond}\ell_j,\qquad v_{n,\mathrm{OL}}=\sum_{j=1}^n\ell_j,\qquad v_{n,\mathrm{NOL}}=n.
\end{equation*}
Here $B_j=\{x_{t_j},\ldots,x_j\}$ is a batch of iterates, $\ell_j=|B_j|$, and $I_j=\{t_j,\ldots,j\}$ is its index set, as in Lemma~\ref{lem:block-calculus}. Recall its mean $c_j$ from \eqref{eq:batch-average}, the exact vector $b_{j,n}=R_{\wt x_n}^{-1}(c_j)$, and its approximation $\tb_{j,n}$ from \eqref{eq:batch-vector-update}. Using the averages in \eqref{eq:tangent-averages}, extend the four operators in \eqref{eq:ol-comparison-operators} to either batch selection:
\begin{align*}
\hSigma_{n,\diamond}^{(0)} &:=(\Gamma_{x_\star}^{\wt x_n})^{-1}\circ\wt\Sigma_{n,\diamond}\circ\Gamma_{x_\star}^{\wt x_n} =\frac1{v_{n,\diamond}}\sum_{j\in E_n^\diamond}\ell_j^2 \bigl((\Gamma_{x_\star}^{\wt x_n})^{-1}\tb_{j,n}\bigr)^{\otimes2},\\
\hSigma_{n,\diamond}^{(1)} &:=\frac1{v_{n,\diamond}}\sum_{j\in E_n^\diamond}\ell_j^2 \bigl((\Gamma_{x_\star}^{\wt x_n})^{-1}b_{j,n}\bigr)^{\otimes2},\\
\hSigma_{n,\diamond}^{(2)} &:=\frac1{v_{n,\diamond}}\sum_{j\in E_n^\diamond}\ell_j^2 (\tDelta_j^B-\tDelta_n)^{\otimes2},\\
\hSigma_{n,\diamond}^{(3)} &:=\frac1{v_{n,\diamond}}\sum_{j\in E_n^\diamond} \left(\sum_{k=t_j}^j\Delta_k-\ell_j\barDelta_n\right)^{\otimes2} =\frac1{v_{n,\diamond}}\sum_{j\in E_n^\diamond}\ell_j^2 (\barDelta_j^B-\barDelta_n)^{\otimes2}.
\end{align*}
All four operators act on $\T_{x_\star}\mM$. Isometry of parallel transport identifies its inverse with its adjoint in the representation of $\hSigma_{n,\diamond}^{(0)}$. Hence
\begin{equation}\label{eq:D-ol-four-term-decomposition}
\mE\|\hSigma_{n,\diamond}^{(0)}-\Sigma\| \le\mE\|\hSigma_{n,\diamond}^{(0)}-\hSigma_{n,\diamond}^{(1)}\| +\mE\|\hSigma_{n,\diamond}^{(1)}-\hSigma_{n,\diamond}^{(2)}\| +\mE\|\hSigma_{n,\diamond}^{(2)}-\hSigma_{n,\diamond}^{(3)}\| +\mE\|\hSigma_{n,\diamond}^{(3)}-\Sigma\|.
\end{equation}
The first three terms correspond to online vector updates, changes of base point, and averaging on the manifold. The last is the batch-means error in the fixed tangent space. We first bound that last term in Section~\ref{app:D-batch-means}, obtaining $\mE\|\hSigma_{n,\diamond}^{(3)}-\Sigma\| \le C\bigl(n^{-1/(2\beta)}+n^{(\alpha-1)/2+1/(2\beta)}\bigr).$ Section~\ref{app:D-geometry} then bounds each of the first three terms by $Cn^{1/2-1/(2\beta)-\alpha}$. We combine these bounds at the end of the appendix.

\subsection{Batch-means error in the fixed tangent space}\label{app:D-batch-means}
We first bound the noise-based estimator, then compare it with the linear and RSGD estimators. These bounds also control the expected operator norms needed below. Set $\zeta_k:=-\grad f(x_\star,\xi_k)$, $k\ge1$; these vectors are independent and identically distributed by Section~\ref{sec:2.2}. Recall that
\begin{equation*}
A=\Hess F(x_\star), \qquad S=\mE_{\xi\sim\mathcal D} [\grad f(x_\star,\xi)^{\otimes2}], \qquad \Sigma=A^{-1}SA^{-1}=\mE[(A^{-1}\zeta_1)^{\otimes2}].
\end{equation*}
Let $(U_k)$ be the coupled linear recursion in \eqref{eq:C-linear-recursion}, and set $\bar U_n:=n^{-1}\sum_{k=1}^nU_k$. For every $n\ge1$, define
\begin{align}\label{eq:D-linear-estimators}
\hSigma_{n,\diamond}^{(3,1)} &:= \frac1{v_{n,\diamond}}\sum_{j\in E_n^{\diamond}} \left(\sum_{k=t_j}^{j}U_k-\ell_j\bar U_n\right)^{\otimes2},\notag\\
\hSigma_{n,\diamond}^{(3,2)} &:= \frac1{v_{n,\diamond}}\sum_{j\in E_n^{\diamond}} \left(\sum_{k=t_j}^{j}A^{-1}\zeta_k -\frac{\ell_j}{n}\sum_{k=1}^{n}A^{-1}\zeta_k\right)^{\otimes2}.
\end{align}
Thus, $\hSigma_{n,\diamond}^{(3,1)}$ replaces $\Delta_k$ in $\hSigma_{n,\diamond}^{(3)}$ by $U_k$, whereas $\hSigma_{n,\diamond}^{(3,2)}$ replaces $U_k$ by $A^{-1}\zeta_k$.

\begin{lemma}
\label{lem:D-ol-ideal}
Let $\alpha\in(1/2,1)$, and let $a_1=1$ and $a_m=\lfloor cm^\beta\rfloor$ for $m\ge2$, where $c\ge1$ and $\beta>1/(1-\alpha)$. Then, under Assumptions~\ref{ass:objective} and~\ref{ass:noise}, there exists $C<\infty$ such that the noise-based covariance estimator defined in \eqref{eq:D-linear-estimators} satisfies, for every $n\ge1$ and $\diamond\in\{\mathrm{OL},\mathrm{NOL}\}$,
\begin{equation*}
\mE\|\hSigma_{n,\diamond}^{(3,2)}-\Sigma\| \le Cn^{-1/(2\beta)}.
\end{equation*}
\end{lemma}

\begin{proof}
By Assumptions~\ref{ass:objective} and~\ref{ass:noise}, $\mE(A^{-1}\zeta_k)=0$, $\mE(A^{-1}\zeta_k)^{\otimes2}=\Sigma$, and $\mE\|A^{-1}\zeta_k\|^4\le\mu^{-4}\tau^4$, using $\|A^{-1}\|\le\mu^{-1}$. First omit the common centering term:
\begin{equation*}
\hSigma_{n,\diamond}^{(3,2,0)}:=v_{n,\diamond}^{-1}\sum_{j\in E_n^\diamond}\left(\sum_{k=t_j}^jA^{-1}\zeta_k\right)^{\otimes2}.
\end{equation*}
Independence and centering imply $\mE (A^{-1}\zeta_k)(A^{-1}\zeta_r)^\top=0$ for $k\ne r$, so
\begin{equation*}
\mE \left(\sum_{k=t_j}^jA^{-1}\zeta_k\right)^{\otimes2}=\sum_{k,r=t_j}^j\mE((A^{-1}\zeta_k)(A^{-1}\zeta_r)^\top)=\ell_j\Sigma.
\end{equation*}
Since $\sum_{j\in E_n^\diamond}\ell_j=v_{n,\diamond}$, the uncentered estimator has expectation $\Sigma$. In the fourth-moment expansion, independence and centering remove every term with an index appearing only once. For $k\ne r$, independence yields
\begin{align*}
\mE(\|A^{-1}\zeta_k\|^2\|A^{-1}\zeta_r\|^2)&=(\mE\|A^{-1}\zeta_k\|^2)(\mE\|A^{-1}\zeta_r\|^2)=(\tr\Sigma)^2,\\
\mE\langle A^{-1}\zeta_k,A^{-1}\zeta_r\rangle^2&=\tr\bigl(\mE[(A^{-1}\zeta_k)(A^{-1}\zeta_k)^\top(A^{-1}\zeta_r)(A^{-1}\zeta_r)^\top]\bigr) =\tr(\Sigma^2).
\end{align*}
The remaining terms in the expansion therefore yield
\begin{align}
\mE\left\|\sum_{k=t_j}^jA^{-1}\zeta_k\right\|^4&=\sum_{k=t_j}^j\mE\|A^{-1}\zeta_k\|^4+2\sum_{t_j\le k<r\le j}\mE(\|A^{-1}\zeta_k\|^2\|A^{-1}\zeta_r\|^2)\notag\\
&\quad+4\sum_{t_j\le k<r\le j}\mE\langle A^{-1}\zeta_k,A^{-1}\zeta_r\rangle^2\notag\\
&=\ell_j\mE\|A^{-1}\zeta_1\|^4+\ell_j(\ell_j-1)\bigl((\tr\Sigma)^2+2\tr(\Sigma^2)\bigr) \le C\ell_j^2.\label{eq:D-noise-sum-fourth-moment}
\end{align}
The inequality uses the fourth-moment bound and $\ell_j\ge1$. Since $\|M\|_{\mathrm F}^2=\tr(M^\top M)$ and $\|vv^\top\|_{\mathrm F}^2=\|v\|^4$, the second-moment identity implies
\begin{align}
\mE\left\|\left(\left(\sum_{k=t_j}^jA^{-1}\zeta_k\right)^{\otimes2}-\ell_j\Sigma\right)\right\|_{\mathrm F}^2&=\mE\left\|\sum_{k=t_j}^jA^{-1}\zeta_k\right\|^4-2\ell_j\tr\left[\mE\left(\sum_{k=t_j}^jA^{-1}\zeta_k\right)^{\otimes2}\Sigma\right]+\ell_j^2\|\Sigma\|_{\mathrm F}^2\notag\\
&=\mE\left\|\left(\sum_{k=t_j}^jA^{-1}\zeta_k\right)\right\|^4-\ell_j^2\|\Sigma\|_{\mathrm F}^2\;\stackrel{\mathclap{\eqref{eq:D-noise-sum-fourth-moment}}}{\le}\; C\ell_j^2.\label{eq:D-fourth-moment-bound}
\end{align}
We now bound the cross terms separately for OL and NOL.

\emph{OL.} If $I_j$ and $I_q$ lie in different blocks, independence and centering imply
\begin{align*}
&\mE\tr\left[\left(\left(\sum_{k=t_j}^jA^{-1}\zeta_k\right)^{\otimes2}-\ell_j\Sigma\right)\left(\left(\sum_{k=t_q}^qA^{-1}\zeta_k\right)^{\otimes2}-\ell_q\Sigma\right)\right]\\
&\quad=\tr\left[\mE\left(\left(\sum_{k=t_j}^jA^{-1}\zeta_k\right)^{\otimes2}-\ell_j\Sigma\right)\,\mE\left(\left(\sum_{k=t_q}^qA^{-1}\zeta_k\right)^{\otimes2}-\ell_q\Sigma\right)\right]=0.
\end{align*}
Within one block, let $j\le q$. Since $t_j=t_q$, we have
\begin{equation*}
\left(\sum_{k=t_q}^qA^{-1}\zeta_k\right)=\left(\sum_{k=t_j}^jA^{-1}\zeta_k\right)+\sum_{k=j+1}^q(A^{-1}\zeta_k),\qquad \ell_q-\ell_j=q-j.
\end{equation*}
The first term is $\mF_j$-measurable, while the remaining observations are independent of $\mF_j$. Thus
\begin{equation*}
\mE\left[\sum_{k=j+1}^q(A^{-1}\zeta_k)\mid\mF_j\right]=0,\qquad \mE\left[\left(\sum_{k=j+1}^q(A^{-1}\zeta_k)\right)^{\otimes2}\mid\mF_j\right]=(\ell_q-\ell_j)\Sigma.
\end{equation*}
Expanding the outer product and conditioning on $\mF_j$, we obtain
\begin{align*}
\mE\left[\left(\left(\sum_{k=t_q}^qA^{-1}\zeta_k\right)^{\otimes2}-\ell_q\Sigma\right)\mid\mF_j\right]&=\left(\sum_{k=t_j}^jA^{-1}\zeta_k\right)^{\otimes2}+\left(\sum_{k=t_j}^jA^{-1}\zeta_k\right)\mE\left[\sum_{k=j+1}^qA^{-1}\zeta_k\mid\mF_j\right]^\top\\
&\quad+\mE\left[\sum_{k=j+1}^qA^{-1}\zeta_k\mid\mF_j\right]\left(\sum_{k=t_j}^jA^{-1}\zeta_k\right)^\top\\
&\quad+\mE\left[\left(\sum_{k=j+1}^qA^{-1}\zeta_k\right)^{\otimes2}\mid\mF_j\right]-\ell_q\Sigma\\
&=\left(\sum_{k=t_j}^jA^{-1}\zeta_k\right)^{\otimes2}+(\ell_q-\ell_j)\Sigma-\ell_q\Sigma\\
&=\left(\left(\sum_{k=t_j}^jA^{-1}\zeta_k\right)^{\otimes2}-\ell_j\Sigma\right).
\end{align*}
Since the centered outer product for batch $j$ is $\mF_j$-measurable, the tower property yields
\begin{align}
&\mE\tr\left[\left(\left(\sum_{k=t_j}^jA^{-1}\zeta_k\right)^{\otimes2}-\ell_j\Sigma\right)\left(\left(\sum_{k=t_q}^qA^{-1}\zeta_k\right)^{\otimes2}-\ell_q\Sigma\right)\right]\notag\\
&\quad=\mE\left(\mE\left[\tr\left[\left(\left(\sum_{k=t_j}^jA^{-1}\zeta_k\right)^{\otimes2}-\ell_j\Sigma\right)\left(\left(\sum_{k=t_q}^qA^{-1}\zeta_k\right)^{\otimes2}-\ell_q\Sigma\right)\right]\mid\mF_j\right]\right)\notag\\
&\quad=\mE\tr\left[\left(\left(\sum_{k=t_j}^jA^{-1}\zeta_k\right)^{\otimes2}-\ell_j\Sigma\right)\mE\left[\left(\left(\sum_{k=t_q}^qA^{-1}\zeta_k\right)^{\otimes2}-\ell_q\Sigma\right)\mid\mF_j\right]\right]\notag\\
&\quad=\mE\tr\left[\left(\left(\sum_{k=t_j}^jA^{-1}\zeta_k\right)^{\otimes2}-\ell_j\Sigma\right)^2\right] =\mE\left\|\left(\left(\sum_{k=t_j}^jA^{-1}\zeta_k\right)^{\otimes2}-\ell_j\Sigma\right)\right\|_{\mathrm F}^2.\label{eq:D-ol-overlap-identity}
\end{align}
Cauchy--Schwarz and \eqref{eq:D-fourth-moment-bound} ensure integrability. The identities also hold for $j=q$, with empty sums. Recall the block count $m_n$ and lengths $L_m=\min\{a_{m+1}-1,n\}-a_m+1$ from Lemma~\ref{lem:block-calculus}. Summing over batch lengths $r,s$ within each block yields
\begin{align*}
\mE\|\hSigma_{n,\mathrm{OL}}^{(3,2,0)}-\Sigma\|_{\mathrm F}^2&=\frac1{v_{n,\mathrm{OL}}^2}\sum_{j,q\in E_n^{\mathrm{OL}}}\mE\tr\left[\left(\left(\sum_{k=t_j}^jA^{-1}\zeta_k\right)^{\otimes2}-\ell_j\Sigma\right)\left(\left(\sum_{k=t_q}^qA^{-1}\zeta_k\right)^{\otimes2}-\ell_q\Sigma\right)\right]\\
&\le\frac{C}{v_{n,\mathrm{OL}}^2}\sum_{m=1}^{m_n}\sum_{r,s=1}^{L_m}\min(r,s)^2 \le\frac{C}{v_{n,\mathrm{OL}}^2}\sum_{m=1}^{m_n}L_m^4.
\end{align*}
The first inequality combines \eqref{eq:D-ol-overlap-identity} and \eqref{eq:D-fourth-moment-bound} within blocks, exchanging $j,q$ if needed; cross terms between blocks are zero. The second uses $\min(r,s)^2\le L_m^2$.

\emph{NOL.} In contrast, the sets $I_j$, $j\in E_n^{\mathrm{NOL}}$, are disjoint, so their noise sums are independent. For distinct selected indices $j,q$,
\begin{align*}
&\mE\tr\left[\left(\left(\sum_{k=t_j}^jA^{-1}\zeta_k\right)^{\otimes2}-\ell_j\Sigma\right)\left(\left(\sum_{k=t_q}^qA^{-1}\zeta_k\right)^{\otimes2}-\ell_q\Sigma\right)\right]\\
&\quad=\tr\left[(\mE\left(\left(\sum_{k=t_j}^jA^{-1}\zeta_k\right)^{\otimes2}-\ell_j\Sigma\right))(\mE\left(\left(\sum_{k=t_q}^qA^{-1}\zeta_k\right)^{\otimes2}-\ell_q\Sigma\right))\right]=0.
\end{align*}
Consequently, we have
\begin{align*}
\mE\|\hSigma_{n,\mathrm{NOL}}^{(3,2,0)}-\Sigma\|_{\mathrm F}^2&=\frac1{n^2}\sum_{j,q\in E_n^{\mathrm{NOL}}}\mE\tr\left[\left(\left(\sum_{k=t_j}^jA^{-1}\zeta_k\right)^{\otimes2}-\ell_j\Sigma\right)\left(\left(\sum_{k=t_q}^qA^{-1}\zeta_k\right)^{\otimes2}-\ell_q\Sigma\right)\right]\\
&=\frac1{n^2}\sum_{j\in E_n^{\mathrm{NOL}}}\mE\left\|\left(\left(\sum_{k=t_j}^jA^{-1}\zeta_k\right)^{\otimes2}-\ell_j\Sigma\right)\right\|_{\mathrm F}^2\\
\;&\stackrel{\mathclap{\eqref{eq:D-fourth-moment-bound}}}{\le}\;\frac{C}{n^2}\sum_{j\in E_n^{\mathrm{NOL}}}\ell_j^2.
\end{align*}
Lemma~\ref{lem:block-calculus}(ii) for OL and (i) for NOL now imply
\begin{equation*}
\mE\|\hSigma_{n,\diamond}^{(3,2,0)}-\Sigma\|_{\mathrm F}^2\le Cn^{-1/\beta},\qquad \diamond\in\{\mathrm{OL},\mathrm{NOL}\}.
\end{equation*}

For either choice of $\diamond$, Cauchy--Schwarz and $\|\cdot\|\le\|\cdot\|_{\mathrm F}$ now yield
\begin{equation}\label{eq:D-ol-uncentered-error}
\mE\|\hSigma_{n,\diamond}^{(3,2,0)}-\Sigma\| \le\bigl(\mE\|\hSigma_{n,\diamond}^{(3,2,0)}-\Sigma\|_{\mathrm F}^2\bigr)^{1/2} \le Cn^{-1/(2\beta)}.
\end{equation}
In particular, $\sup_{n,\diamond}\mE\|\hSigma_{n,\diamond}^{(3,2,0)}\|<\infty$. To restore the centering term, independence and centering imply
\begin{equation*}
\mE\left\|\frac1n\sum_{k=1}^nA^{-1}\zeta_k\right\|^2 =\frac1{n^2}\sum_{k,r=1}^n\mE\langle A^{-1}\zeta_k,A^{-1}\zeta_r\rangle =\frac1{n^2}\sum_{k=1}^n\mE\|A^{-1}\zeta_k\|^2=\frac{\tr(\Sigma)}n.
\end{equation*}
Lemma~\ref{lem:block-calculus}(i) then implies
\begin{equation}\label{eq:D-ol-centering-error}
\frac1{v_{n,\diamond}}\sum_{j\in E_n^\diamond}\ell_j^2\mE\left\|\left(\frac1n\sum_{r=1}^nA^{-1}\zeta_r\right)\right\|^2 =\frac{\tr(\Sigma)}{nv_{n,\diamond}}\sum_{j\in E_n^\diamond}\ell_j^2 \le Cn^{-1/\beta}.
\end{equation}
The centering changes the estimator by
\begin{align*}
\hSigma_{n,\diamond}^{(3,2)}-\hSigma_{n,\diamond}^{(3,2,0)}&=\frac1{v_{n,\diamond}}\sum_{j\in E_n^\diamond}\ell_j^2\left(\frac1n\sum_{r=1}^nA^{-1}\zeta_r\right)^{\otimes2} -\frac1{v_{n,\diamond}}\sum_{j\in E_n^\diamond}\ell_j\left(\sum_{k=t_j}^jA^{-1}\zeta_k\right)\left(\frac1n\sum_{r=1}^nA^{-1}\zeta_r\right)^\top\\
&\quad-\frac1{v_{n,\diamond}}\sum_{j\in E_n^\diamond}\ell_j\left(\frac1n\sum_{r=1}^nA^{-1}\zeta_r\right)\left(\sum_{k=t_j}^jA^{-1}\zeta_k\right)^\top.
\end{align*}
Independence and centering imply
\begin{equation*}
\mE\left[\left(\sum_{k=t_j}^jA^{-1}\zeta_k\right)\left(\frac1n\sum_{r=1}^nA^{-1}\zeta_r\right)^\top\right] =\frac1n\sum_{k=t_j}^j\sum_{r=1}^n\mE[(A^{-1}\zeta_k)(A^{-1}\zeta_r)^\top] =\frac1n\sum_{k=t_j}^j\Sigma=\frac{\ell_j}{n}\Sigma.
\end{equation*}
These mixed terms need not vanish. Apply Lemma~\ref{lem:quadratic-comparison} over $E_n^\diamond$, with $D=v_{n,\diamond}$, $a_j=1$, $\sigma_j=\sum_{k=t_j}^jA^{-1}\zeta_k$, and $e_j=-\ell_j n^{-1}\sum_{r=1}^nA^{-1}\zeta_r$. By \eqref{eq:D-ol-centering-error} and the bounded expected norm following~\eqref{eq:D-ol-uncentered-error},
\begin{equation*}
\mE\|\hSigma_{n,\diamond}^{(3,2)}-\hSigma_{n,\diamond}^{(3,2,0)}\| \le C\bigl(n^{-1/(2\beta)}+n^{-1/\beta}\bigr) \le 2Cn^{-1/(2\beta)}.
\end{equation*}
Since $n^{-1/\beta}\le n^{-1/(2\beta)}$, combining this estimate with \eqref{eq:D-ol-uncentered-error} proves the result.
\end{proof}

\begin{lemma}
\label{lem:D-ol-linear-to-ideal}
Under the assumptions of Lemma~\ref{lem:D-ol-ideal}, suppose that Assumptions~\ref{ass:geometry} and~\ref{ass:raw-moments} also hold, and let $\gamma_k=\gamma k^{-\alpha}$ with $\gamma>0$. Then there exists $C<\infty$ such that the covariance estimators based on $U_k$ and $A^{-1}\zeta_k$ in \eqref{eq:D-linear-estimators} satisfy, for every $n\ge1$ and $\diamond\in\{\mathrm{OL},\mathrm{NOL}\}$,
\begin{equation*}
\mE\|\hSigma_{n,\diamond}^{(3,1)} -\hSigma_{n,\diamond}^{(3,2)}\| \le Cn^{(\alpha-1)/2+1/(2\beta)}.
\end{equation*}
\end{lemma}

\begin{proof}
Lemma~\ref{lem:linear-interval-remainder} supplies $\mE\|r^U_{s,t}\|^2\le Ct^\alpha$, and its proof shows $r^U_{s,t}=\sum_{k=s}^t(U_k-A^{-1}\zeta_k)$. The centered batch sums differ by
\begin{multline*}
\left(\sum_{k=t_j}^j U_k-\ell_j\bar U_n\right)
-\left(\sum_{k=t_j}^j A^{-1}\zeta_k-\frac{\ell_j}{n}\sum_{k=1}^n A^{-1}\zeta_k\right)\\
=\sum_{k=t_j}^j(U_k-A^{-1}\zeta_k)-\frac{\ell_j}{n}\sum_{k=1}^n(U_k-A^{-1}\zeta_k) =r^U_{t_j,j}-\frac{\ell_j}{n}r^U_{1,n}.
\end{multline*}
For either batch rule, the remainder bound yields
\begin{equation*}
\frac1{v_{n,\diamond}}\sum_{j\in E_n^{\diamond}}\mE\left\|r^U_{t_j,j}-\frac{\ell_j}{n}r^U_{1,n}\right\|^2\le\frac{C}{v_{n,\diamond}}\sum_{j\in E_n^{\diamond}}j^\alpha +\frac{Cn^\alpha}{n^2v_{n,\diamond}}\sum_{j\in E_n^{\diamond}}\ell_j^2.
\end{equation*}
Applying Lemma~\ref{lem:block-calculus}(iv) to the first sum and (i) to the second yields
\begin{equation}\label{eq:D-ol-linear-remainder}
\frac1{v_{n,\diamond}}\sum_{j\in E_n^{\diamond}}\mE\left\|r^U_{t_j,j}-\frac{\ell_j}{n}r^U_{1,n}\right\|^2\le C\bigl(n^{\alpha-1+1/\beta}+n^{\alpha-1-1/\beta}\bigr)\le 2Cn^{\alpha-1+1/\beta}.
\end{equation}
The last step uses $\beta>0$. In Lemma~\ref{lem:quadratic-comparison}, take indices $E_n^\diamond$, $D=v_{n,\diamond}$, $a_j=1$, and
\begin{equation*}
\sigma_j:=\sum_{k=t_j}^jA^{-1}\zeta_k-\frac{\ell_j}{n}\sum_{k=1}^nA^{-1}\zeta_k, \qquad e_j:=r^U_{t_j,j}-\frac{\ell_j}{n}r^U_{1,n}.
\end{equation*}
Then $Q=\hSigma_{n,\diamond}^{(3,1)}$ and $Q_0=\hSigma_{n,\diamond}^{(3,2)}$, with $\sup_n\mE\|Q_0\|<\infty$ by Lemma~\ref{lem:D-ol-ideal}. Combining \eqref{eq:A-quadratic-comparison} and \eqref{eq:D-ol-linear-remainder} yields
\begin{equation*}
\mE\|\hSigma_{n,\diamond}^{(3,1)}-\hSigma_{n,\diamond}^{(3,2)}\| \le C\bigl(n^{(\alpha-1)/2+1/(2\beta)}+n^{\alpha-1+1/\beta}\bigr) \le 2C n^{(\alpha-1)/2+1/(2\beta)}.
\end{equation*}
The last step uses $\alpha-1+1/\beta<0$, proving the claim.
\end{proof}

\begin{lemma}
\label{lem:D-ol-perturbed-linear}
Under the assumptions of Lemma~\ref{lem:D-ol-linear-to-ideal}, there exists $C<\infty$ such that the covariance estimators based on $\Delta_k$ and its linear approximation $U_k$ satisfy, for every $n\ge1$ and $\diamond\in\{\mathrm{OL},\mathrm{NOL}\}$,
\begin{equation*}
\mE\|\hSigma_{n,\diamond}^{(3)} -\hSigma_{n,\diamond}^{(3,1)}\| \le Cn^{1/2-1/(2\beta)-\alpha}.
\end{equation*}
\end{lemma}

\begin{proof}
Recall $d_k=\Delta_k-U_k$. Since $\barDelta_n-\bar U_n=n^{-1}\sum_{i=1}^nd_i$, Minkowski's inequality and Lemma~\ref{lem:C-linear-coupling} yield
\begin{equation*}
\left(\mE\left\|\sum_{k=t_j}^j d_k-\frac{\ell_j}{n}\sum_{i=1}^n d_i\right\|^2\right)^{1/2} \le C\sum_{k=t_j}^j k^{-\alpha} +\frac{C\ell_j}{n}\sum_{i=1}^n i^{-\alpha} \le C\left(1+\frac1{1-\alpha}\right)\ell_jt_j^{-\alpha}.
\end{equation*}
The last step uses $k\ge t_j$, $\sum_{i=1}^n i^{-\alpha}\le n^{1-\alpha}/(1-\alpha)$, and $t_j\le n$. Squaring and applying Lemma~\ref{lem:block-calculus}(iii) yields
\begin{equation}\label{eq:D-ol-coupling-remainder}
\frac1{v_{n,\diamond}}\sum_{j\in E_n^{\diamond}}\mE\left\|\sum_{k=t_j}^j d_k-\frac{\ell_j}{n}\sum_{i=1}^n d_i\right\|^2 \le C n^{1-1/\beta-2\alpha}.
\end{equation}
In Lemma~\ref{lem:quadratic-comparison}, take indices $E_n^\diamond$, $D=v_{n,\diamond}$, $a_j=1$, $\sigma_j=\sum_{k=t_j}^jU_k-\ell_j\bar U_n$, and $e_j=\sum_{k=t_j}^jd_k-\frac{\ell_j}{n}\sum_{i=1}^nd_i$. Then $Q=\hSigma_{n,\diamond}^{(3)}$ and $Q_0=\hSigma_{n,\diamond}^{(3,1)}$, with $\sup_n\mE\|Q_0\|<\infty$ by Lemmas~\ref{lem:D-ol-ideal} and~\ref{lem:D-ol-linear-to-ideal}. Combining \eqref{eq:A-quadratic-comparison} and \eqref{eq:D-ol-coupling-remainder} yields
\begin{equation*}
\mE\|\hSigma_{n,\diamond}^{(3)}-\hSigma_{n,\diamond}^{(3,1)}\| \le C\bigl(n^{1/2-1/(2\beta)-\alpha}+n^{1-1/\beta-2\alpha}\bigr) \le 2Cn^{1/2-1/(2\beta)-\alpha}.
\end{equation*}
The last step uses $1-1/\beta-2\alpha<0$, proving the claim.
\end{proof}

\begin{proposition}
\label{prop:ol-euclidean-bm}
Under the assumptions of Lemma~\ref{lem:D-ol-linear-to-ideal}, there exists $C<\infty$ such that, for every $n\ge1$ and $\diamond\in\{\mathrm{OL},\mathrm{NOL}\}$,
\begin{equation*}
\mE\|\hSigma_{n,\diamond}^{(3)}-\Sigma\| \le C \left(n^{-1/(2\beta)} +n^{(\alpha-1)/2+1/(2\beta)}\right).
\end{equation*}
\end{proposition}

\begin{proof}
Lemmas~\ref{lem:D-ol-ideal}, \ref{lem:D-ol-linear-to-ideal}, and~\ref{lem:D-ol-perturbed-linear} yield
\begin{equation*}
\mE\|\hSigma_{n,\diamond}^{(3)}-\Sigma\| \le C\bigl(n^{-1/(2\beta)}+n^{(\alpha-1)/2+1/(2\beta)}+n^{1/2-1/(2\beta)-\alpha}\bigr).
\end{equation*}
Since $\alpha>1/2$, the last term equals $n^{-1/(2\beta)}n^{1/2-\alpha}$ and can be absorbed into the first. Both remaining exponents are negative because $\beta>1/(1-\alpha)$, so $\sup_n\mE\|\hSigma_{n,\diamond}^{(3)}\|<\infty$. This proves the claim.
\end{proof}

\subsection{Geometric errors}\label{app:D-geometry}

Starting from Proposition~\ref{prop:ol-euclidean-bm}, we bound the first three terms of \eqref{eq:D-ol-four-term-decomposition} in reverse order: averaging, change of base point, and online vector updates. This establishes $\sup_n\mE\|Q_0\|<\infty$ before each use of Lemma~\ref{lem:quadratic-comparison}. All three applications take $D=v_{n,\diamond}$ and $a_j=\ell_j^2$ for $j\in E_n^\diamond$, with $Q_0$ formed from $\sigma_j$ and $Q$ from $\sigma_j+e_j$. If $\mE\|e_j\|^2\le Ct_j^{-2\alpha}$, then
\begin{equation*}
\delta=\frac1{v_{n,\diamond}}\sum_{j\in E_n^\diamond}\ell_j^2\mE\|e_j\|^2 \le\frac{C}{v_{n,\diamond}}\sum_{j\in E_n^\diamond}\ell_j^2t_j^{-2\alpha}.
\end{equation*}
Lemma~\ref{lem:block-calculus}(iii) therefore yields
\begin{equation}\label{eq:D-weighted-geometric-error}
\delta\le Cn^{1-1/\beta-2\alpha}.
\end{equation}
If $\sup_n\mE\|Q_0\|<\infty$, we obtain
\begin{equation*}
\mE\|Q-Q_0\|\;\stackrel{\mathclap{\eqref{eq:A-quadratic-comparison}}}{\le}\; C(\delta^{1/2}+\delta).
\end{equation*}
Since $1-1/\beta-2\alpha<0$, $n^{1-1/\beta-2\alpha}\le n^{1/2-1/(2\beta)-\alpha}$. Substituting the bound on $\delta$ and absorbing the smaller term yields
\begin{equation}\label{eq:D-geometric-comparison}
\mE\|Q-Q_0\|\;\,\stackrel{\mathclap{\eqref{eq:D-weighted-geometric-error}}}{\le}\;\, Cn^{1/2-1/(2\beta)-\alpha}.
\end{equation}
Each proof below identifies the reference and error vectors and verifies the required bound on $Q_0$.

\paragraph{Error from averaging on the manifold.}
We first bound the difference between recursive retraction averaging and arithmetic averaging in $\T_{x_\star}\mM$.

\begin{proposition}
\label{prop:ol-manifold-tangent-error}
Under the assumptions of Lemma~\ref{lem:D-ol-linear-to-ideal}, there exists $C<\infty$ such that, for every $n\ge1$ and $\diamond\in\{\mathrm{OL},\mathrm{NOL}\}$,
\begin{equation*}
\mE\|\hSigma_{n,\diamond}^{(2)} -\hSigma_{n,\diamond}^{(3)}\| \le Cn^{1/2-1/(2\beta)-\alpha}.
\end{equation*}
\end{proposition}

\begin{proof}
For $j\in E_n^\diamond$, Lemma~\ref{lem:C-retraction-arithmetic-discrepancy} yields
\begin{equation*}
\mE\|(\tDelta_j^B-\tDelta_n)-(\barDelta_j^B-\barDelta_n)\|^2\le C(\gamma_{t_j}^2+\gamma_n^2)\le 2C\gamma^2t_j^{-2\alpha}.
\end{equation*}
The second inequality uses $\gamma_n\le\gamma_{t_j}=\gamma t_j^{-\alpha}$. In \eqref{eq:D-geometric-comparison}, take $\sigma_j=\barDelta_j^B-\barDelta_n$ and $e_j=(\tDelta_j^B-\tDelta_n)-(\barDelta_j^B-\barDelta_n)$, so $Q_0=\hSigma_{n,\diamond}^{(3)}$ and $Q=\hSigma_{n,\diamond}^{(2)}$. Proposition~\ref{prop:ol-euclidean-bm} implies $\sup_n\mE\|Q_0\|<\infty$, verifying the remaining condition and proving the claim.
\end{proof}

\paragraph{Error from changing the base point.}
We next compare $(\Gamma_{x_\star}^{\wt x_n})^{-1}b_{j,n}$ with $\tDelta_j^B-\tDelta_n$ for each batch and then sum the resulting bounds with weights $\ell_j^2/v_{n,\diamond}$.

\begin{lemma}
\label{lem:D-coordinate-endpoint}
Under the assumptions of Lemma~\ref{lem:C-rough-average-moments}, suppose additionally that Assumption~\ref{ass:raw-moments}(b) holds. For each batch $B_j=\{x_{t_j},\ldots,x_j\}$, let $c_j$ be its recursive mean from \eqref{eq:batch-average}. Recall that $b_{j,n}=R_{\wt x_n}^{-1}(c_j)$ and $\tDelta_j^B=R_{x_\star}^{-1}(c_j)$ represent this mean at $\wt x_n$ and $x_\star$, respectively. Then there exists $C>0$ such that, for every $n\ge1$ and $1\le j\le n$,
\begin{equation*}
\mE\left\|(\Gamma_{x_\star}^{\wt x_n})^{-1}b_{j,n} -(\tDelta_j^B-\tDelta_n)\right\|^2 \le Ct_j^{-2\alpha}.
\end{equation*}
\end{lemma}

\begin{proof}
Fix $1\le j\le n$. We have
\begin{equation*}
(\Gamma_{x_\star}^{\wt x_n})^{-1}b_{j,n}-(\tDelta_j^B-\tDelta_n) =\bigl([\Lambda_{x_\star}^{\wt x_n}]^{-1}b_{j,n}-(\tDelta_j^B-\tDelta_n)\bigr) +\bigl((\Gamma_{x_\star}^{\wt x_n})^{-1}-[\Lambda_{x_\star}^{\wt x_n}]^{-1}\bigr)b_{j,n}.
\end{equation*}
Apply \eqref{eq:B-coordinate-change} with $x=x_\star$, $u=\tDelta_n$, and $w=\tDelta_j^B-\tDelta_n$ to the first term, and \eqref{eq:B-inverse-transport-comparison} to the second. The domain conditions hold since $u$ and $u+w$ represent points in $\mX$. Thus
\begin{equation}\label{eq:D-coordinate-pointwise-error}
\|(\Gamma_{x_\star}^{\wt x_n})^{-1}b_{j,n}-(\tDelta_j^B-\tDelta_n)\|\le C_\Psi\|\tDelta_n\|\|\tDelta_j^B-\tDelta_n\| +C_{\Gamma,\Lambda}\|\tDelta_n\|^2\|b_{j,n}\|_{\wt x_n}.
\end{equation}
Taking $x=\wt x_n$ and $z=c_j$, we obtain
\begin{equation*}
\|b_{j,n}\|_{\wt x_n}\;\stackrel{\mathclap{\eqref{eq:C-inverse-coordinate-distance}}}{\le}\;C\|\tDelta_j^B-\tDelta_n\|.
\end{equation*}
Combining this with \eqref{eq:D-coordinate-pointwise-error} and $\|\tDelta_n\|\le D_\mX$ bounds the pointwise error by $C\|\tDelta_n\|\|\tDelta_j^B-\tDelta_n\|$. Cauchy--Schwarz then yields
\begin{equation*}
\mE\|(\Gamma_{x_\star}^{\wt x_n})^{-1}b_{j,n}-(\tDelta_j^B-\tDelta_n)\|^2 \le C(\mE\|\tDelta_n\|^4)^{1/2}(\mE\|\tDelta_j^B-\tDelta_n\|^4)^{1/2}.
\end{equation*}
Lemma~\ref{lem:C-rough-average-moments} and $\gamma_n\le\gamma_{t_j}$ bound the two fourth moments by $C\gamma_n^2$ and $C\gamma_{t_j}^2$, respectively. Therefore,
\begin{equation*}
\mE\|(\Gamma_{x_\star}^{\wt x_n})^{-1}b_{j,n}-(\tDelta_j^B-\tDelta_n)\|^2\le C\gamma_n\gamma_{t_j}\le C\gamma^2t_j^{-2\alpha}.
\end{equation*}
The last inequality uses $\gamma_n\le\gamma_{t_j}=\gamma t_j^{-\alpha}$. This proves the bound for the change of base point.
\end{proof}

\begin{proposition}
\label{prop:ol-coordinate-error}
Under the assumptions of Lemma~\ref{lem:D-ol-linear-to-ideal}, there exists $C<\infty$ such that, for every $n\ge1$ and $\diamond\in\{\mathrm{OL},\mathrm{NOL}\}$,
\begin{equation*}
\mE\|\hSigma_{n,\diamond}^{(1)} -\hSigma_{n,\diamond}^{(2)}\| \le Cn^{1/2-1/(2\beta)-\alpha}.
\end{equation*}
\end{proposition}

\begin{proof}
In \eqref{eq:D-geometric-comparison}, take $\sigma_j=\tDelta_j^B-\tDelta_n$ and $e_j=(\Gamma_{x_\star}^{\wt x_n})^{-1}b_{j,n}-(\tDelta_j^B-\tDelta_n)$. Then $Q_0=\hSigma_{n,\diamond}^{(2)}$ and $Q=\hSigma_{n,\diamond}^{(1)}$. Lemma~\ref{lem:D-coordinate-endpoint} implies $\mE\|e_j\|^2\le Ct_j^{-2\alpha}$, and Propositions~\ref{prop:ol-euclidean-bm} and~\ref{prop:ol-manifold-tangent-error} ensure $\sup_n\mE\|Q_0\|<\infty$. Thus \eqref{eq:D-geometric-comparison} proves the claim.
\end{proof}

\paragraph{Error from the online vector updates.}
Finally, for each batch $j$, we bound $\tb_{j,n}-b_{j,n}$ and then sum these bounds with weights $\ell_j^2/v_{n,\diamond}$.

\begin{lemma}
\label{lem:D-online-endpoint}
Under the assumptions of Lemma~\ref{lem:D-coordinate-endpoint}, let $\tb_{j,n}$ be the recursive approximation of $b_{j,n}$ from \eqref{eq:batch-vector-update}, initialized by $\tb_{j,j}=b_{j,j}$. Then there exists $C>0$ such that, for every $n\ge1$ and $1\le j\le n$,
\begin{equation*}
\mE\|(\Gamma_{x_\star}^{\wt x_n})^{-1} (\tb_{j,n}-b_{j,n})\|^2 \le Ct_j^{-2\alpha}.
\end{equation*}
\end{lemma}

\begin{proof}
For $k\ge j$, set $d_{j,k}:=\tb_{j,k}-b_{j,k}\in\T_{\wt x_k}\mM$, with $d_{j,j}=0$. First bound the products of later transports. Recall $u_k=k^{-1}R_{\wt x_{k-1}}^{-1}(x_k)$. Assumption~\ref{ass:geometry}(a) implies $\|u_k\|_{\wt x_{k-1}}<\rho_{\mathrm{ret}}/k$ and $u_k=R_{\wt x_{k-1}}^{-1}(\wt x_k)$. By \eqref{eq:B-lambda-distortion}, $\|\Lambda_{\wt x_{k-1}}^{\wt x_k}\|\le1+C_L\rho_{\mathrm{ret}}^2k^{-2}$. Submultiplicativity and $1+t\le e^t$ yield, for $k<n$,
\begin{equation}\label{eq:D-propagation-bound}
\left\|\Lambda_{\wt x_{n-1}}^{\wt x_n}\circ\cdots\circ\Lambda_{\wt x_k}^{\wt x_{k+1}}\right\| \le\prod_{r=k+1}^n(1+C_L\rho_{\mathrm{ret}}^2r^{-2}) \le\exp\left(C_L\rho_{\mathrm{ret}}^2\sum_{r=1}^\infty r^{-2}\right)\le C<\infty.
\end{equation}
This deterministic bound is uniform in $k,n$ and also covers $k=n$, when the composition is the identity. Next, for $k>j$, define the error in one linear update by
\begin{equation*}
\rho_{j,k}:=b_{j,k} -\Lambda_{\wt x_{k-1}}^{\wt x_k}(b_{j,k-1}-u_k).
\end{equation*}
Apply \eqref{eq:B-coordinate-change} with $x=\wt x_{k-1}$, $u=u_k$, and $w=b_{j,k-1}-u_k$; the domain conditions hold since $u$ and $u+w$ represent $\wt x_k,c_j\in\mX$. Multiplying by $\Lambda_{\wt x_{k-1}}^{\wt x_k}$ and using \eqref{eq:lambda-uniform-bounds} yields
\begin{equation}\label{eq:D-update-pointwise-error}
\|\rho_{j,k}\|_{\wt x_k} \le C\|u_k\|_{\wt x_{k-1}} \bigl(\|b_{j,k-1}\|_{\wt x_{k-1}}+\|u_k\|_{\wt x_{k-1}}\bigr).
\end{equation}
The distance comparison in Lemma~\ref{lem:distance-equivalence} implies
\begin{align*}
\mE\|u_k\|_{\wt x_{k-1}}^4&\le Ck^{-4}\bigl(\mE\|\tDelta_{k-1}\|^4+\mE\|\Delta_k\|^4\bigr),\\
\mE\|b_{j,k-1}\|_{\wt x_{k-1}}^4&\le C\bigl(\mE\|\tDelta_{k-1}\|^4+\mE\|\tDelta_j^B\|^4\bigr).
\end{align*}
Lemma~\ref{lem:C-rough-average-moments}, \eqref{eq:C-raw-moment-bounds}, $k>j\ge t_j$, and $\gamma_{k-1}\le2^\alpha\gamma_k$ then imply
\begin{equation*}
\mE\|u_k\|_{\wt x_{k-1}}^4\le Ck^{-4}\gamma_k^2,\qquad \mE\|b_{j,k-1}\|_{\wt x_{k-1}}^4\le C\gamma_{t_j}^2.
\end{equation*}
Minkowski yields the first inequality below, and Cauchy--Schwarz yields the second:
\begin{align*}
(\mE\|\rho_{j,k}\|_{\wt x_k}^2)^{1/2}\;\,
&\stackrel{\mathclap{\eqref{eq:D-update-pointwise-error}}}{\le}\;\, C\bigl((\mE[\|u_k\|_{\wt x_{k-1}}^2\|b_{j,k-1}\|_{\wt x_{k-1}}^2])^{1/2}+(\mE\|u_k\|_{\wt x_{k-1}}^4)^{1/2}\bigr)\\
&\le C\bigl((\mE\|u_k\|_{\wt x_{k-1}}^4)^{1/4}(\mE\|b_{j,k-1}\|_{\wt x_{k-1}}^4)^{1/4}+(\mE\|u_k\|_{\wt x_{k-1}}^4)^{1/2}\bigr).
\end{align*}
Substitute the two fourth-moment bounds above and use $k^{-2}\gamma_k\le k^{-1}\gamma_k^{1/2}\gamma_{t_j}^{1/2}$, since $k>j\ge t_j$. We obtain
\begin{equation}\label{eq:D-one-step-error}
(\mE\|\rho_{j,k}\|_{\wt x_k}^2)^{1/2}\le Ck^{-1}\gamma_k^{1/2}\gamma_{t_j}^{1/2}.
\end{equation}
Subtracting the exact update from \eqref{eq:batch-vector-update} and iterating from $d_{j,j}=0$ yields
\begin{align*}
d_{j,k}&=\Lambda_{\wt x_{k-1}}^{\wt x_k}d_{j,k-1}-\rho_{j,k},\\
d_{j,n}&=-\sum_{k=j+1}^n\bigl(\Lambda_{\wt x_{n-1}}^{\wt x_n}\circ\cdots\circ\Lambda_{\wt x_k}^{\wt x_{k+1}}\bigr)\rho_{j,k}.
\end{align*}
Taking norms and applying Minkowski's inequality yields
\begin{equation*}
(\mE\|d_{j,n}\|_{\wt x_n}^2)^{1/2}\;\,\stackrel{\mathclap{\eqref{eq:D-propagation-bound}}}{\le}\;\, C\sum_{k=j+1}^n(\mE\|\rho_{j,k}\|_{\wt x_k}^2)^{1/2}.
\end{equation*}
Using $\gamma_k=\gamma k^{-\alpha}$ in the one-step error bound, we obtain
\begin{align*}
(\mE\|d_{j,n}\|_{\wt x_n}^2)^{1/2}\;\,&\stackrel{\mathclap{\eqref{eq:D-one-step-error}}}{\le}\;\, Ct_j^{-\alpha/2}\sum_{k=j+1}^n k^{-1-\alpha/2}\\
&\le Ct_j^{-\alpha/2}\int_j^\infty x^{-1-\alpha/2}\,\mathrm dx =\frac{2C}{\alpha}j^{-\alpha/2}t_j^{-\alpha/2}\le\frac{2C}{\alpha}t_j^{-\alpha}.
\end{align*}
The integral comparison uses that $x^{-1-\alpha/2}$ is decreasing; the last step uses $j\ge t_j\ge1$. Squaring and using the isometry of parallel transport proves the claim, including $n=j$, when the error is zero.
\end{proof}

\begin{proposition}
\label{prop:ol-online-vector-update}
Under the assumptions of Lemma~\ref{lem:D-ol-linear-to-ideal}, there exists $C<\infty$ such that, for every $n\ge1$ and $\diamond\in\{\mathrm{OL},\mathrm{NOL}\}$,
\begin{equation*}
\mE\|\hSigma_{n,\diamond}^{(0)} -\hSigma_{n,\diamond}^{(1)}\| \le C n^{1/2-1/(2\beta)-\alpha}.
\end{equation*}
\end{proposition}

\begin{proof}
In \eqref{eq:D-geometric-comparison}, take $\sigma_j=(\Gamma_{x_\star}^{\wt x_n})^{-1}b_{j,n}$ and $e_j=(\Gamma_{x_\star}^{\wt x_n})^{-1}(\tb_{j,n}-b_{j,n})$, so $Q_0=\hSigma_{n,\diamond}^{(1)}$ and $Q=\hSigma_{n,\diamond}^{(0)}$. Lemma~\ref{lem:D-online-endpoint} implies $\mE\|e_j\|^2\le Ct_j^{-2\alpha}$, including $j=n$. Propositions~\ref{prop:ol-euclidean-bm}, \ref{prop:ol-manifold-tangent-error}, and~\ref{prop:ol-coordinate-error} imply $\sup_n\mE\|Q_0\|<\infty$. Thus \eqref{eq:D-geometric-comparison} proves the claim.
\end{proof}

\begin{proof}[Proof of Theorem~\ref{thm:covariance-rate}]
For either $\diamond\in\{\mathrm{OL},\mathrm{NOL}\}$, apply Propositions~\ref{prop:ol-online-vector-update}, \ref{prop:ol-coordinate-error}, \ref{prop:ol-manifold-tangent-error}, and~\ref{prop:ol-euclidean-bm} to the four terms in \eqref{eq:D-ol-four-term-decomposition}, respectively. We obtain
\begin{equation*}
\mE\|\hSigma_{n,\diamond}^{(0)}-\Sigma\| \le C\bigl(n^{-1/(2\beta)}+n^{(\alpha-1)/2+1/(2\beta)}+n^{1/2-1/(2\beta)-\alpha}\bigr).
\end{equation*}
The ratio of the last term to $n^{-1/(2\beta)}$ is $n^{1/2-\alpha}\to0$. Absorbing it into the first term proves the theorem.
\end{proof}

\begin{proof}[Proof of Corollary~\ref{cor:recommended-beta}]
The two decay exponents in Theorem~\ref{thm:covariance-rate} are $1/(2\beta)$ and $(1-\alpha)/2-1/(2\beta)$. Equating them yields $\beta=2/(1-\alpha)$, which satisfies the required strict lower bound. Substituting this value into either exponent yields $(1-\alpha)/4$, as claimed.
\end{proof}

\section{Proof of Proposition~\ref{prop:confidence-procedures}}\label{app:E}

\begin{proof}[Proof of Proposition~\ref{prop:confidence-procedures}]
Fix $\diamond\in\{\mathrm{OL},\mathrm{NOL}\}$. For the joint region, compare $R_{\wt x_n}^{-1}(x_\star)$ with $\tDelta_n=R_{x_\star}^{-1}(\wt x_n)$ after transport to $\T_{x_\star}\mM$:
\begin{align*}
\bigl\|(\Gamma_{x_\star}^{\wt x_n})^{-1}R_{\wt x_n}^{-1}(x_\star)+\tDelta_n\bigr\|_{x_\star}&\le\bigl\|[\Lambda_{x_\star}^{\wt x_n}]^{-1}R_{\wt x_n}^{-1}(x_\star)+\tDelta_n\bigr\|_{x_\star} +\bigl\|\bigl((\Gamma_{x_\star}^{\wt x_n})^{-1}-[\Lambda_{x_\star}^{\wt x_n}]^{-1}\bigr)R_{\wt x_n}^{-1}(x_\star)\bigr\|_{x_\star}\\
&\leq C_\Psi\|\tDelta_n\|_{x_\star}^2+C_{\Gamma,\Lambda}\|\tDelta_n\|_{x_\star}^2\|R_{\wt x_n}^{-1}(x_\star)\|_{\wt x_n} \lesssim\|\tDelta_n\|_{x_\star}^2.
\end{align*}
The second inequality applies \eqref{eq:B-coordinate-change} at $(x,u,w)=(x_\star,\tDelta_n,-\tDelta_n)$ and \eqref{eq:B-inverse-transport-comparison} to the two terms, respectively. The last uses $\|R_{\wt x_n}^{-1}(x_\star)\|_{\wt x_n}\le C_Mc_M^{-1}\|\tDelta_n\|_{x_\star}$ from \eqref{eq:B-distance-equivalence} and boundedness of $R_{x_\star}^{-1}(\mX)$. By Theorem~\ref{thm:rasgd-asymptotic-normality}, $\tDelta_n=O_{\mP}(n^{-1/2})$ and $\wt x_n\xrightarrow{\mP}x_\star$, so
\begin{equation}\label{eq:F-reverse-error-expansion}
\sqrt n\,(\Gamma_{x_\star}^{\wt x_n})^{-1}R_{\wt x_n}^{-1}(x_\star) =-\sqrt n\,\tDelta_n+o_{\mP}(1) \xrightarrow{\mathrm d}\mathcal N(0,\Sigma).
\end{equation}
Set $\hSigma_{n,\diamond}:=(\Gamma_{x_\star}^{\wt x_n})^{-1}\circ\wt\Sigma_{n,\diamond}\circ\Gamma_{x_\star}^{\wt x_n}$. Theorem~\ref{thm:covariance-rate} implies
\begin{equation}\label{eq:F-covariance-consistency}
\hSigma_{n,\diamond}\xrightarrow{\mP}\Sigma.
\end{equation}
Positive definiteness of $\Sigma$, the isometry of parallel transport, and Weyl's inequality imply $\mP\bigl(\lambda_{\min}(\wt\Sigma_{n,\diamond})\geq\lambda_{\min}(\Sigma)/2\bigr)\to1$. Continuity of inversion implies $\hSigma_{n,\diamond}^{-1}\xrightarrow{\mP}\Sigma^{-1}$. Combining this with \eqref{eq:F-reverse-error-expansion}, Slutsky's theorem and continuity of the quadratic form yield
\begin{equation}\label{eq:F-quadratic-limit}
n\left\langle R_{\wt x_n}^{-1}(x_\star),\wt\Sigma_{n,\diamond}^{-1}R_{\wt x_n}^{-1}(x_\star)\right\rangle_{\wt x_n} \xrightarrow{\mathrm d}\chi_d^2.
\end{equation}
Recall $\mathcal C_{n,\diamond}(1-\delta)=R_{\wt x_n}\bigl(\mathcal E_{n,\diamond}(1-\delta)\bigr)$. By \eqref{eq:F-covariance-consistency} and Weyl's inequality, $\lambda_{\max}(\wt\Sigma_{n,\diamond})=O_{\mP}(1)$. On the positive-definite event, the spectral theorem implies $\|v\|_{\wt x_n}^2\le\lambda_{\max}(\wt\Sigma_{n,\diamond})\langle v,\wt\Sigma_{n,\diamond}^{-1}v\rangle_{\wt x_n}$ for every $v\in\T_{\wt x_n}\mM$. The ellipsoid definition in \eqref{eq:confidence-sets} therefore yields
\begin{equation*}
\sup_{v\in\mathcal E_{n,\diamond}(1-\delta)}\|v\|_{\wt x_n}^2 \le\frac{\chi^2_{d,1-\delta}\lambda_{\max}(\wt\Sigma_{n,\diamond})}{n} =O_{\mP}(n^{-1}).
\end{equation*}
Thus the ellipsoid lies in $B_{\wt x_n}(0_{\wt x_n},\rho_{\mathrm{ret}})$ with probability tending to one. Assumption~\ref{ass:geometry}(a) makes $R_{\wt x_n}$ one-to-one on this ball, which also contains $R_{\wt x_n}^{-1}(x_\star)$. On this event, $x_\star\in\mathcal C_{n,\diamond}(1-\delta)$ if and only if $R_{\wt x_n}^{-1}(x_\star)\in\mathcal E_{n,\diamond}(1-\delta)$. Hence
\begin{equation*}
\mP\bigl(x_\star\in\mathcal C_{n,\diamond}(1-\delta)\bigr) =\mP\left(n\left\langle R_{\wt x_n}^{-1}(x_\star),\wt\Sigma_{n,\diamond}^{-1}R_{\wt x_n}^{-1}(x_\star)\right\rangle_{\wt x_n}\le\chi^2_{d,1-\delta}\right)+o(1) \stackrel{\mathclap{\eqref{eq:F-quadratic-limit}}}{\longrightarrow}1-\delta.
\end{equation*}
For the scalar interval, a first-order expansion of $h\circ R_{x_\star}$ and Theorem~\ref{thm:rasgd-asymptotic-normality} yield
\begin{equation*}
\sqrt n\bigl(h(\wt x_n)-h(x_\star)\bigr) =Dh(x_\star)\left[\sqrt n\,R_{x_\star}^{-1}(\wt x_n)\right]+o_{\mP}(1) \xrightarrow{\mathrm d}\mathcal N(0,\sigma_h^2),
\end{equation*}
where $\sigma_h^2:=\langle\grad h(x_\star),\Sigma\grad h(x_\star)\rangle_{x_\star}>0$. Here $DR_{x_\star}(0_{x_\star})=I$, and the remainder satisfies $\sqrt n\,o(\|\tDelta_n\|_{x_\star})=o_{\mP}(1)$. By the isometry of parallel transport, \eqref{eq:scalar-variance} becomes
\begin{equation*}
\widehat\sigma_{n,h,\diamond}^2 =\left\langle(\Gamma_{x_\star}^{\wt x_n})^{-1}\grad h(\wt x_n),\hSigma_{n,\diamond}(\Gamma_{x_\star}^{\wt x_n})^{-1}\grad h(\wt x_n)\right\rangle_{x_\star}.
\end{equation*}
Continuity of the gradient and parallel transport implies $(\Gamma_{x_\star}^{\wt x_n})^{-1}\grad h(\wt x_n)\xrightarrow{\mP}\grad h(x_\star)$, and \eqref{eq:F-covariance-consistency} therefore implies $\widehat\sigma_{n,h,\diamond}^2\xrightarrow{\mP}\sigma_h^2$. Thus, $\widehat\sigma_{n,h,\diamond}>0$ with probability tending to one. Set the following ratio to zero if its denominator is zero. Slutsky's theorem then yields
\begin{equation*}
\frac{\sqrt n\bigl(h(\wt x_n)-h(x_\star)\bigr)}{\widehat\sigma_{n,h,\diamond}} \xrightarrow{\mathrm d}\mathcal N(0,1).
\end{equation*}
Since $\mP\bigl(\widehat\sigma_{n,h,\diamond}>0\bigr)\to1$ and the standard normal distribution is continuous, we have
\begin{equation*}
\mP\bigl(h(x_\star)\in\mathcal I_{n,h,\diamond}(1-\delta)\bigr) =\mP\left(\left|\frac{\sqrt n\bigl(h(\wt x_n)-h(x_\star)\bigr)}{\widehat\sigma_{n,h,\diamond}}\right|\le z_{1-\delta/2}\right)+o(1) \longrightarrow1-\delta.
\end{equation*}
This proves the assertion for the scalar confidence interval.
\end{proof}

\section{Additional numerical experiments}\label{app:F}

This appendix supplements Section~\ref{sec:5} with further experiments on covariance estimation and statistical inference. We first compare covariance errors in higher dimensions, then examine interval coverage and length across different dimensions and eigengaps. All experiments use the Gaussian model and algorithm settings described in Section~\ref{sec:5.1}, with $200$ independent runs for each setting.

\subsection{Covariance estimation in higher dimensions}\label{sec:G.1}

We extend the covariance estimation experiments in Section~\ref{sec:5.2} to $G_{15,3}$, $G_{20,4}$, and $G_{25,5}$, with eigengaps $\Delta_\lambda\in\{1,0.25\}$ and batch constants $c\in\{1,2,5\}$. Recall from Section~\ref{sec:4.2} that $G_{p,k}$ is the space of $k$-dimensional subspaces of $\mR^p$, with dimension $d=k(p-k)$. Thus, the dimensions of these three manifolds are $36$, $64$, and $100$, respectively.

Figure~\ref{fig:appendix-covariance-highdim} shows an overall decrease in the mean errors of both estimators as the sample size increases. At larger sample sizes, the slopes of the log--log curves are close to the reference slope $-1/8$. At $n=3\times10^5$, $c=1$ gives the lowest mean error among the three batch constants for both estimators in all six settings.

\begin{figure}[H]
\centering
\includegraphics[width=\textwidth]{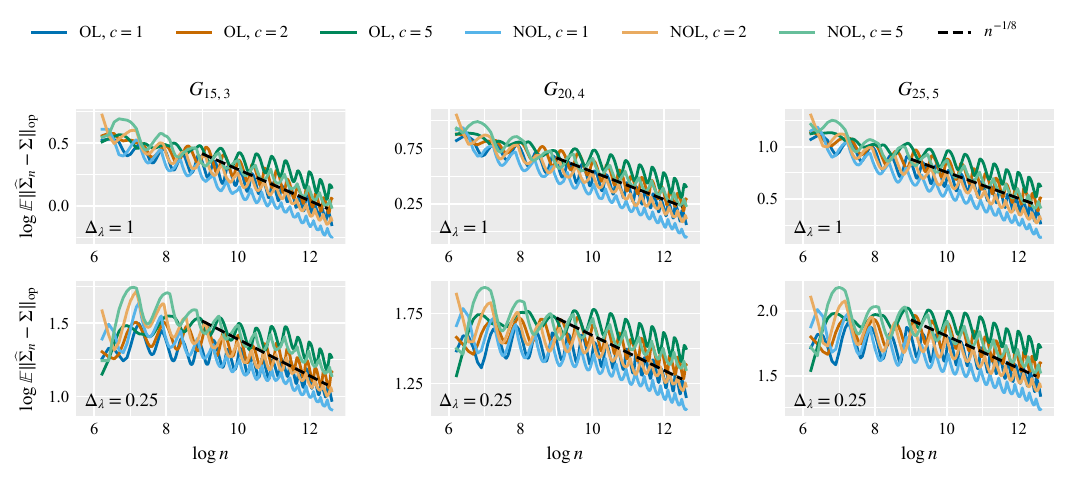}
\caption{Gaussian streaming PCA: log mean covariance error (operator norm) against log sample size. Columns correspond to $G_{15,3}$, $G_{20,4}$, and $G_{25,5}$, and the top and bottom rows use eigengaps $1$ and $0.25$, respectively. As in Figure~\ref{fig:final-covariance-convergence}, darker and lighter curves represent OL and NOL, with blue, orange, and green indicating $c=1,2,5$, respectively. Each mean is based on $200$ independent runs. The black dashed lines indicate the reference slope $-1/8$.}
\label{fig:appendix-covariance-highdim}
\end{figure}

\subsection{Additional scalar confidence intervals}\label{sec:G.2}

We next extend the confidence-interval comparison in Section~\ref{sec:5.3} to a wider range of dimensions and eigengaps. Recall from Section~\ref{sec:4.2} that $h_q([X])=\|X^\top q\|^2$ is the squared length of the projection of a fixed unit vector $q$ onto the subspace $[X]$. We select $q$ as in Section~\ref{sec:5.3} and construct nominal $95\%$ confidence intervals for $h_q(x_\star)$ using the OL and NOL variance estimates, with $c=1$. For comparison, we also include oracle intervals based on the true scalar asymptotic variance.

Table~\ref{tab:appendix-scalar-inference} reports empirical coverage and mean interval length at $n=3\times10^5$ for all twelve settings, including the two reported in Table~\ref{tab:scalar-inference}. OL and NOL have similar mean interval lengths within each setting, and their empirical coverage ranges from $0.880$ to $0.925$. In particular, on $G_{25,5}$ with the smaller eigengap $\Delta_\lambda=0.25$, the coverage rates are $0.885$ and $0.880$ for OL and NOL, respectively. Although these coverage rates are below the nominal $95\%$ level, coverage for both methods is close to $90\%$ at this sample size, even in this higher-dimensional problem with a small eigengap.

\begin{table}[!h]
\centering
\small
\caption{Empirical coverage and mean length of nominal $95\%$ confidence intervals for $h_q(x_\star)$ in Gaussian streaming PCA, with $n=3\times10^5$ and $c=1$. Results are based on $200$ independent runs, with Monte Carlo standard errors in parentheses. Mean lengths are multiplied by $10^3$.}
\label{tab:appendix-scalar-inference}
\setlength{\tabcolsep}{4pt}
\renewcommand{\arraystretch}{1.15}
\begin{tabular}{@{}lcrrrrrr@{}}
\hline
& & \multicolumn{3}{c}{Coverage} & \multicolumn{3}{c}{Mean length $\times10^3$}\\
$G_{p,k}$ & $\Delta_\lambda$ & OL & NOL & Oracle & OL & NOL & Oracle\\
\hline
$G_{3,1}$ & 1 & 0.920 (0.019) & 0.920 (0.019) & 0.980 (0.010) & 4.425 & 4.447 & 4.953 \\
$G_{3,1}$ & 0.25 & 0.920 (0.019) & 0.925 (0.019) & 0.980 (0.010) & 11.305 & 11.722 & 13.424 \\
$G_{7,2}$ & 1 & 0.915 (0.020) & 0.910 (0.020) & 0.960 (0.014) & 3.267 & 3.252 & 3.488 \\
$G_{7,2}$ & 0.25 & 0.915 (0.020) & 0.895 (0.022) & 0.955 (0.015) & 4.581 & 4.628 & 5.012 \\
$G_{10,3}$ & 1 & 0.915 (0.020) & 0.925 (0.019) & 0.955 (0.015) & 2.224 & 2.273 & 2.415 \\
$G_{10,3}$ & 0.25 & 0.895 (0.022) & 0.920 (0.019) & 0.945 (0.016) & 3.891 & 4.009 & 4.320 \\
$G_{15,3}$ & 1 & 0.915 (0.020) & 0.910 (0.020) & 0.955 (0.015) & 2.032 & 2.075 & 2.265 \\
$G_{15,3}$ & 0.25 & 0.900 (0.021) & 0.905 (0.021) & 0.955 (0.015) & 3.166 & 3.276 & 3.654 \\
$G_{20,4}$ & 1 & 0.915 (0.020) & 0.915 (0.020) & 0.955 (0.015) & 2.653 & 2.653 & 2.888 \\
$G_{20,4}$ & 0.25 & 0.895 (0.022) & 0.895 (0.022) & 0.955 (0.015) & 4.478 & 4.571 & 5.107 \\
$G_{25,5}$ & 1 & 0.900 (0.021) & 0.895 (0.022) & 0.915 (0.020) & 1.799 & 1.803 & 1.932 \\
$G_{25,5}$ & 0.25 & 0.885 (0.023) & 0.880 (0.023) & 0.905 (0.021) & 3.512 & 3.596 & 3.948 \\
\hline
\end{tabular}
\end{table}

Figure~\ref{fig:appendix-intervals} complements Table~\ref{tab:appendix-scalar-inference} by showing how coverage and mean interval length change with the sample size on $G_{3,1}$ and $G_{25,5}$. For both eigengaps, the intervals become shorter on average as the sample size increases. OL and NOL have similar coverage and mean interval lengths, with empirical coverage around $90\%$ at larger sample sizes.

\begin{figure}[h!]
\centering
\includegraphics[width=\textwidth]{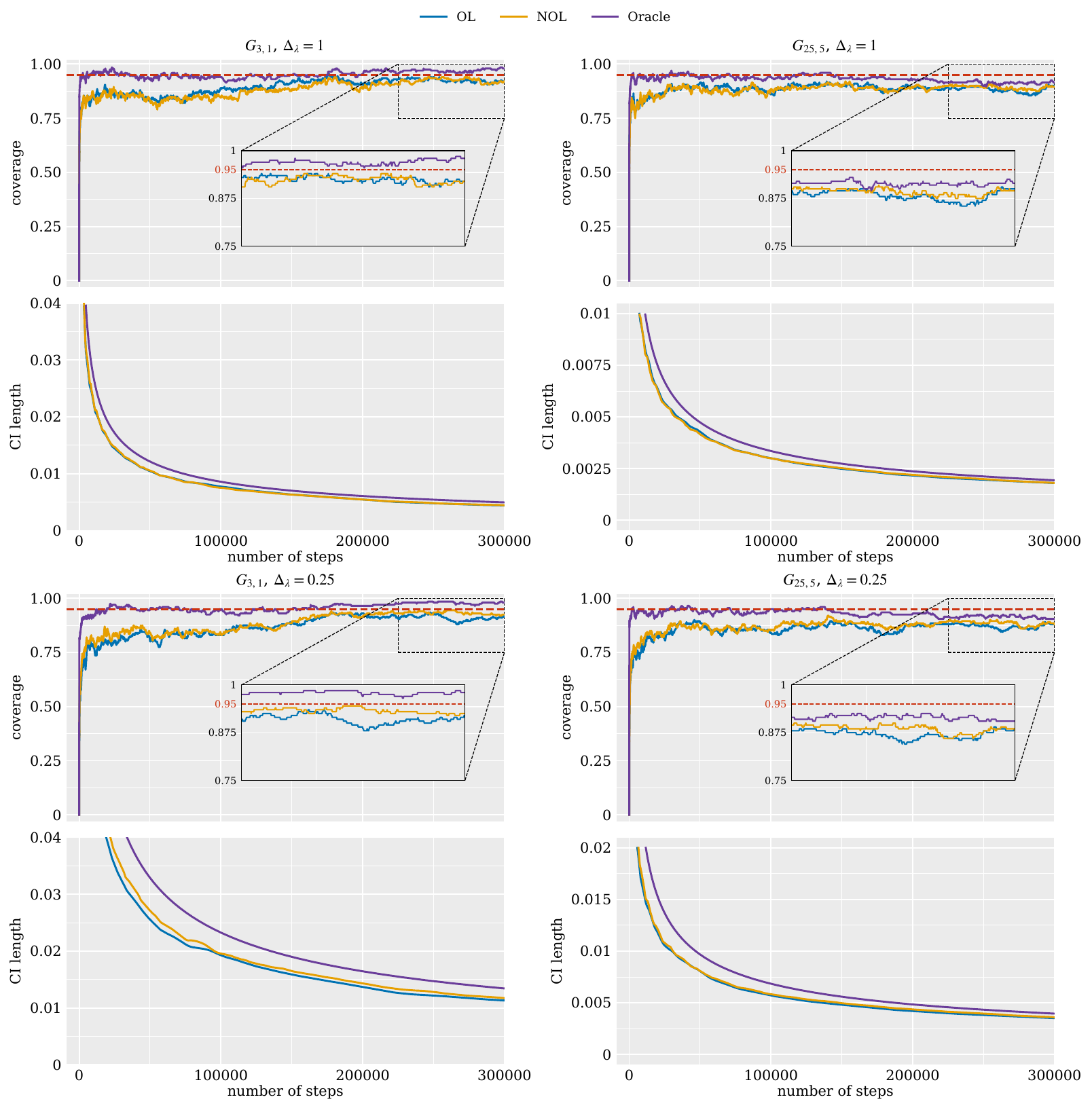}
\caption{Gaussian streaming PCA: empirical coverage and mean length of $95\%$ confidence intervals for $h_q(x_\star)$, with $c=1$. Columns correspond to $G_{3,1}$ and $G_{25,5}$; the upper and lower pairs of rows use eigengaps $1$ and $0.25$, respectively, with coverage above mean length in each pair. The blue, orange, and purple curves show OL, NOL, and Oracle, respectively, with results based on $200$ independent runs. The small plots inside the coverage panels show the later coverage values, and the red dashed lines mark the nominal coverage level $0.95$.}
\label{fig:appendix-intervals}
\end{figure}

Together, these Gaussian experiments complement the local theory by showing decreasing covariance errors and similar interval performance for OL and NOL across the settings considered.